%% file: 0_revised_paper_v3.tex
\documentclass[twoside,11pt]{article}

\usepackage{jmlr2e}
\usepackage{diagbox}

\usepackage{graphicx}
\usepackage{bm}

\usepackage{booktabs}
\usepackage{xcolor}
\usepackage{graphicx}
\usepackage{caption}
\usepackage{subcaption}
\usepackage{tabularx}

\usepackage{amsmath, amssymb, mathtools}
\usepackage{algorithm}
\usepackage[noend]{algpseudocode}
\usepackage{verbatim}
\usepackage{colortbl}
\usepackage{chngcntr}
\usepackage{apptools}

\newcommand{\pp}{\bar{p}}
\newcommand{\qq}{\bar{q}}

\ShortHeadings{Sig-Wasserstein GANs for Conditional Time Series Generation}{Ni, Szpruch, Wiese, Liao, Xiao and Sabate-Vidales }

\firstpageno{1}

\begin{document}

\title{Sig-Wasserstein GANs for \\Conditional  Time Series Generation}
\author{\name Shujian Liao \email shujian.liao.18@ucl.ac.uk \\
       University College London\\
       London, UK
       \AND 
       \name Hao Ni \email h.ni@ucl.ac.uk \\
       University College London\\
       London, UK
       \AND
       \name Marc Sabate-Vidales \email m.sabate-vidales@sms.ed.ac.uk \\
       University of Edinburgh\\
       Edinburgh, Scotland
        \AND
       \name Lukasz Szpruch \email lszpruch@turing.ac.uk \\
       University of Edinburgh\\
       Edinburgh, Scotland
       \AND
       \name Magnus Wiese \email wiese@rhrk.uni-kl.de \\
       University of Kaiserslautern\\
       Kaiserslautern, Germany
       \AND
       \name Baoren Xiao \email baoren.xiao.18@ucl.ac.uk \\
       University College London\\
       London, UK
       }

\editor{}

\maketitle

\begin{abstract}
Generative adversarial networks (GANs) have been extremely successful in generating samples, from seemingly high dimensional probability measures. However, these methods struggle to capture the temporal dependence of joint probability distributions induced by time-series data. Furthermore, long time-series data streams hugely increase the dimension of the target space, which may render generative modelling infeasible. To overcome these challenges, motivated by the autoregressive models in econometric, we are interested in the conditional distribution of future time series given the past information. We propose the generic conditional Sig-WGAN framework by integrating Wasserstein-GANs (WGANs) with mathematically principled and efficient path feature extraction called the signature of a path. The signature of a path is a graded sequence of statistics that provides a universal description for a stream of data, and its expected value characterises the law of the time-series model. In particular, we develop the conditional Sig-$W_1$ metric, that captures the conditional joint law of time series models, and use it as a discriminator. The signature feature space enables the explicit representation of the proposed discriminators which alleviates the need for expensive training. We validate our method on both synthetic and empirical dataset and observe that our method consistently and significantly outperforms state-of-the-art benchmarks with respect to measures of similarity and predictive ability.
\end{abstract}


\section{INTRODUCTION}
Ability to generate high-fidelity synthetic time-series datasets can facilitate testing and validation of data-driven products and enable data sharing by respecting the demand for privacy constraints, \cite{bellovin2019privacy,tucker2020generating,assefagenerating}.
Until recently, time-series models were mostly conceived by handcrafting a parsimonious parametric model, which would best capture the desired statistical and structural properties or the so-called stylized facts of the time series data. Typical examples are discrete time autoregressive econometric models, \cite{tsay2005analysis}, or continuous time stochastic differential equations (SDEs), \cite{karatzas1998brownian}. In many applications, such as finance and economics, one cannot base models on well-established ``physical laws'' and the risk of handcrafting inappropriate models might be significant. It is therefore tempting to build upon success of non-parametric unsupervised learning method such as deep generative modelling (DGM) to enable data-driven model selection mechanisms for dynamically evolving data sets such as time-series.
However, off-the-shelf DGMs perform poorly on the task of learning the temporal dynamics of multivariate time series data $x_{1:T} = (x_{1}, \cdots, x_{T}) \in \mathbb{R}^{d \times T}$ due to (1) complex interaction between temporal features and spatial features, and (2) potential high dimension for the joint distribution of $x$ (e.g when $T>>1$), see e.g. \cite{mescheder2018training}. 

In this work, we are interested in developing a data-driven non-parametric model for the conditional distribution $\text{Law}(x_{\text{future}}|x_{past})$ of future time series given  $x_{\text{past}}:=x_{t-\pp+1: t}$. This setting includes classical auto-regressive processes. Learning conditional distributions is particularly important in the cases of  (1) predictive modelling: it can be directly used to forecast future time series distribution given the past information;  (2) causal modelling: conditional generator can be used to produce counterfactual statements; and (3) building the joint law through conditional laws enables to incorporate a prior into the learning process which is necessary for building high-fidelity generators.   

Learning the conditional distribution is often more desirable than learning the joint law and can lead to more efficient learning with a smaller amount of data, \cite{ng2002discriminative,blanka2020}. To see that consider the following example.
\begin{example}[Auto regressive process]
Let $Z_t\sim N(0,\Sigma)$ be $d$-dimensional Gaussian random variable. Fix $a:\mathbb R^{d\times  \pp} \times \mathbb R^{d}\rightarrow \mathbb R^d$.  Define an auto-regressive process $(X_t)_{t\geq 0}$ with the initial condition $X_{1:\pp}=x_{1:\pp}$, as
$
X_{t+1}= a(X_{t-\pp+1:t},Z_{t+1}).
$
Hence one can see that 
$$\text{Law}(X_{1:T})=\prod_t\text{Law}(X_{t+1}| X_{t-\pp+1:t})\,.$$
As a consequence, the problem of learning distribution over $\mathbb{R}^{d\times T}$ can be reduced to learning conditional distribution over  $\mathbb R^{d}$.
\end{example}
In our setting, the conditional law is time invariant and hence having only one data trajectory $x_{1:T}$ gives $T-\pp-1$ samples. This should be contrasted with having one sample when trying to learn $\text{Law}(X_{1:T})$ directly.


\paragraph{Structure}The problem of calibrating a generative model in the time series domain is formulated in \autoref{sec:problem_formulation}. There we overview the key results of this work against the work available in literature. In \autoref{sec:sigs_and_expected_sigs}, we introduce the signature of a path formally. In \autoref{sec Sig-Wasserstein metric}, we establish key theoretical results of this work. In \autoref{sec Sig-Wasserstein GANs for Conditional law}, we present the algorithm
while in Section \ref{Section_Numerical_Experiment} we present extensive numerical experiments. 

\section{PROBLEM FORMULATION}
    \label{sec:problem_formulation}
 Fix $T>0$ and $X:=(X_1, \cdots, X_T) \in \mathbb{R}^{d \times T}$ is a $d$-dimensional time series of length $T$. Let $W$ be the window size (typically $W << T$).  Suppose that we have access to one realization of $X$, i.e. $\left(x_{1}, \cdots, x_{T}\right)$ and then obtain the $N$ copies of time series segment of a window size $W$ by sliding window. We assume that for each $t$, the time series segment $(x_{t+1}: \cdots, x_{t+W})$ is sampled from the same but \textit{unknown} distribution on the time series (path) space $\mu \in \mathcal P(\mathbb{R}^{d \times W})$. The objective of the \textit{unconditional} generative model is to train a generator such as to produce a $\mathbb{R}^{d \times W}$-valued random variable whose law is close to $\mu$ using time series data $x$.\footnote{For any distribution $\mu^X \in \mathcal P(\mathbb{R}^{d \times W})$, one can construct a stochastic process $X: \Omega \rightarrow \mathbb{R}^{d \times W}$, such that $\text{Law}(X)=\mu^X$, see \cite[Proposition 9.1.2 and Theorem 13.1.1]{dudley1989real}.}

In contrast, this paper focuses on the task of the \textit{conditional} generative model of future time series when conditioning on past time series. Let $\pp, \qq$ denote the window size of the past time series $X_{\text{past}, t}:= (X_{t-\pp+1}, \cdots, X_{t}) \in \mathbb{R}^{d \times \pp}  =: \mathcal{X}$ and future time series $X_{\text{future}, t}:= (X_{t+1}, \cdots, X_{t+\qq}) \in \mathbb{R}^{d \times \qq} =: \mathcal{Y}$, respectively. Assume that the joint distribution of $(X_{\text{future}, t}, X_{\text{past}, t}) = (X_{t-\pp+1, t+\qq})$ does not depend on time $t$. Given a realization of time series $(x_1, \cdots, x_T)$, at each time $t$, the pairs of past path $x_{\text{past}, t}:= (x_{t-\pp+1}, \cdots, x_{t}) \in \mathcal{X}$ and future path $x_{\text{future}, t}:= (x_{t+1}, \cdots, x_{t+\qq}) \in \mathcal{Y}$ are sampled from the same but \textit{unknown} distribution of $\mathcal{X} \times \mathcal{Y}$-valued random variable, denoted by $\left(X_{\text{past}}, X_{\text{future}}\right)$. We aim to train a generator to produce the conditional law, denoted by $\mu_{t}(x):=\text{Law}(X_{\text{future}, t}| X_{\text{past}, t}=x)$. As $\mu_{t}(x)$ is independent with $t$ and hence we write $\mu(x)$ for simplicity. But of course, the methodology developed here all applies if one can access a collection of $(X^{{(i)}}_{\text{past}}, X^{{(i)}}_{\text{future}})_{i=1}^N$ of $N$ independent copies of the past and future time series for $N\geq 1$. 

More specifically, the aim of the conditional generative model is to map samples from some basic distribution $\mu^{Z}$ supported on $\mathcal{Z}\subseteq \mathbb R^{d_{z}}$ together with data $x_{\text{past}, t}$ into samples from the conditional law $\mu(x_{\text{past},t})$. Given latent  $(\mathcal{Z}, \mathcal{B}(\mathcal{Z}))$, conditional  $(\mathcal{X}, \mathcal{B}(\mathcal{X}))$ and target $(\mathcal{Y}, \mathcal{B}(\mathcal{Y}))$ measurable spaces, one considers a map $G: \Theta^{(g)} \times \mathcal{X} \times \mathcal{Z}  \rightarrow  \mathcal{Y}$, with $\Theta^{(g)}$ being a parameter space. Given parameters $\theta^{(g)} \in \Theta^{(g)}$ and $x_{\text{past},t}$ , $G(\theta^{(g)},x_{\text{past},t})$ transports $\mu_{z}$ into $\nu(\theta^{(g)}, x_{\text{past},t}):=G(\theta^{(g)},x_{\text{past}})_{\#}\mu_{z}= \mu_{z}(G(\theta^{(g)},x_{\text{past},t})^{-1}(B))$, $B\in\mathcal{B}(\mathcal{Y})$. The aim is to find $\theta$ such that $\nu(\theta, X_{\text{past},t})$ is a good approximation of $\mu(X_{\text{past},t})$ with respect to a suitable metric. Often the metric of choice is a Wasserstein distance which leads to
\begin{eqnarray} \label{def cw1 metric}
\text{W}_{1}\left(\mu(x_{\text{past}}), \nu(\theta^{(g)}, x_{\text{past}})\right) = \sup_{|f|_{\text{Lip}} \leq 1 } \mathbb{E}_{\mu(x_{\text{past}})}[f(X_{\text{future}}) ] - \mathbb{E}_{\nu(\theta^{(g)},x_{\text{past}})}[f(X_{\text{future}}) ]
\end{eqnarray}
The optimal transport metrics, such as Wasserstein distance, are attractive due to their ability to capture meaningful geometric features between measures even when their supports do not overlap, but are expensive to compute, \cite{genevay2018sample}. 
Furthermore, when computing Wasserstein distance for conditional laws, one needs to compute the conditional expectation $ \mathbb{E}_{\mu(x_{\text{past}})}[f(X_{\text{future}}) ]$ using input data. In the continuous setting studied in this paper, this is computationally heavy and typically will introduce additional bias (e.g. due to employing least square regression to compute an approximation to the conditional expectation).

Since our aim is to learn the conditional law for all possible conditioning random variables, we consider 
$$\mathbb E_{X_{\text{past}} \sim \mu}[ \text{W}_{1}\left(\mu(X_{\text{past}}), \nu(\theta, X_{\text{past}})\right)].$$
Note that, since $\text{W}_1$ is non-negative, $\mathbb E[ \text{W}_{1}\left(\mu(X_{\text{past}}), \nu(\theta, X_{\text{past}})\right)]=0$ implies that $\mu(x_{\text{past}}) =  \nu(\theta, x_{\text{past}})$ almost surely. 

\paragraph{Challenges in implementing $W_{1}$-GAN for conditional laws.} There are two key challenges when one aims to implement $W_{1}$-GAN for conditional laws. 

\textbf{Challenge 1: Min-Max problem}
A typical implementation of $W_{1}$-GAN would require introduction of a parametric function approximation $\Theta^{(d)} \times \mathcal Y\ni (\theta^{(d)},\omega)\mapsto f(\theta^{(d)},\omega)$ such that $\omega \mapsto  f(\theta^{(d)},\omega)$ is 1-Lip. In the case of neural network approximation, this can be achieved by clipping the weights or adding penalty that ensures $\nabla_{\omega}  f(\theta^{(d)},\omega)$ is less than $1$, see \cite{gulrajani2017improved}. Recall a  definition of $W_1$ in \eqref{def cw1 metric} and define
\[
\ell(\theta^{(g)},\theta^{(d)}):= \mathbb E_{X_{\text{past}}\sim \mu} \left[
\mathbb{E}_{\mu(X_{\text{past}})}[f(\theta^{(d)},X_{\text{future}}) ] - \mathbb{E}_{\nu(\theta^{(g)},X_{\text{past}})}[f(\theta^{(d)},X_{\text{future}}) ] \right]\,.
\]
Training conditional $W_1$-GAN constitutes solving the min-max problem  
\begin{equation}\label{eq:minmax}
\min_{\theta^{(g)}}   \max_{\theta^{(d)}}\ell(\theta^{(g)},\theta^{(d)})\,.
\end{equation}
In practice, the min-max problem is solved by iterating gradient descent-ascent algorithms and its convergence can be studied using tools from game theory \cite{mazumdar2019finding,lin2020gradient}. However, it is well known that the first order method, which is typically used in practice,  might not converge even in the  convex-concave case, \cite{daskalakis2018limit,daskalakis2017training,mertikopoulos2018cycles}. Consequently, the adversarial training is notoriously difficult to tune, \cite{farnia2020gans,mazumdar2019finding}, and generalisation error is very sensitive to the choice of discriminator and hyper-parameters, as it was demonstrated in large scale study in \cite{lucic2017gans}.

\textbf{Challenge 2: Computation of the conditional expectation} In addition to the challenge of solving a min-max for each new parameter $\theta^{(d)}$, one needs to compute the conditional expectation  $\mathbb{E}_{\mu(X_{\text{past}})}[f(\theta^{(d)},X_{\text{future}}) ]$ (or $\mathbb{E}_{\mu(X_{\text{past}})}[\nabla_{\theta^{(d)}}f(\theta^{(d)},X_{\text{future}}) ]$ if one can interchange differentiation and integration). From Doob-Dynkin Lemma we know that this conditional expectation is a measurable function of $X_{\text{past}}$ and approximation of these is computationally heavy and can be recast as a mean-square optimisation problem
\[
\mathbb E[ | f(\theta^{(d)},X_{\text{future}}) -  \mathbb{E}_{\mu(X_{\text{past}})}[f(\theta^{(d)},X_{\text{future}}) ]  |^2 ] = \inf_{h \,\, \text{measurable}}
\mathbb E[ | f(\theta^{(d)},X_{\text{future}}) -  h(X_{\text{past}})  |^2 ]\,.
\]
Practical solution of this problem requires an additional function approximation which may introduce additional bias and makes the overall algorithm much harder to tune.

\subsection{Summary of the key results} \label{sec summary}

Discrete time econometric models can be viewed as discretisation of certain SDEs type models, \cite{kluppelberg2004continuous}. The continuous time perspective by embedding discrete time series into a path space, which we follow in this paper, is particularly useful when learning from irregularly sampled data sets and designing efficient training methods that naturally scale when working with high and ultra high frequency data \cite{liu2019neural,gierjatowicz2020robust,cuchiero2020generative}. Our approach utilises the signature of a path which is a mathematical object that emerges from rough-path theory and provides a highly abstract and universal description of complex multimodal data streams that has recently demonstrated great success in several machine learning tasks \cite{xie2017learning, yang2017leveraging, kidger2019deep}. To be more precise, we add a time dimension to $\tilde{d}$ dimensional time series $(x_t)_{t = 1}^{T}$ and embed it into $X: [0, T] \rightarrow E:=\mathbb{R}^{d}$ with $d = \tilde{d} +1$. For example, this is easily done by linearly interpolating discrete time data points. We assume that $X$ is regular (c.f Section \ref{subsec:ts2path}) and denote the space of all such regular paths by $\Omega_{0}([0, T], E)$.
The signature of a path determines the path up to tree-like equivalence \cite{UniquenessOfSignature, boedihardjo2014uniqueness}. Roughly speaking, there is an almost one-to-one correspondence between the signature and the path, but when restricting the path space to $\Omega_{0}([0, T], E)$, the signature (feature) map $S: x \mapsto S(x)$, $x \in \Omega_{0}([0, T], E)$, is bijective. In other words, the signature of a path in $\Omega_{0}([0, T], E)$ determines the path completely \cite{levin2013learning}.  Let $S(\Omega_{0}([0, T], E))$ denote the range of the signature of all the possible paths in $\Omega_{0}([0, T], E)$. Note that 
the signature map $S$, defined on $\Omega_0([0, T], E)$, is continuous with respect to the 1-variation topology \cite{lyons2007differential}. A remarkable property of the signature is the following universal approximation property:

\begin{theorem}[Universality of Signature \cite{levin2013learning}]\label{th sig} Consider a compact set $\mathcal{K} \subset S(\Omega_{0}([0, T], E))$. Let $\bm{f}: \mathcal{K} \to \mathbb{R}$ be any continuous function. Then, for any $\epsilon>0$, there exists a linear functional $\bm{L} \in T((E))^*$ acting on the signature such that 
\begin{equation}
    \sup_{S \in \mathcal{K}}|\bm{f}(S) - \bm{L}(S)| < \epsilon.
\end{equation}
\end{theorem}

Theorem \ref{th sig} applies to any subspace topology on $(S(\Omega_{0}([0, T], E))$, which is inherited from the Hausdorff topology $T((E))$, that is finer than the weak topology.  The theorem tells us that any continuous functional on the signature space can be arbitrarily well approximated by a linear combination of coordinate signatures.

Since the signature $S$ is bijective and continuous when restricting the path space to $\Omega_{0}([0, T], E)$,  the pushforward of the measure on the path space, $\bm{\mu}(B):=(S_{\#}\mu)(B) = \mu(S^{-1}(B))$ for $B$ in the $\sigma$-algebra of $S(\Omega_{0}(J,E))$, induces the measure on the signature space. With this in mind, the $W_1$ on the signature space is given by 
\[
W_1^{\text{Sig}}(\mu, \nu):= \sup_{||\bm{f}||_{Lip, \mathcal{K}} \leq 1} \mathbb{E}_{S \sim \bm{\mu}}[\bm{f}(S)] - \mathbb{E}_{S \sim \bm{\nu}}[\bm{f}(S)]\,.
\]
 
Motivated by the universality of signature, we consider the following $\text{Sig-}W_{1}$ metric as the proxy for $W_1^{\text{Sig}}$ by restricting the admissible test functions to be linear functionals:
 \[
\text{Sig-}W_{1}(\mu, \nu) = \sup_{\vert\vert L \vert \vert_{Lip} \leq 1, \bm{L} \text{ is a linear functional}}  \mathbb{E}_{ S \sim{\bm{\mu}}}[\bm{L}(S)] - \mathbb{E}_{S \sim{\bm{\nu}}}[\bm{L}(S)]\,.
\]

The Sig-$W_1$ metric was initially proposed in \cite{ni2021sig}, where the Lipschitz norm of $f$ is obtained by endowing the underlying signature space equipped with the $l^2$ norm. Here, we consider a more general case, where the norm of the signature space is chosen as $l^{\mathrm{p}}$ for some $\mathrm{p} >1$.

In Lemma \ref{Lemma_Sig_W1}, we show that when 
$$
||\bm{L}||_{Lip} := \sup_{x \neq y, x, y \in \bm{T}^{\mathrm{p}(E)} } \frac{|\bm{L}(x-y)|}{ ||x-y||_{\mathrm{p}}},\text{ for some }\mathrm{p}\geq 1, 
$$
where $\bm{T}^{\mathrm{p}}(E)$ is the set of all the tensor series elements with finite $l^{\mathrm{p}}$ norm,
then Sig-$W_1$ admits analytic formula 
\begin{equation*}
	\text{Sig-}W_{1}(\mu, \nu)= \|\mathbb{E}_{S \sim \bm{\mu}}[S] - \mathbb{E}_{S \sim \bm{\nu}}[S]\|_{\mathrm{p}}\,.
\end{equation*}
The significance of this result is that Sig-$W_1$-GAN framework reduces the challenging min-max problem to supervised learning, without severing loss of accuracy when compared with Wasserstein distance on the path space. Figure~\ref{fig:VAR(1) dim2 score metrics intro} of the 2-dimensional VAR(1) dataset illustrates that the SigCWGAN helps stablize the training process and accelerate the training to converge compared with the CWGAN when keeping the same conditional generator for both methods. 

\begin{figure}
	\centering
	\begin{subfigure}{\linewidth}
		\centering
		\includegraphics[width=\textwidth]{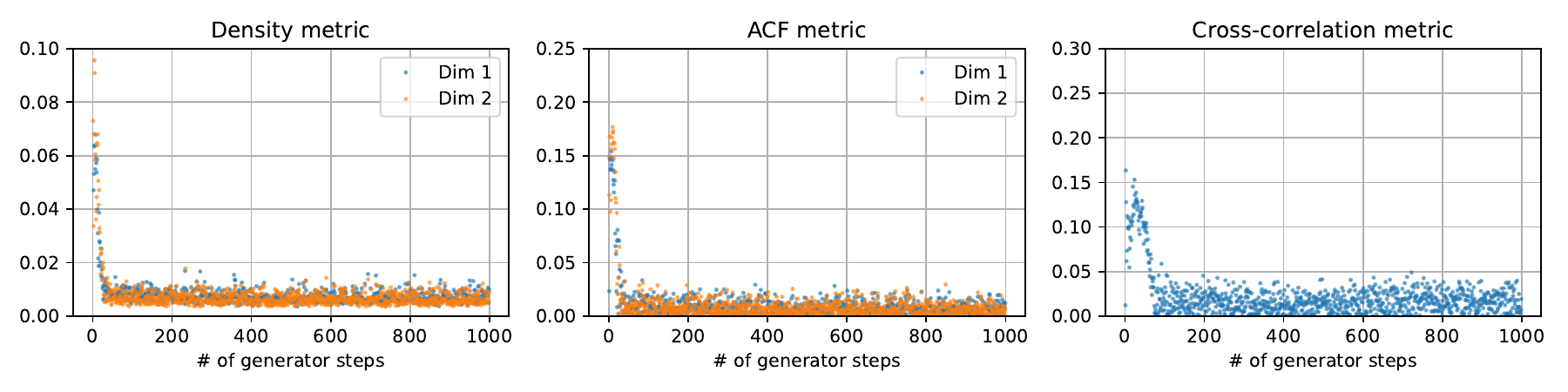}
		\caption{SigCWGAN}
	\end{subfigure}
	\begin{subfigure}{\linewidth}
		\centering
		\includegraphics[width=\textwidth]{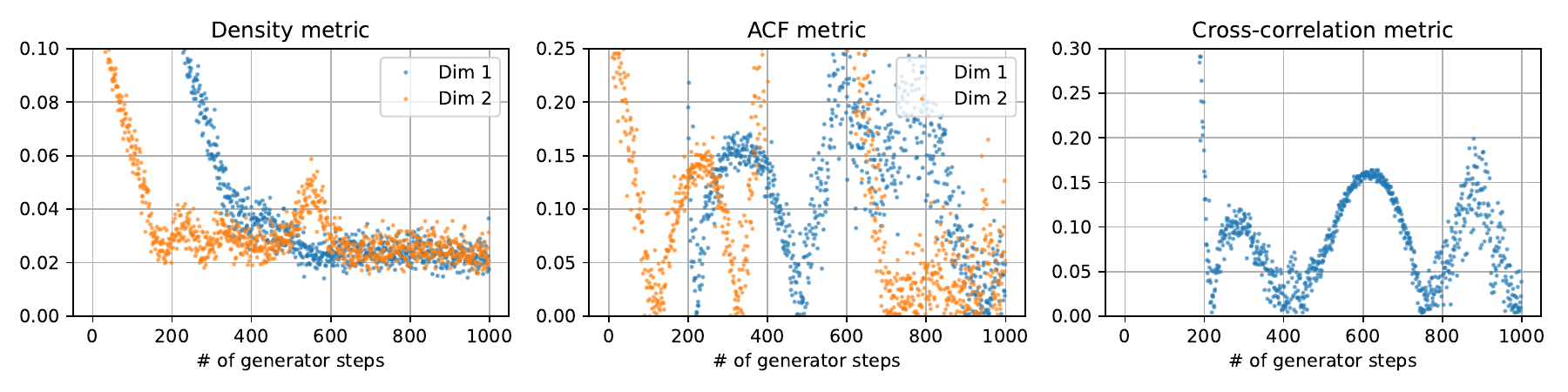}
		\caption{CWGAN}
	\end{subfigure}
\caption{Comparison across three performance metrics (see section \ref{Section_Numerical_Experiment}) of training SigCWGAN with loss function \eqref{eq sigcwgan loss} and CWGAN with loss function \eqref{eq:minmax}
for $2$-dimensional $\operatorname{VAR}(1)$, given by 
$X_{t+1} = \phi X_{t} +\epsilon_{t+1}$
with $(\epsilon_{t})_{t=1}^{T}$  iid Gaussian-distributed random variables with co-variance matrix $\sigma \mathbf{1} + (1-\sigma) \mathbf{I}$ and autocorrelation coefficient $\phi=0.8$ and co-variance parameter $\sigma=0.8$.
The explicit form of the model allows for an unbiased approximation of conditional expectation in \eqref{eq:minmax} using Monte Carlo samples.  
The colours blue and orange indicate the relevant distance/score for each dimension.}
\label{fig:VAR(1) dim2 score metrics intro}
\end{figure}

tting studied here, we lift both $(X_{\text{past}},X_{\text{future}})$ into the signature space, that is $(X_{\text{past}},X_{\text{future}})\mapsto (S_{\text{past}},S_{\text{future}}) := (S(X_{\text{past}}),S(X_{\text{future}}))$. The corresponding Sig-$W_1$ distance is given 
\begin{equation*}
\begin{split}
  	\text{Sig-}W_{1}(\mu(S_{\text{past}}), \nu(S_{\text{past}}))
  	&= \|\mathbb{E}_{S \sim \bm{\mu}(S_{\text{past}})}[S] - \mathbb{E}_{S \sim \bm{\nu}(S_{\text{past}})}[S]\|_{\mathrm{p}} \\
  	& =
  	\|\mathbb{E}_{S \sim \bm{\mu}}[S_{\text{future}} \mid S_{\text{past}} ] - \mathbb{E}_{S \sim \bm{\nu}}[S_{\text{future}}\mid S_{\text{past}}]\|_{\mathrm{p}}.
  	\,.  
\end{split}
\end{equation*}
where $S$ denotes $(S_{\text{past}}, S_{\text{future}})$.
From Doob-Dynkin Lemma we know that the conditional expectations are measurable functions of $S_{\text{past}}$. Assuming the continuity of conditional expectation, and by the universal approximation results, these can be approximated arbitrarily well by linear functional of signature. Hence we have \begin{eqnarray}\label{eqn:conditionalExpReg}
\mathbb E_{S \sim \bm{\mu}}[ |S_{\text{future}} -  \mathbb{E}_{\bm {\mu}(S_{\text{past}})}[S_{\text{future}} ]  |^2 ] \approx \inf_{L \,\, \text{linear functional}}
\mathbb E_{S \sim \bm{\mu}}  [ | S_{\text{future}} -  L(S_{\text{past}})  |^2 ]\,.
\end{eqnarray}

Due to linearity of the functional $L$, the solution of the above optimization problem can be estimated by linear regression.
 
Let $\hat{\bm{L}}$ denote the linear regression estimator of the conditional expectation  $x \mapsto \mathbb{E}_{S \sim \bm{\mu}}[S_{\text{future}} |S_{\text{past}} = x]$.

Unlike classical $W_1$-GAN described above, the conditional expectation under the data measure needs to be computed only once. Complete training is then reduced to solving following supervised learning problem 

\begin{equation} \label{eq sigcwgan loss}
    \ell(\theta^{(g)}) := \mathbb E^{X_{\text{past}}} \left[ || \hat{\bm{L}}(S_{\text{past}}) - \mathbb{E}_{S_{\text{future}} \sim \bm{\nu}(\theta^{(g)}, X_{\text{past}})}[S_{\text{future}}]||_{\mathrm{p}}  \right]\,.
\end{equation}

Note that for each $\theta^{(g)}$ one needs to approximate $\mathbb{E}_{S \sim \bm{\nu(\theta^{(g)})}}[S_{\text{future}}] $ using Monte Carlo simulations. A complete approximation algorithm also requires Monte Carlo approximation of outer expectation and truncation of the signature map (see Section \ref{subsec_CSIGWGN} for exact details). The flowchart of SigCWGAN algorithm is given in Figure \ref{fig:flowchar_sigcgn}. 
\begin{figure}[!ht]
    \centering    
    \includegraphics[width =1 \textwidth]{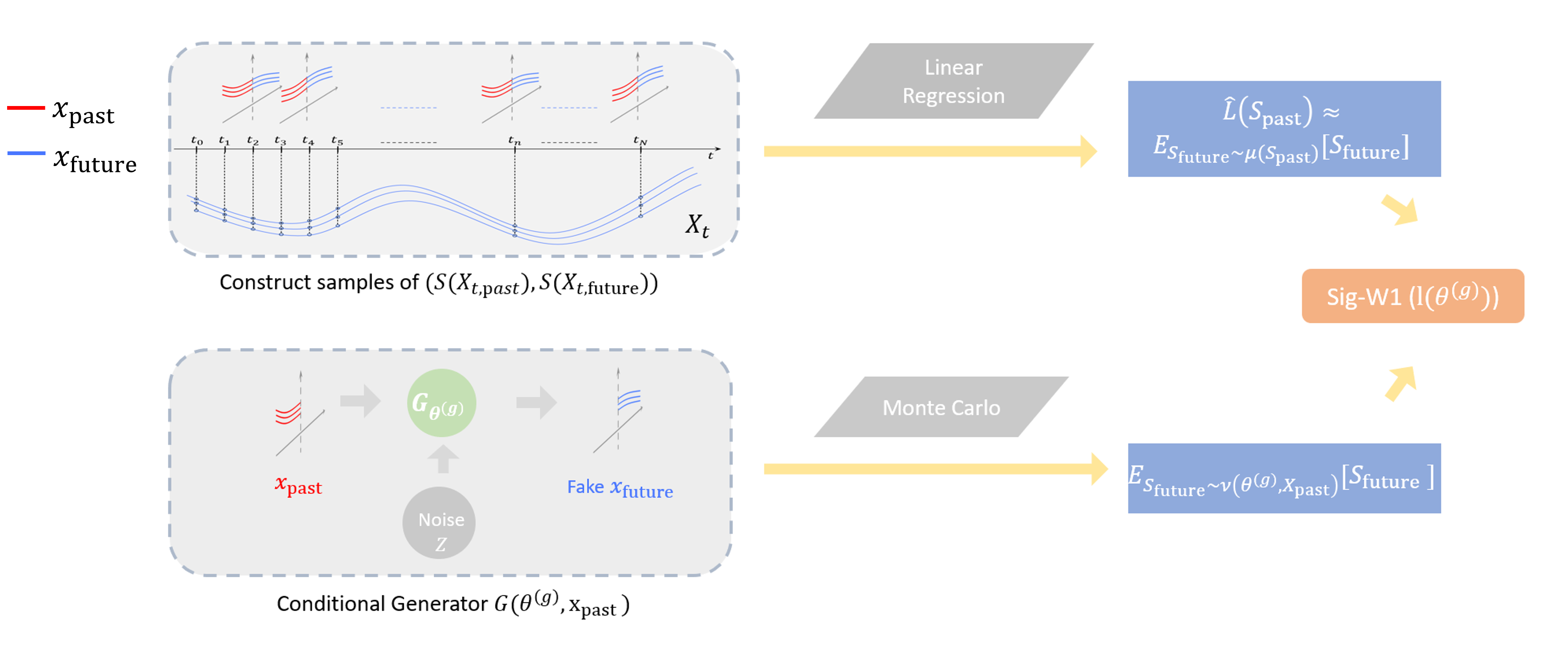}
    \caption{The illustration of the flowchart of SigCWGAN.}
    \label{fig:flowchar_sigcgn}
\end{figure}
    
\subsection{Related work}
\label{sec:related_work}

In the time series domain, the unconditional generative model was approached by various works such as \cite{koshiyama2019generative, wiese2020quant}. Among the signature-based models, \cite{kidger2019deep} used Sig-MMD, originated in \cite{chevyrev2018signature}, a version of the maximum mean discrepancy (MMD) with the signature feature, to generate the Ornstein–Uhlenbeck process. Independently, \cite{ni2021sig} proposed the Sig-Wasserstein GAN motivated by combining the Wasssertain-1 distance and the signature feature. Also the conditional generative objective was approached by various authors. \cite{esteban2017realvalued, koochali2020if, fu2019time, wiese2019DH} used FNNs / LSTMs with recurrent conditional GANs (RCGANs), \cite{donahue2019AdversarialAS, engel2019GANSynthAN} use GANs to generating log-magnitude spectrograms and phases directly for audio synthesis, and \cite{blanka2020} pair log-signatures with variational autoencoders (VAEs) and formulate a conditional generator in log-signature space. Conditional VAEs with the log-signature in \cite{blanka2020} are well adapted to small data environment, but it may require an additional step of inverting synthetic log-signature to the path for time series generation. TimeGAN \cite{yoon2019time} demonstrates the improvement by adding the supervised loss to the adversarial loss to force network to adhere to the dynamics of the training data during sampling. The supervised loss of TimeGAN is defined in terms of the sample-wise discrepancy between the true latent variable $h_{t+1}$ and the generated one-sample estimator $\hat{h}_{t+1}$ given $h_t$. However, even if the estimator $\hat{h}_{t+1}$ has the same conditional distribution as $h_{t+1}$, the supervised loss may not be equal to zeros, and hence it suggests that the proposed loss function might not be suitable to capture the conditional distribution of the latent variable $h_{t+1}$ given the $h_{t}$.

Conditional moment matching network (CMMN) introduced in \cite{ren2016conditional} derives the conditional MMD criteria based on the kernel mean embedding of conditional distributions, which avoids the approximation issues mentioned in the above conditional WGANs. However, the performance of CMMN depends on the kernel choice and it is yet unclear how to choose the kernel on the path space. While our SigWGAN method is built on the conditional WGANs and the signature features, we would like to highlight the difference of method to the conditional WGAN and its link to CMMD. SigCWGAN resolves the computational bottleneck of the conditional WGANs given the past time series by using the analytic formula for the conditional discriminator without training. Building upon \cite{ni2021sig}, our work expands the SigWGAN framework from its initial application to unconditional generative models to enable conditional generative modelling. Moreover, one can view the SigCWGAN as the combination of unnormalized Sig-MMD (\cite{chevyrev2018signature}) and CMMD, which has not been explored in the literature. It is worth noting that we also extend the definition of Sig-$W_{1}$ in \cite{ni2021sig}, from the $l^2$ norm of the signature space to the general $l^{\mathrm{p}}$ for some $\mathrm{p} >1$.

\begin{table}[!ht]
    \centering
    \begin{tabular}{l|l}
    \hline
       Symbol  & meaning \\
       \hline
       $E$ &  $E = \mathbb{R}^{d}$ \\
       $T((E))$& the tensor algebra space of $E$\\
        $\bf{T}^{\mathrm{p}}(E)$&the set of all the elements in $T((E))$ with finite $l^{\mathrm{p}}$ norm \\
        $||.||_{\mathrm{p}}$ & the $l^{\mathrm{p}}$ norm on $\mathbf{T}^{\mathrm{p}}(E)$ \\
        $||.||_{\mathrm{q}}$ & the $l^{\mathrm{q}}$ norm on $T((E))^{*}$ \\
        $X $& $E$-valued time series of length $T$, i.e. $X= (X_1, \cdots, X_T) \in \mathbb{R}^{d \times T}$\\
        $X_{t, \text{past}}$& the $\pp$ lagged values of $X_t$, i.e. $(X_{t-p+1}, \cdots, X_{t}) \in \mathbb{R}^{d \times \pp} =: \mathcal{X}$.\\
        $X_{t, \text{future}}$& the next $\qq$ step forecast of $X_t$, i.e. $(X_{t+1}, \cdots, X_{t+\qq})\in \mathbb{R}^{d \times \qq} =: \mathcal{Y}$.\\
        $\pp$& the window size of the past path $X_{t, \text{past}}$\\
       $\qq$& the window size of the future path $X_{t, \text{future}}$\\
        $S_{t, \text{past}}$& the signature of $X_{t, \text{past}}$\\
        $S_{t, \text{future}}$& the signature of $X_{t, \text{future}}$\\
         \hline
    \end{tabular}
    \caption{Notation summary table}
    \label{tab:my_label}
\end{table}    

\input{Sections/section1_signature}
\input{Sections/section2_SigW1}

\input{Sections/section3_SigCWGAN}

\input{Sections/section4_numerical_results}

\acks{HN is supported by the EPSRC under the program grant EP/S026347/1. HN and LS are supported by the Alan Turing Institute under the EPSRC grant EP/N510129/1. All authors thank the anonymous referees for constructive feedback, which greatly improves the paper. Moreover, HN extends her gratitude to Siran Li, Terry Lyons, Chong Lou, Jiajie Tao, and Hang Lou for their helpful discussion.} 

\paragraph{Data Availability Statement}
The data that support the findings of this study are openly available in Conditional-Sig-Wasserstein-GANs repository at \url{https://github.com/SigCGANs/Conditional-Sig-Wasserstein-GANs}.  These empirical data were derived from the following resources available in the public domain: (1) the Oxford-Man Institute's "realised library"  \url{https://realized.oxford-man.ox.ac.uk/data};(2) \url{https://github.com/David-Woroniuk/Historic_Crypto}.

\bibliography{myref}

\newpage
\input{Sections/section5_appendix}

\end{document}

%% file: Sections/section1_signature.tex
\section{SIGNATURES and EXPECTED SIGNATURES}
\label{sec:sigs_and_expected_sigs}
In order to introduce formally the optimal conditional time series discriminator, in this section we recall basic definitions and concepts from rough path theory. 
\subsection{Tensor algebra space}
We start with introducing the tensor algebra space of $E$, where the signature of a $E$-valued path takes values. For simplicity, fix $E=\mathbb{R}^{d}$ throughout the rest of the paper. $E$ has the canonical basis $\left\{e_{1},\ldots ,e_{d}\right\} $. Consider the successive tensor powers $E^{\otimes n}$ of $E$.\footnote{The tensor power $E^{\otimes n}$ is defined based on the concept of the tensor product. Consider two vector spaces $V$ and $W$ over the same field $F$ with basis $B_V$ and $B_W$, respectively. The tensor product of $V$ and $W$, denoted by $V \otimes W$, is a vector space consisting of basis $b \otimes b'$, where $b \in B_V$ and $b' \in B_W$ that is equipped with a bilinear map $\otimes$. Here $b \otimes b'$ can be regarded as a function $V \times W \rightarrow \mathbb{R}$, which maps every $(v, w)$ to $\mathbf{1}_{v=b, w=b'}$. For any two elements $v = \sum_{b \in B_V} v_b ~ b \in V$ and $ w = \sum_{b' \in B_W}  w_{b'}  ~ b' \in W$, then $v \otimes w = \sum_{b\in B_V, b'\in B_W} (v_b w_b) ~ b \otimes b'$.} If one thinks of the elements $e_{i}$ as letters, then $E^{\otimes n}$ is spanned by the words of length $n$ in the letters $\left\{ e_{1},\ldots,e_{d}\right\} $, and can be identified with the space of real homogeneous non-commuting polynomials of degree $n$ in $d$ variables, i.e., $(e_{I}:= e_{i_{1}} \otimes \cdots \otimes e_{i_n})_{I = (i_1, \cdots, i_n) \in \{1, \cdots, d\}^{n}}$. We give the formal definition of the tensor algebra series as follows.


\begin{definition}
The space of all formal $E$-tensors series, denoted by $T\left(\left( E\right) \right)$ is defined to be the following space of infinite series:
\begin{eqnarray*}
T((E)) = \left\{\mathbf{a} = (a_0, a_1, \cdots) \Big \vert a_n \in E^{\otimes n}, \forall n \geq 0 \right\}.
\end{eqnarray*}
It is equipped with two operations, an addition and a product defined as follows: $\forall \mathbf{a} = (a_0, a_1, \cdots ), \mathbf{b} = (b_0, b_1, \cdots) \in T((E))$, it holds that
\begin{eqnarray*}
\mathbf{a}+\mathbf{b} &=& (a_0+b_0, a_1+b_1, \cdots );\\
\mathbf{a} \otimes \mathbf{b} &=&  (c_0, c_1, \cdots).
\end{eqnarray*}
where $c_n = \sum_{j=0}^{n} a_j \otimes b_{n-j}$.
\end{definition}

We endow the space $T((E))$ with the action of $\mathbb{R}$ by $\lambda \mathbf{a} = (\lambda a_0, \lambda a_1, \cdots)$ is a real non-communtative untial algebra with the unit $\mathbf{1} = (1, 0, 0, \cdots)$ \cite{lyons2007differential}. 


Let us first introduce the function $|| \cdot ||_{\mathrm{p}}: T((E)) \rightarrow [0, +\infty]$ for some $\mathrm{p}\geq 1$. For any element $\mathbf{a}:= \sum_{n \in \mathbb{N}} \sum_{I \in \{1, \cdots, d\}^{n}} a_{I} e_I\in T((E))$,
\begin{eqnarray}\label{def_a_lq}
||a||_\mathrm{p} = \left(\sum_{n \in \mathbb{N}} |a_n|_{\mathrm{p}}^{\mathrm{p}}\right)^{1/\mathrm{p}},
\end{eqnarray}
where $|a_n|_{\mathrm{p}} = \left(\sum_{I \in \{1, \cdots, d\}^{n}} |a_{I}|^{\mathrm{p}}\right)^{\frac{1}{\mathrm{p}}}$.

Similarly, we define the map $||\cdot ||_{\mathrm{q}}: T((E))^{*} \rightarrow [0, + \infty]$. Define the canonical basis of the dual space $T((E))^{*}$, i.e. $(e_{I}^{*})_{I = (i_1, \cdots, i_n) \in  \{1, \cdots, d\}^{n}, n \in \mathbb{N}}$
by $\langle e_{I_1}^{*}, e_{I_{2}} \rangle  = \mathbf{1}_{I_1 = I_2}$. For any $L \in T((E))^{*}$, one can write 
\begin{eqnarray*}L = \sum_{n \in \mathbb{N}} \sum_{I \in \{1, \cdots, d\}^{n}} L_I e^{*}_{I}.
\end{eqnarray*} 
Then $||L||_{\mathrm{q}}$ is defined as 
\begin{eqnarray}\label{def_L_lq}
||L||_\mathrm{q} = \left(\sum_{n \in \mathbb{N}} \sum_{I = (i_1, \cdots, i_n) \in \{1, \cdots, d\}^{n}} |L_{I}|^{\mathrm{q}}\right)^{1/\mathrm{q}}.
\end{eqnarray}

In particular, we consider the subspace $\mathbf{T}^{\mathrm{p}}(E)$, consisting with all the elements $a \in T((E))$ with finite $||a||_{\mathrm{p}}$. In this case, $|| \cdot ||_{\mathrm{p}}$ becomes the $l^{\mathrm{p}}$ norm of $\mathbf{T}^{\mathrm{p}}((E))$. 

\begin{definition}\label{def_lp_tensor_space}
Fix some $\mathrm{p} \geq 1$. We denote by $\mathbf{T}^{\mathrm{p}}(E)$ the following space equipped with the $l^{\mathrm{p}}$ topology:
\begin{eqnarray*}
\mathbf{T}^{\mathrm{p}}(E):=\left\{ \mathbf{a} \in T((E)) \big \vert \quad ||\mathbf{a}||_{\mathrm{p}} < + \infty \right\}.
\end{eqnarray*}
Furthermore, we define
\begin{eqnarray*}
\tilde{T}^{\mathrm{p}}(E) = \{\mathbf{a} \in \mathbf{T}^{\mathrm{p}}(E) \big \vert a_0 = 1 \}.
\end{eqnarray*}
\end{definition}
In practice, instead of the signature (an infinite series of $E$-tensors), we often work with the truncated signature. Hence we introduce the corresponding truncated tensor algebra space.

\begin{definition}
Let $n\geq 1$ be an integer. Let $B_{n}=\{\mathbf{a}%
=(a_{0},a_{1},...)|a_{0}=...=a_{n}=0\}.$The truncated tensor algebra $%
T^{(n)}(E)$ of order $n$ over $E$ is defined as the quotient algebra
\begin{equation}
T^{(n)}(E)=T\left( \left( E\right) \right) / B_{n}.
\end{equation}%
The canonical homomorphism $T\left( \left( E\right) \right) \longrightarrow
T^{(n)}(E)$ is denoted by $\pi _{n}$.
\end{definition}

\subsection{Signature of time series}
\paragraph{Embed time series in the path space}\label{subsec:ts2path}
The signature feature takes a continuous function perspective on discrete time series. It allows the unified treatment on irregular time series (e.g. variable length, missing data, uneven spacing, asynchronous multi-dimensional data) to the path space [\cite{chevyrev2016primer}]. To embed time series to the signature space, we first lift discrete time series to a continuous path of bounded $1$-variation. 

Let $\bar{x} = (x_{t})_{t = 1}^{T} \in \mathbb{R}^{\tilde{d} \times T}$ be a $\tilde{d}$-dimensional time series of length $T$. We embed $\bar{x}$ to $X: [0, T] \rightarrow \mathbb{R}^{d}$ with $d = \tilde{d}+1$ as follows: (1) Interpolate the cumulative sum process of $\bar{x}$ to get the $d$-dimensional piecewise linear path; (2) Add the time dimension to the $0^{th}$ coordinate of $X$. 

Let $\Omega_{0}([0, T], \mathbb{R}^{d})$ denote the space of continuous $d$ dimensional paths of finite $1$-variation starting from the origin, with the $0^{th}$ coordinate being the time dimension. We endow $\Omega_{0}([0, T], \mathbb{R}^{d})$ with the $1$-variation metric.\footnote{One may refer Definition \ref{p_variation}, Appendix \ref{SupplementaryMaterial} for the $p$-variation metric of a path.}
For any $\tilde{d}$-dimensional time series, its embedded path $X$ lives in $\Omega_{0}([0, T], \mathbb{R}^{d})$. Figure (\ref{fig:timeSeries2Path}) gives one concrete example to illustrate the time series embedding. 
\begin{figure}[t]
    \centering
    \includegraphics[width=1\textwidth]{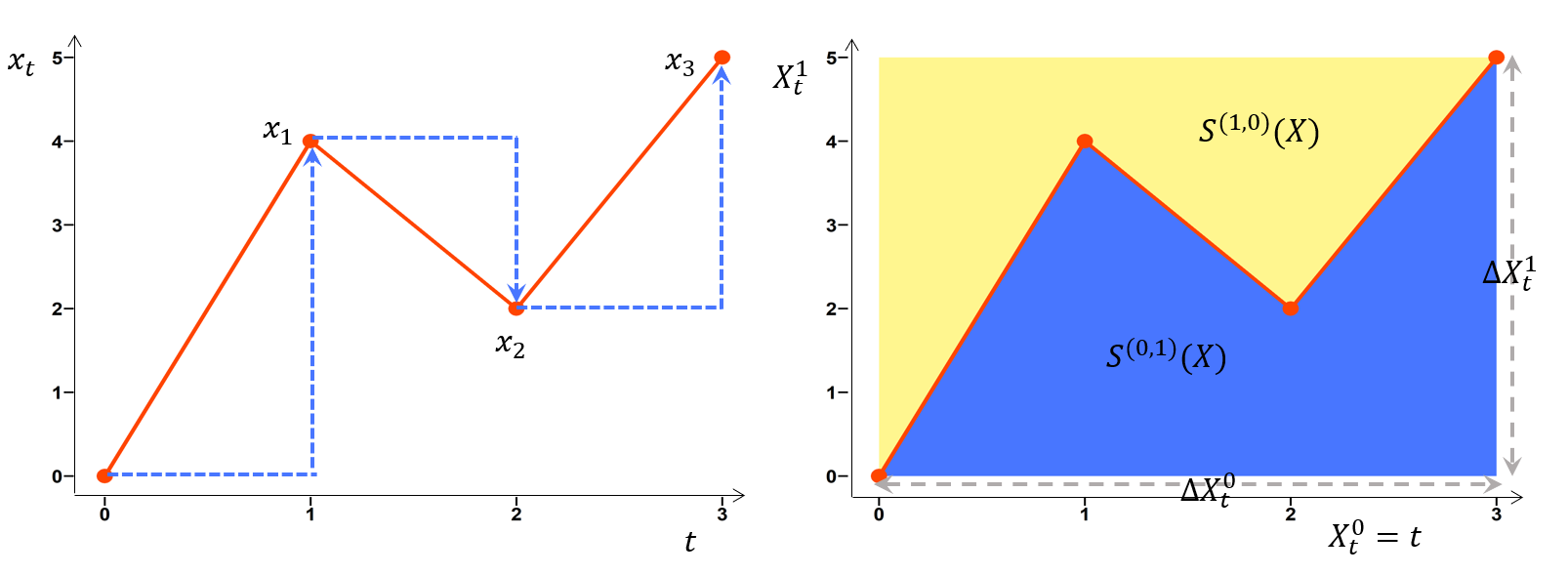}
    \caption{(Left) Embed one dimensional time series $\bar{x} = (x_{1}, x_{2}, x_{3})$ to $X$ (in blue) in the path space. First we compute $(X_{t})_{t = 0}^{3}$, which is the cumulative sum of $(x_{t})_{t  = 1}^{3}$, i.e. $X_{0} = 0$ and $X_{t}^{(1)} = \sum_{i = 1}^{t}x_{i}$, $X_{t}^{(0)} = t$ for $t = 1, 2, 3$. Then we linearly interpolate $X$ to a continuous path in $\Omega_{0}([0, 3], \mathbb{R}^{2})$; (right) Embed the time series to the path space and visualize the low order signature.}
    \label{fig:timeSeries2Path}
\end{figure}

Throughout the rest of the paper, we restrict our discussion on the path space $\Omega_{0}(J, E)$. However, our methodology discussed later can be applied to other methods of transforming discrete time series to the path space provided that the embedding ensures the uniqueness of the signature. The commonly-used path transformations with such uniqueness property are listed in Subsection \ref{subsec:path transformation}.

\paragraph{The signature of the path}
We first introduce the $k$-fold iterated integral of a path $X \in \Omega_{0}([0, T], \mathbb{R}^{d})$. Let $I = (i_{1}, \ldots, i_{k})$ be a multi-index of length $k$, where $i_1, \ldots, i_{n} \in \{0, 1, 2, \cdots, d-1\}$. Let $X^{(i)}$ denote the $i^{th}$ coordinate of $X$, which is a real-valued function. The iterated integral of $x$ indexed by $I$ is defined as 
\begin{eqnarray*}
S^{(i_{1}, i_{2}, \ldots, i_{k})}(X) = \int_{0<t_{1} <t_{2}<\cdots <t_{k}<T} dX_{t_{1}}^{(i_{1})}  dX_{t_{2}}^{(i_{2})} \cdots  dX_{t_{k}}^{(i_{k})}.  
\end{eqnarray*}

Collecting the iterated integrals of $X$ with all possible indices of length $k$ gives rise to the $k^{th}$ fold iterated integral of $X$. It can be also written in the tensor form, i.e.,
\begin{eqnarray*}
\int_{0 \leq t_1 \leq t_2 \leq \cdots \leq t} dX_{t_1} \otimes dX_{t_2} \otimes \cdots \otimes dX_{t_n} \in E^{\otimes n}.
\end{eqnarray*}
Figure \ref{fig:timeSeries2Path} (Left) shows the $1$-fold iterated integral of $X$ which is the increment of $X$, i.e. $X_{3} - X_{0}$, and the $2$-fold iterated integral of $X$ which is given by 
$$
\mathbf{X}^{(2)} = (S^{(0,0)}(X), S^{(0,1)}(X),  S^{(1,0)}(X), S^{(1,1)}(X)), 
$$ where $S^{(i, i)}(X) = \frac{1}{2}(\Delta X^{i})^{2}$ and $S^{(0,1)}(X),  S^{(1,0)}(X)$ are blue and yellow area in Figure \ref{fig:timeSeries2Path} (Right) respectively.\\

Now we are ready to introduce the \emph{signature of a path} $X$.

\begin{definition}[Signature of a path]
Let $X \in \Omega_{0}(J, E)$. The signature of the path $X$ is defined as 
\begin{eqnarray}
S(X) = (1, \mathbf{X}^{(1)}, \mathbf{X}^{(2)}, \ldots) \in T((E)),
\end{eqnarray}
where $\mathbf{X}^{(n)} = \int_{0 \leq t_1 \leq t_2 \leq \cdots \leq t} dX_{t_1} \otimes dX_{t_2} \otimes \cdots \otimes dX_{t_n}$.

The truncated signature of the path $X$ of degree $M$, denoted by $S_{M}(X)$ and defined by
\begin{eqnarray*}
S_{M}(X) :=\pi_{M}(S(X)) =  (1, \mathbf{X}^{(1)}, \cdots, \mathbf{X}^{(M)}).
\end{eqnarray*}
\end{definition}

\begin{lemma}
Fix some $\mathrm{p} \geq 1$. For any $X \in \Omega_0([0, T], \mathbb{R}^d)$, the signature of $X$ is an element of $\mathbf{T}^{\mathrm{p}}(E)$.
\end{lemma}
\begin{proof}
It is a consequence of the factorial decay of the signature of a path of bounded $1-$variation (c.f., Lemma \ref{Lemma_factorial_decay} in Appendix \ref{SupplementaryMaterial}).
\end{proof}

\begin{lemma}[Uniqueness of signature]
For any $X \in \Omega_0([0, T], E)$, the signature of $X$ uniquely determines $X$.    
\end{lemma}
\begin{proof}
We refer the proof to that of Lemma 2.14 in \cite{levin2013learning}. 
\end{proof}

The \textit{universality} and \textit{uniqueness} of signature, described in Section \ref{sec summary}, make it an excellent candidate as a feature extractor of time series. 

As we mainly work on the signature space in the later section, we provide a remark on the structure of the range of $S$ on $\Omega_0([0, T], E)$.
\begin{remark}
Let $\mathcal{K}$ denote a compact set of $\Omega_0([0, T], E)$. Then the range $S(\mathcal{K})$ is a compact set of $\mathbf{T}^{\mathrm{p}}(E)$ endowed with $l^{\mathrm{p}}$ topology. We defer the proof to that of Lemma \ref{Lemma_continuity_sig} at Appendix \ref{SupplementaryMaterial}.
\end{remark}

\subsection{Expected signature}
\paragraph{Expected signature}

Since the signature $S$ is a bijective and continuous map when restricting the path space to $\Omega_{0}([0, T], E)$,  the pushforward of the measure on the path space,  $\bm{\mu}(B) :=(S_{\#}\mu)(B) = \mu(S^{-1}(B))$ for $B$ in the $\sigma$-algebra of $S(\Omega_{0}([0, T], E))$, induces the measure on the signature space.

\begin{lemma}
Let $\mu, \nu$ be two measures defined on the path space $\Omega_{0}(J, E)$. Then for $\bm{\mu}(B) :=(S_{\#}\mu)(B) $ and $\bm{\nu}(B) :=(S_{\#}\nu)(B) $ with $B$ in the $\sigma$-algebra of $S(\Omega_{0}(J,E))$ we have
\begin{eqnarray*}
\bm{\mu} = \bm{\nu} \Longleftrightarrow \mu = \nu.
\end{eqnarray*}
\end{lemma}
\begin{proof}
This is an immediate result of the bijective property of the signature map $S$, when $S$ is restricted to $\Omega_{0}(J, E)$.
\end{proof}
 
By Proposition 6.1 in \cite{chevyrev2016characteristic}, we have the following result:

\begin{theorem}
Let $\mu$ and $\nu$ be two measures on the path space $\Omega_{0}(J, E)$. Let $\bm{\mu}(B) :=(S_{\#}\mu)(B) $ and $\bm{\nu}(B) :=(S_{\#}\nu)(B) $ for $B$ is in the $\sigma$-algebra of $S(\Omega_{0}(J,E))$. Suppose that $\mathbb{E}_{S \sim \bm{\mu}}[S]$ exists and has infinite radius of convergence\footnote{The definition of infinite radius of convergence of expected signature can be found in Definition \ref{Def_ROC_ES} of Appendix A.}. If $\mathbb{E}_{S \sim \bm{\mu}}[S] = \mathbb{E}_{S \sim \bm{\nu}}[S]$, then $\mu = \nu$.
\end{theorem}

In other words, under the regularity condition, the distribution $\mu$ on the path space is characterized by $\mathbb{E}_{X \sim \mu }[S(X)]$. We call $\mathbb{E}_{X \sim \mu }[S(X)]$ the expected signature of the stochastic process $X$ under measure $\mu$. Intuitively, the signature of a path plays a role of a non-commutative polynomial on the path space. Therefore the expected signature of a random process can be viewed as an analogy of the moment generating function of a $d$-dimensional random variable. For example, the expected Stratonovich signature of Brownian motion determines the law of the Brownian motion in \cite{lyons2015expected}. However, it is challenging to establish a general condition to guarantee the infinite radius of convergence (ROC). In fact, the study of the expected signature of stochastic processes is an active area of research. For example, the expected signature of fractional Brownian motion for the Hurst parameter $H \geq 1/2$ is shown to have the infinite ROC \cite{PASSEGGERI20201226, fawcett2002problems}, whereas the ROC of the expected signature of stopped Brownian motion up to the first exit domain is finite \cite{boedihardjo2021expected,li2022expected}. Theorem 6.3, \cite{chevyrev2016characteristic} provides a sufficient condition for the infinite ROC of the expected signature, potentially offering an alternative way to show the infinite ROC without directly examining the decay rate of the expected signature.

%% file: Sections/section2_SigW1.tex
\section{SIG-WASSERSTEIN METRIC} \label{sec Sig-Wasserstein metric}
In this section, we formalise the derivation of \textit{Signature Wasserstein-1} (Sig-$W_1$) metric introduced in Section \ref{sec summary}. The Sig-$W_1$ is a generalisation of the one proposed in \cite{ni2021sig} by considering the general $l^{\mathrm{p}}$ metric of the signature space. 

Let $f: \Omega \rightarrow \mathbb{R}$, where $\Omega$ is a generic metric space. Define 
\begin{eqnarray}\label{eqn_lip}
 ||f||_{Lip, \Omega}:= \sup_{x \neq y, x,y \in \Omega} \frac{|f(x) - f(y)|}{D(x, y)},
\end{eqnarray}
where $D$ is a metric defined on $\Omega$. Let $\mu$ and $\nu$ be two compactly supported measures on the path space $\Omega_{0}([0, T], E)$ such that the corresponding induced measures on the signature space $\bm{\mu}$ and $\bm{\nu}$ respectively have a compact support $\mathcal{K} \subset S(\Omega_{0}([0, T], E)) \subset \mathbf{T}^{\mathrm{p}}(E)$. Recall that 
\begin{eqnarray*}
W_1^{\text{Sig}}(\mu, \nu)&:=&W_1(\bm{\mu}, \bm{\nu})= \sup_{||\bm{f}||_{Lip, \mathcal{K}} \leq 1} \mathbb{E}_{S \sim \bm{\mu}}[\bm{f}(S)] - \mathbb{E}_{S \sim \bm{\nu}}[\bm{f}(S)]\,. 
 \end{eqnarray*}
From the definition of the supremum, there exists a sequence of $\bm{f}_{n}: \mathcal{K} \rightarrow \mathbb{R}$ with bounded Lipschitz norm along which the supremum $W^{\text{Sig}}_{1}(\mu, \nu)$ is attained. By the universality of the signature, it implies that for any $\epsilon>0$, for each $\bm{f}_{n}$, there exists a linear functional $\bm{L}_{n}: \mathcal{K}\rightarrow \mathbb{R}$ to approximate $\bm{f}_{n}$ uniformly, i.e.
\begin{eqnarray*}
\vert  \int_{\mathcal{K}} \bm{f}_{n}(S) \bm{\mu}(dS) - \bm{f}_{n}(S)\bm{\nu}(dS) - \left(\int_{\mathcal{K}} \bm{L}_{n}(S) \bm{\mu}(dS) - \bm{L}_{n}(S)\bm{\nu}(dS))\right)\vert \leq 2 \epsilon.
\end{eqnarray*}
As $\bm{L}_{n}: \mathcal{K} \rightarrow \mathbb{R}$ is linear, there is a natural extension of $\bm{L}_n$ mapping from $\mathbf{T}^{\mathrm{p}}(E)$ to $\mathbb{R}$.

Motivated by the above observation, to approximate $W_1^{\text{Sig}}(\mu, \nu)$, we restrict the admissible set of $\bm{f}$ to be \emph{linear} functionals $\bm{L}: T((E))\rightarrow \mathbb{R}$, which leads to the following definition:
\begin{definition}[$\text{Sig-}W_1$ metric]
For two measures $\mu, \nu$ on the path space $\Omega_{0}([0, T], E))$ such that their induced measures $\bm{\mu}$ and $\bm{\nu}$ respectively has  a compact support $\mathcal{K} \subset \mathcal{S}(\Omega_{0}([0, T], E))$,
\begin{eqnarray*}
\text{Sig-}W_{1}(\mu, \nu) = \sup_{\vert\vert \bm{L} \vert \vert_{Lip} \leq 1, L \text{ is a linear functional:}}  \left(\mathbb{E}_{ S \sim{\bm{\mu}}}[\bm{L}(S)] - \mathbb{E}_{S \sim \bm{\nu}}[\bm{L}(S)]\right). 
\end{eqnarray*}
\end{definition}
Here we skip the domain $\mathbf{T}^{\mathrm{p}}(E)$ in the Lip norm of $|| \mathbf{L}||_{Lip}$ for the simplicity of the notation. 

\begin{remark}
Despite the motivation of $\text{Sig-}W_1$ from the approximation of $W^{Sig}_1$, it is hard to establish the theoretical results on the link between these two metrics. The main difficulty comes from that the uniform approximation of the continuous function $f$ by a linear map $L$ on $\mathcal{K}$ does not guarantee the closeness of their Lipschitz norms. We conjecture that in general $W_{1}^{Sig}(\mu, \nu)$ is not equal to $\text{Sig-}W_1(\mu, \nu)$. However, it would be interesting but technically challenging to find out the sufficient conditions such that these two metrics coincide. 
\end{remark}

To derive the analytic formulae for the Sig-$W_1$ metric, we shall introduce the following auxiliary lemma on the $l^{\mathrm{p}}$ norm of the tensor space $\mathbf{T}^{\mathrm{p}}(E)$ and its dual space. 

\begin{lemma}\label{Lemma_lp}
Fix $\mathrm{p}, \mathrm{q} > 1$ such that $\frac{1}{\mathrm{p}} + \frac{1}{\mathrm{q}} = 1$ (i.e. $\mathrm{q} = \frac{\mathrm{p}}{\mathrm{p}-1}$).\\
For any linear functional $\bm{L} \in\mathbf{T}^{\mathrm{p}}(E)^{*}$, it holds that
\begin{eqnarray}\label{Eqn_T(E)*_q_norm}
\sup_{||a||_\mathrm{p} = 1}|\bm{L}a| = ||\bm{L}||_\mathrm{q}, 
\end{eqnarray}
Similarly, for any $a \in \mathbf{T}^{\mathrm{p}}(E)$, it holds that
\begin{eqnarray}\label{Eqn_T(E)_p_norm}
\sup_{||\bm{L}||_\mathrm{q} \leq 1}|\bm{L}a| = \sup_{||\bm{L}||_\mathrm{q} = 1}|\bm{L}a| = ||a||_\mathrm{p}.
\end{eqnarray}
\end{lemma}
We refer to the proof of Lemma \ref{Lemma_lp} in Appendix \ref{subsec: sigw1}.
\begin{remark}
The sequence space $L_\mathrm{p}(\mathcal{I})$ is defined as 
\begin{eqnarray*}
L_\mathrm{p}(\mathcal{I}) =\left\{ (a_I)_{I \in \mathcal{I}} \big \vert \sum_{I \in \mathcal{I}} |a_I|^\mathrm{p} < \infty\right\},
\end{eqnarray*}
where $\mathcal{I}$ is a general index set and $\mathrm{p} \geq 1$.
It is well known that the dual space of $L_\mathrm{p}(\mathcal{I})$ for $ \mathrm{p} \geq 1$ has naturally isomorphic to $L_\mathrm{q}(\mathcal{I})$. This isomorphism is exactly the same as the map $L^{*}: \mathbf{T}^{\mathrm{p}}(E) \setminus \{0\} \rightarrow \mathbf{T}^{\mathrm{p}}(E)^{*} \setminus \{0\}: a \to L^{*}(a)$ used in our proof. Similarly, the dual space of $L_\mathrm{p}(\mathcal{I})^{*}$ has a natural isomorphism with $L_\mathrm{p}(\mathcal{I})$ for any $\mathrm{p}>1$.
\end{remark}

By exploiting the linearity of the functional $\bm{L} \in \mathbf{T}^{\mathrm{p}}(E)^{*}$, we can compute the Lip norm of $L$ analytically for $D$ being the $l^\mathrm{p}$ norm of $\mathbf{T}^{\mathrm{p}}(E)$ without the need of numerical optimization. By Lemma \ref{Lemma_lp}, the Lip norm of $\bm{L}$ is the $L_\mathrm{p}$ norm of $\bm{L}$, given as
\begin{eqnarray*}
||\bm{L}||_{Lip} := \sup_{x \neq y, x, y \in \mathbf{T}^{\mathrm{p}}(E)} \frac{|\bm{L}(x-y)|}{ ||x-y||_{\mathrm{p}}}= \sup_{||a||_\mathrm{p} = 1}|\bm{L}a| = ||\bm{L}||_\mathrm{q},
\end{eqnarray*}
where $\frac{1}{\mathrm{p}} + \frac{1}{\mathrm{q}} = 1$ and $D(x, y) = ||x-y||_\mathrm{p}$ with some $\mathrm{p}  > 1$.

The simplification of the Lip norm enables us to derive an analytic formula of the corresponding Sig-$W_1$ metric.

\begin{lemma}\label{Lemma_Sig_W1} 
For two measures $\mu, \nu$ on the path space $\Omega_{0}([0, T], E)$ such that their induced measures $\bm{\mu}$ and $\bm{\nu}$ have a compact support $\mathcal{K} \subset S(\Omega_{0}([0, T], E))$. Then it holds that
\begin{equation}
	\label{eq:sigw1}
	\text{Sig-}W_{1}(\mu, \nu)= \|\mathbb{E}_{S \sim \bm{\mu}}[S] - \mathbb{E}_{S \sim \bm{\nu}}[S]\|_\mathrm{p} = \|\mathbb{E}_{X \sim \mu}[S(X)] - \mathbb{E}_{X \sim \nu}[S(X)]\|_\mathrm{p}. 
	\end{equation}
\end{lemma}
\begin{proof}
Let a linear functional $\bm{L}: \mathcal{K} \rightarrow \mathbb{R}$ endowed with the Lip norm when $D(x, y) = ||x-y||_{\mathrm{p}}$.  In this case, the Lip norm coincides with  $l^{\mathrm{q}}$ norm. The compact support $\mathcal{K}$ of $\bm{\mu}$ and $\bm{\nu}$ ensures that $\mathbb{E}_{S \sim \bm{\mu}}(S)$ and $\mathbb{E}_{S \sim \bm{\mu}}(S)$ are in $\mathbf{T}^{\mathrm{p}}(E)$. Let $a:= \mathbb{E}_{S \sim \bm{\mu}}(S) - \mathbb{E}_{S \sim \bm{\nu}}(X))$ and $a = (a_I)_{I}$. Then by Lemma \ref{Lemma_lp}, one can derive the analytic formula of Sig-$W_1$ metric as follows: 
\begin{eqnarray*}
\text{Sig-}W_1(\mu, \nu) = \sup_{||\bm{L}||_{\mathrm{q}} \leq 1} \bm{L}(\mathbb{E}_{S \sim \bm{\mu}}(S)) -  \bm{L}(\mathbb{E}_{S \sim \bm{\nu}}(S)) = \sup_{||\bm{L}||_{\mathrm{q}} \leq 1} \bm{L}(a) = ||a||_{\mathrm{p}}.
\end{eqnarray*}
\end{proof}

\begin{remark}
When $\mathrm{p}=2$, $\text{Sig-}W_{1}$ is the same as the unnormalised Sig-MMD metric proposed in \cite{chevyrev2018signature}. The theoretical results in \cite{chevyrev2018signature} might be useful for studying the properties of Sig-Wasserstein metric. 
\end{remark}

Throughout the rest of the paper, by default, we use $\text{Sig-}W_{1}$ metric when $D$ is $l^2$ norm on $T((E))$, i.e. $\text{Sig-}W_1(\mu, \nu) = \|\mathbb{E}_{X \sim \mu}[S(X)] - \mathbb{E}_{X \sim \nu}[S(X)]\|_2$. In practice, we truncate the $\text{Sig-}W_1(\mu, \nu)$ up to degree $M$, i.e.
\begin{eqnarray*}
\text{Sig-}W_1^{M}(\mu, \nu) = \|\mathbb{E}_{X \sim \mu}[S_{M}(X)] - \mathbb{E}_{X \sim \nu}[S_{M}(X)]\|_2.
\end{eqnarray*}

%% file: Sections/section3_SigCWGAN.tex
\section{SIG-WASSERSTEIN GANS FOR CONDITIONAL LAW } \label{sec Sig-Wasserstein GANs for Conditional law}

In this section, we introduce a general framework, so-called Conditional Sig-Wasserstein GAN (SigCWGAN) based on Sig-$W_1$ metric to learn the conditional distribution $ \mu(X_{\text{future}}| X_{\text{past}})$ from data $x$. The C-SigWGAN algorithm is mainly composed of two steps: 
\begin{enumerate}
    \item We apply a one-off linear regression to learn the conditional expected signature under true measure $\mathbb{E}_{X_{\text{future}} \sim \mu(X_{\text{past}})}[S(X_{\text{future}})]$ (see Section \ref{subsec_SuperLearnES});
    \item We solve an optimization problem to find optimal parameters $\theta^{(g)}$ of the conditional generator, when using loss \eqref{eq sigcwgan loss} (see Section \ref{subsec_CSIGWGN}).
\end{enumerate}
In the last subsection of this section, we propose a conditional generator, i.e. AR-FNN generator, which is a non-linear generalization of the classical autoregressive models by using a feed-forward neural network. It can generate the future time series of arbitrary length.

\subsection{Learning the conditional expected signature under the true measure}\label{subsec_SuperLearnES}

The problem of estimating the conditional expected signature under the true measure $\bm{\mu}(S_{\text{past}})$, by Eqn. \eqref{eqn:conditionalExpReg} and the universality of the signature (Theorem \ref{th sig}), can be viewed as a linear regression task, with the signature of the past path and future path respectively (\cite{levin2013learning}). 

More specifically, given a long realization of $x:=(x_1, \cdots, x_T) \in \mathbb{R}^{d \times T}$ and fixed window size of the past and future path $\pp, \qq > 0$, we construct the samples of past / future path pairs $(X_{\text{past}}, X_{\text{future}})$ in a rolling window fashion, where the $i^{th}$ sample is given by
\begin{eqnarray*}
\left(x_{\text{past}}^{(i)}, x_{\text{future}}^{(i)}\right) = (x_{\text{past}, i+\pp-1}, x_{\text{future}, i+\qq-1} ).
\end{eqnarray*}

Assuming stationarity of the time series, the samples of past and future signature pairs are identically distributed
\begin{eqnarray*}
\left(S_{M_{1}}(x^{(i)}_{\text{past}}), S_{M_2}(x^{(i)}_{\text{future}})\right)_{i}\overset{d}{\sim}(S_{M_1}(X_{\text{past}}), S_{M_2}(X_{ \text{future}})),
\end{eqnarray*}
where  $M_{1},M_2$ are the degrees of the signature of the past and future paths, which can be chosen by cross-validation in terms of fitting result. One may refer to \cite{fermanian2022functional} for further discussion on the choice of the degree of the signature truncation.

In principle, linear regression methods on the signature space could be applied to solve this problem using the above constructed data. When we further assume that under the true measure, 
\begin{eqnarray*}
S_{M_2}(X^{(i)}_{\text{future}}) = L(S_{M_1}(X^{(i)}_{\text{past}})) + \epsilon_i,
\end{eqnarray*}
where $\epsilon_i \overset{iid}{\sim } \epsilon$ and $\mathbb{E}[\epsilon_i | X^{(i)}] = 0$, then an ordinary least squares regression (OLS) can be directly used. This simple linear regression model on the signature space achieves satisfactory results on the numerical examples of this paper. But it could be potentially replaced by other sophisticated regression models when dealing with other datasets. 

We highlight that this supervised learning module to learn $\mathbb{E}_{\mu(X_{\text{past}})}[S_{M}(X_{\text{future}})]$ is one-off and can be done prior to the generative learning. It is in striking contrast to the conditional WGAN learning, which requires to learn $\mathbb{E}_{\mu(X_{\text{past}})}[f_{\alpha}(X_{\text{future}})]$ every time the discriminator $f_{\alpha}$ is updated, and hence saves significant computational cost. 

\subsection{Sig-Wasserstein GAN algorithm for conditional law}\label{subsec_CSIGWGN}
We recall that in order to quantify the goodness of the conditional generator $\nu(\theta^{(g)}, x_{\text{past},t}):=G(\theta^{(g)},x_{\text{past}})_{\#}\mu_{z}$, we defined the loss 
\begin{equation*}
    \ell(\theta^{(g)}) := \mathbb E^{X_{\text{past}}} \left[ || \hat{\bm{L}}(S_{\text{past}}) - \mathbb{E}_{S_{\text{future}} \sim \bm{\nu}(\theta^{(g)}, X_{\text{past}})}[S_{\text{future}}]||_{\mathrm{p}}  \right]\,,
\end{equation*}
where $\hat{\bm{L}}$ denotes the linear regression estimator for the conditional expectation  $\hat{\bm{L}}: x \mapsto \mathbb{E}_{S \sim \bm{\mu}}[S_{\text{future}} |S_{\text{past}} = x]$. Given the conditional generator $G(\theta^{(g)},\cdot)$, the conditional expected signature $\mathbb{E}_{X \sim \nu(X_{\text{past}},\theta^{(g)} )}[S(X)]$ can be estimated by Monte Carlo method. We denote by  $\hat{\nu}_i$ the empirical approximation of $\nu(\theta, x_{\text{past}}^{(i)})$,  computed by sampling the future trajectory  $\hat{X}^{(t)}_{t+1: t+\qq}$ using $G(\theta^{(g)},\cdot)$ and a conditioning variable $x_{\text{past}, t}$.  This leads to the following empirical loss function:
\begin{eqnarray}\label{eqn_Loss}
\ell^{(N)}(\theta^{(g)}):= \frac{1}{N}\sum_{i = 1}^{N} || \hat{L}_\mu(S_{M_1}(x_{\text{past}}^{(i)}))) - \mathbb{E}_{\hat{\nu}_i}[S_{M_2}(x_{\text{future}}^{(i)})]||_{\mathrm{p}}\,.
\end{eqnarray}

Using empirical loss function  \eqref{eqn_Loss}, one  updates the  generator parameters $\theta^{(g)}$ with stochastic gradient descent algorithm until it converges or achieves the maximum number of epochs. 
See Algorithm \ref{Algorithm_list} for pseudocode. 

\begin{algorithm}[t]
    \caption{Pseudocode of SigCWGAN}\label{Algorithm_list}
    \hspace*{\algorithmicindent} \textbf{Input}: $(x_{t})_{t = 1}^{T}$, the signature degree of future path $M_2$, the signature degree of past path $M_1$, the length of future path $\qq$, the length of past path $\pp$, learning rate $\eta$, batch size $B$, the number of epochs $N$,  number of Monte Carlo samples $N_{\text{MC}}$.  \\
    \hspace*{\algorithmicindent} \textbf{Output}: $\theta$ - the optimal parameter of the generator $G(\theta,\cdot)$.
    \begin{algorithmic}[1]
\State Compute truncated signature of the past and future paths: $(S_{M_1}(x_{t-\pp+1:t}), S_{M_2}(x_{t+1:t+\qq}))_{t}$. 
\State Compute linear regression coefficient  $\hat{L}$ using $(S_{M_1}(x_{t-\pp+1:t}), S_{M_2}(x_{t+1:t+\qq}))_{t}$. (See Section \ref{subsec_SuperLearnES}.)

\State Initialise the parameters $\theta$ of the generator. 
\For{$i = 1:N$}
\State{\Comment{Denote the set of time index of the batch as $\mathcal{T}_B$.}}
\For{$j = 1:\#$ of batches} 
\State We randomly select the set of time index of batch size $B$, denoted by $\mathcal{T}$.
\State Initialize $l^{(B)}(\theta)  \gets 0$.
\For{$t \in \mathcal{T}_{B}$}
\State Simulate $n_{\text{MC}}$ samples of the simulated future path segments $(\hat{x}^{(j)})_{j = 1}^{n_{\text{MC}}}$ by the generator $G(\theta,\cdot)$ given the past path $x_{t-\pp+1: t}$. 
\State Compute \begin{eqnarray*}
\hat{\mathbb{E}}_{X \sim \nu(\theta, x_{t-\pp+1:t})}[S_{M_2}(X)] \gets \frac{1}{n_{\text{MC}}}\sum_{j = 1}^{n_{\text{MC}}} S_{M_2}(\hat{x}^{(j)}).
\end{eqnarray*}


\State Update $l^{(B)}(\theta) \gets  l^{(B)}(\theta) + \Vert \hat{L}(S_{M_1}(x_{t-\pp+1:t})) - \hat{\mathbb{E}}_{X \sim \nu(\theta, x_{t-\pp+1:t})}[S_{M_2}(X)]\Vert_{2}$.
\EndFor
\EndFor
    \State $\theta \gets \theta - \eta \frac{dl^{(B)}(\theta)}{d\theta}$.
    \EndFor
    \Return $\theta$.
    \end{algorithmic}\label{Algorithm_list}
    \end{algorithm}

\subsection{The Conditional AR-FNN Generator}\label{subsec_generator}
In this subsection, we further assume that the target time series $X$ is stationary and satisfies the following autoregressive structure, i.e.
\begin{eqnarray}\label{eqn_AR_relation}
X_{t + 1} = g(X_{t, \text{past}}, \varepsilon_{t+1}),
\end{eqnarray}
where $f: \mathcal{X} \times \mathcal{Z} \rightarrow \mathbb{R}^{d}$ is continuous and $(\varepsilon_t)_{t}$ are i.i.d. random variables and $\epsilon_t$ and $X_{t, \text{past}}$ are independent. Time series of such kind include the autoregressive model (AR) and the Autoregressive conditional heteroskedasticity (ARCH) model. 

The proposed conditional AR-FNN generator is designed to capture the autoregressive structure of the target time series by using the past path $X_{\text{past}, t}$ as additional input for the AR-FNN generator. The function $f$ in Eqn. \eqref{eqn_AR_relation} is represented by forward neural network with residual connections \cite{he2016deep} and parametric ReLUs as activation functions \cite{he2015delving} (see \autoref{sec:arfnn_architecture} for a detailed description).


We first consider a step-1 conditional generator $G_1(\theta^{(g)},\cdot): \mathbb{R}^{d \times \pp} \times \mathcal{Z} \rightarrow \mathbb R^d$, which takes the past path $x$ and the noise vector $Z_{1}$ to generate a random variable to mimic the conditional distribution of step-1 forecast $\mu(X_{t+1}| X_{\text{past}, t} = x)$. Here the noise vector $Z_{1}$ has the standard normal distribution in $\mathcal{Z} = \mathbb{R}^{d_{Z}}$. 

One can generate the future time series of arbitrary length $\qq \geq 1$ given $x_{\text{past}}$ by applying $G_1(\theta^{(g)},\cdot)$ in a rolling window fashion with i.i.d. noise vector $(Z_{t})_{t}$ as follows. 
Given $x_{\text{past}} = (\mathrm{x}_1, \cdots, \mathrm{x}_{\pp}) \in \mathbb{R}^{d \times \pp}$, we define time series $(\hat{x}_{t})_{t}$ inductively; we first initialize the first $\pp$ term $\hat{x}$ as $x_{\text{past}}$, and then for $t > \pp$, use $G_1(\theta^{(g)},\cdot)$ with the $\pp$-lagged value of $\hat{x}_{t}$ conditioning variable and the noise $Z_{t}$ to generate $\hat{x}_{t+1}$; in formula, 
\begin{eqnarray}\label{eqn_G_theta}
\hat{x}_t =
\begin{cases}
\mathrm{x}_{t}, & \text{ if }t \leq \pp;\\
G_1(\theta^{(g)},\underbrace{\hat{x}_{t-\pp}, \cdots, \hat{x}_{t-1}}_{\pp\text{ lagged values of } \hat{x}_{t}}, Z_{t}),  & \text{ if }t > \pp.
\end{cases}
\end{eqnarray}

Therefore, we obtain the step-$\qq$ conditional generator, denoted by $G_q(\theta^{(g)},\cdot): \mathbb{R}^{d \times \pp} \rightarrow \mathbb{R}^{d \times \qq}$ and defined by 
$x_{\text{past}} \mapsto \left(\hat{x}_{\pp+1}, \cdots, \hat{x}_{\pp+\qq}\right)$,
where $\left(\hat{x}_{\pp+1}, \cdots, \hat{x}_{\pp+\qq}\right)$ is defined in Eqn. \eqref{eqn_G_theta}. We omit $\qq$ in $G_{\qq}$ for simplicity. (See Algorithm \ref{alg_generator} in Supplementary Material.)


%% file: Sections/section4_numerical_results.tex
\section{NUMERICAL EXPERIMENTS }\label{Section_Numerical_Experiment}
To benchmark with SigCWGAN, we consider the baseline conditional WGAN (CWGAN) to compare the performance and training time. Besides, we benchmark SigCWGAN with three representative generative models for the time-series generation, i.e. (1) TimeGAN \cite{yoon2019time}, (2) RCGAN \cite{hyland2018real} - a conditional GAN and (3) GMMN \cite{li2015generative} - an unconditional maximum mean discrepancy (MMD) with Gaussian kernel. For a fair comparison, we use the same neural network generator architecture, namely the 3-layer AR-FNN described in \autoref{sec:arfnn_architecture}, for all the above generative models. Furthermore, we compare the proposed SigCWGAN with Generalized autoregressive conditional heteroskedasticity model (GARCH), which is a popular econometric time series model.

To demonstrate the model’s ability to generate realistic multi-dimensional time series in a controlled environment, we consider synthetic data generated by the Vector Autoregressive (VAR) model, which is a key illustrative example in TimeGAN \cite{yoon2019time}. We also provide two financial datasets, i.e. the SPX/DJI index data and Bitcoin-USD data to validate the efficacy of the proposed SigCWGAN model on empirical applications. The additional example of synthetic data generated by ARCH model is provided in the appendix.

To assess the goodness of the fitting of a generative model, we consider three main criteria (a) the marginal distribution of time series; (b) the temporal and feature dependence; (c) the usefulness \cite{yoon2019time} -  synthetic data should be as useful as the real data when used for the same predictive purposes (i.e. train-on-synthetic, test-on-real).\footnote{To solely focus on the fitting of the conditional law of $X_{\text{future}}$, we use real past paths as the input data of train-on-synthetic experiment. In contrast, the input of train-on-synthetic in  \cite{yoon2019time} are synthetic past path with the goal of assessing the unconditional generation of a long synthetic sequence in terms of its auto-regressive structure. } In the following, we give the precise definition of the test metrics. More specially, we use $\mathcal{D}_{\text{real}}:=(x_{\text{future}}^{(i)})_{i =1}^{N}$ and $\mathcal{D}_{\text{fake}}: = (\hat{x}_{ \text{future}}^{(i)})_{i=1}^{N}$ to compute the test metrics, where $\hat{x}_{\text{future}}^{(i)}$ is a simulated future trajectory sampling by the conditional generator $G(\theta^{(g)}, x^{(i)}_{\text{past}})$. $\mathcal{D}_{\text{real}}$ and $\mathcal{D}_{\text{fake}}$ are the samples of the $\mathbb{R}^{d \times \bar{q}}$-valued random variable $X_{\text{future}}$ under real measure and synthetic measure resp. The test metrics are defined below.
\begin{itemize}
\item \textbf{Metric on marginal distribution}: 
For each feature dimension $i \in \{1, \cdots, d\}$, we compute two empirical density functions (epdfs) based on the histograms of the real data and synthetic data resp. denoted by $\hat{df}_{r}^{i}$ and $\hat{df}_{G}^{i}$. We take the absolute difference of those two epdfs as the metric on marginal distribution averaged over feature dimension, i.e., 
\begin{eqnarray*}
\frac{1}{d}\sum_{i = 1}^{d} \vert \hat{df}_{r}^{i} - \hat{df}_{G}^{i} \vert_{1}. 
\end{eqnarray*}

\item \textbf{Metric on dependency}: \paragraph{(1) Temporal dependency}We use the absolute error of the auto-correlation estimator by real data and synthetic data as the metric to assess the temporal dependency. For each feature dimension $i \in \{1, \cdots, d\}$, we compute the auto-covariance of the $i^{th}$ coordinate of time series data $X$ with lag value $k$ under the real measure and the synthetic measure resp., denoted by $\rho_r^{i}(k)$ and $\rho_{G}^{i}(k)$.  Then the estimator of the lag-1 auto-correlation of the real/synthetic data is given by $\frac{\rho_{r}^{i}(1)}{\rho_{r}^{i}(0)}$/ $\frac{\rho_{G}^{i}(1)}{\rho_{G}^{i}(0)}$. 
The ACF score is defined to be the absolute difference of lag-1 auto-correlation given as follows:
\begin{eqnarray*}
\frac{1}{d}\sum_{i = 1}^{d}\left \vert \frac{\rho_{r}^{i}(1)}{\rho_{r}^{i}(0)} -  \frac{\rho_{G}^{i}(1)}{\rho_{G}^{i}(0)}\right \vert.
\end{eqnarray*}
Note $\rho_r^{i}(k)$ and $\rho_r^{i}(k)$ can be estimated empirically by Eqn. \eqref{eqn: auto_cov_real} and Eqn. \eqref{eqn: auto_cov_sync} in Appendix \ref{sec numerical implementations} resp, which allows us to compute the ACF score on the dataset. In addition, we present the ACF plot, which illustrates the autocorrelation of each coordinate of the time series with different lag values. The synthetic data's quality is evaluated by how closely its ACF plot resembles that of the real data, as it indicates the synthetic data's ability to capture long-term temporal dependencies.

\paragraph{(2) Feature dependency}For $d>1$, we use the $l^1$ norm of the difference between cross correlation matrices. Let $\tau^{i, j}_{r}$ and $\tau^{i, j}_{G}$ denote the correlation of the $i^{th}$ and $j^{th}$ feature of time series under real measure and synthetic measure resp.
The metric on the correlation between the real data and synthetic data is given by $l^{1}$ norm of the difference of two correlation matrices, i.e.
\begin{eqnarray*}
\sum_{i = 1}^{d}\sum_{j = 1}^{d}|\tau^{i, j}_{r} - \tau^{i, j}_{G}|.
\end{eqnarray*}

We defer the estimation of the correlation matrix  $\tau^{i, j}_{r} $ and $\tau^{i, j}_{G}$ from the true data and fake data to Appendix \ref{sec numerical implementations}. 

\item \textbf{$R^2$ comparison}: Following \cite{esteban2017realvalued} and \cite{yoon2019time},  we consider the problem of predicting next-step temporal vectors using the lagged values of time series using the real data and synthetic data. First, we train a supervised learning model on real data and evaluate it in terms of $R^{2}$(TRTR). Then we train the same supervised learning model on synthetic data and evaluate it on the real data in terms of $R^{2}$ (TSTR). The closer two $R^{2}$ are, the better the generative model it is. To assess the performance of the proposed SigCWGAN to generate the longer time series, we consider the $R^2$ score for the regression task to predict the next $q$-step, where $q$ can be even larger than $\qq$.
\end{itemize}
The train and test split is 80\% and 20\% respectively in all the numerical examples. We conduct the hyper-parameter tuning for the signature truncation level. We set $\mathrm{p} = 2$ in the $l^{\mathrm{p}}$ norm used in the $\text{Sig-}W_1$ metric. Appendix \ref{sec:ImplementationCSigWGAN} contains the additional information on implementation details of SigCWGAN, including path transformations and network architecture of the generator. We refer the Appendix \ref{sec numerical implementations} for more details on the evaluation metrics. We also provide the extensive supplementary numerical results of $\operatorname{VAR}(1)$ data, $\operatorname{ARCH}(1)$ data and empirical data in Appendix \ref{sec numerical results}. Implementation of SigCWGAN can be found in \url{https://github.com/SigCGANs/Conditional-Sig-Wasserstein-GANs}. 

\subsection{Synthetic data generated by Vector Autoregressive Model}
In the $d$-dimensional $\operatorname{VAR}(1)$ model, time series $(X_t)_{t = 1}^{T}$ are defined recursively for $t \in \lbrace 1, \dots, T-1\rbrace$ through 
\begin{eqnarray}\label{eqn:var_model}
X_{t+1} = \phi X_{t} +\epsilon_{t+1},
\end{eqnarray}
where $(\epsilon_{t})_{t=1}^{T}$ are iid Gaussian-distributed random variables with co-variance matrix $\sigma \mathbf{1} + (1-\sigma) \mathbf{I}$; $\mathbf{I}$ is a $d \times d$ identity matrix. Here, the coefficient $\phi \in [-1, 1]$ controls the auto-correlation of the time series and $\sigma \in [0, 1]$ the correlation of the $d$ features. In our benchmark, we investigate the dimensions $d = 1, 2, 3$ and various $(\sigma, \phi)$. We set $T = 40000$ and $\bar{p} = \bar{q} = 3$. In this example, the optimal degree of signature of both past paths and future paths is 2.

First, we empirically prove that the proposed SigCWGAN can serve as an enhancement of CGWAN model.  One can see from Figure~\ref{fig:CWGAN_mc}, when the CWGAN training is fed into a more reliable estimator of the conditional mean under real measure $\mathbb{E}_{\mu(x_{\text{past}})}[f(X_{\text{future}})]$, the training tends to converge faster. However, the commonly-used one-sample estimator ($x_{\text{future}}$) in the CWGAN training may suffer from large variance, leading to inefficiency of training. In contrast to it,  the SigCWGAN may alleviate this problem by its supervised learning module. Additionally, the simplification of the min-max game to optimization via SigCWGAN leads to further acceleration and stabliziation of training SigCWGANs, and hence brings the performance boost, as shown in Table \ref{table_Var_dim=3_time}. Figure \ref{fig:CWGAN_mc} illustrates that the SigCWGAN has a better fitting than CWGAN in terms of conditional law as the estimated mean (and standard deviation) is closer to that of the true model compared with CWGAN. Moreover, Table \ref{table_Var_dim=1}- \ref{table_Var_dim=3} show that the SigCWGAN consistently beats the CWGAN in terms of performance for varying $d$, $\phi$ and $\sigma$. 

\begin{figure}[!ht]
    \centering
    \begin{subfigure}{0.4\linewidth}
    \includegraphics[width=1\textwidth]{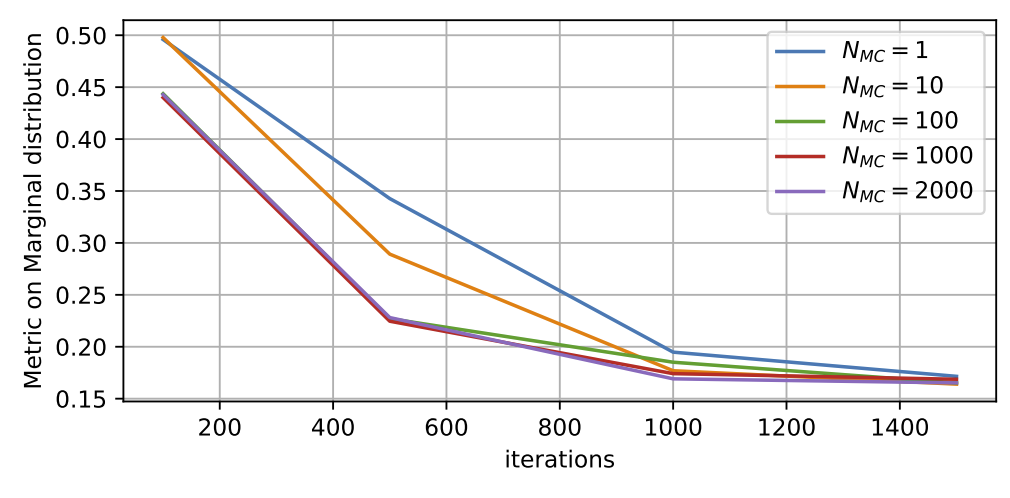}
        \end{subfigure}
            \quad
            \begin{subfigure}{0.49\linewidth}
    \includegraphics[width=1\textwidth]{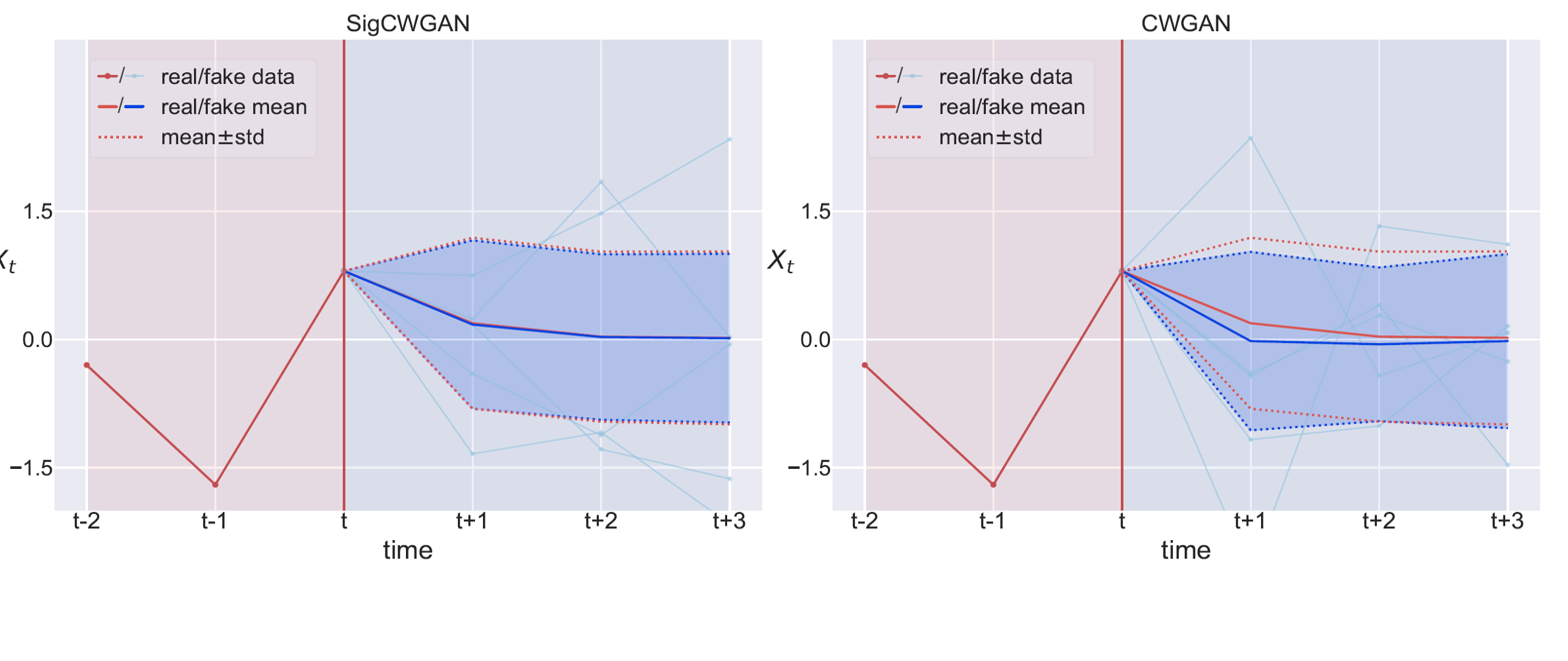}
    \end{subfigure}
    \caption{(Left) Comparison of the performance on CWGAN in terms of metric on the marginal distribution for varying $N_\text{MC}$. This experiment is conducted on a 3-dimensional VAR(1) dataset generated by Eqn.\eqref{eqn:var_model} with $\phi=\sigma=0.8$. We use the Monte-Carlo estimator of the conditional mean ($\mathbb{E}_{\mu(x_{\text{past}})}[f(\theta^{(d)})(X_{\text{future}})]$) generated by the ground truth model for the CWGAN training. The larger number $N_\text{MC}$ of Monte Carlo approximation, the better conditional mean under the true measure. (Right) Comparison of the performance of SigCWGAN and CWGAN in terms of fitting the conditional distribution of future time series given one past path sample on 1-dimensional VAR(1) dataset. 
    }
    \label{fig:CWGAN_mc}
\end{figure}

\input{numerical_results/table_var_dim=3_time}

We then proceed with the comparison of CSigWGN with the other state-of-the-art baseline models.
Across all dimensions, we observe that the CSigWGAN has a comparable performance or outperforms the baseline models in terms of the metrics defined above. In particular, we find that as the dimension increases the performance of SigCWGANs exceeds baselines. We illustrate this finding in Figure \ref{Fig_VAR_R2} which shows the relative error of TSTR $R^2$ when varying the dimensionality of $\operatorname{VAR}(1)$. Observe that the SigCWGAN remains a very low relative error, but the performance of the other models deteriorates significantly, especially the GMMN. 
\begin{figure}[!ht]
   \centering
\includegraphics[width=1\textwidth]{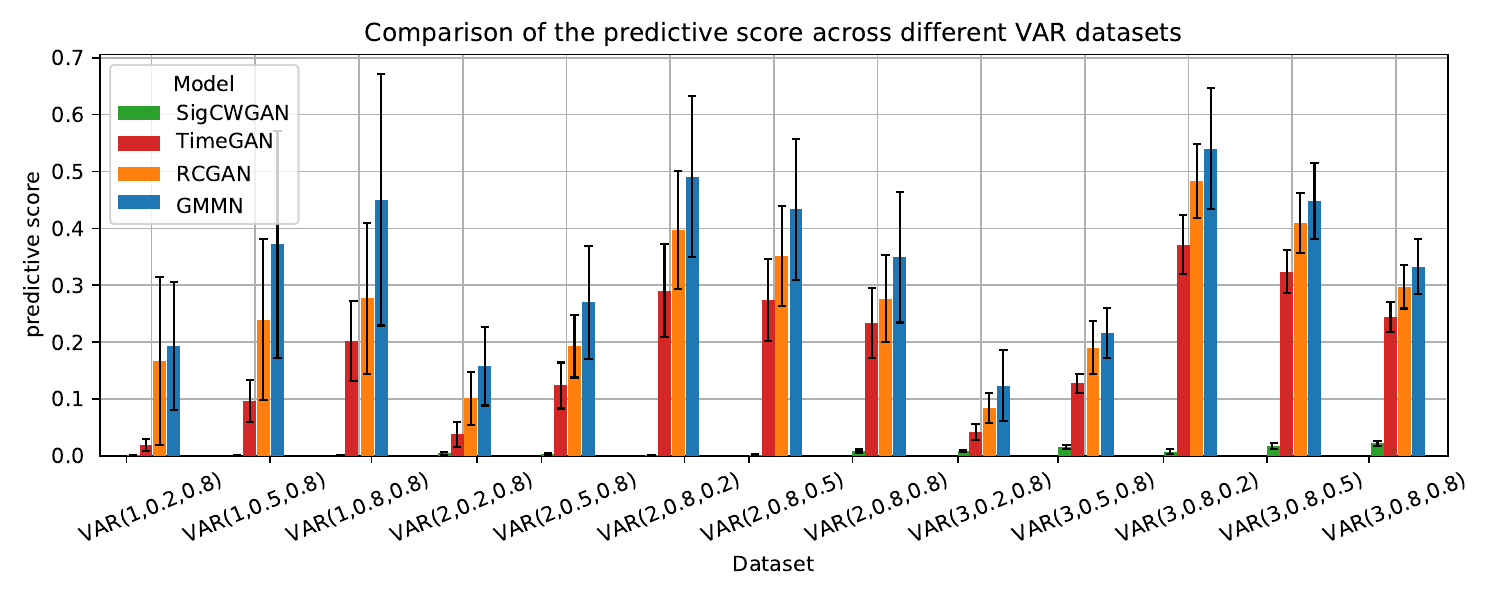}
         \caption{Comparison of predictive score across the VAR(1) datasets. The three numbers in the bracket indicate the hyperparameters $d, \phi, \sigma$ used to generate the corresponding VAR dataset. The predictive score was computed by taking the absolute difference of the $R^2$ obtained from TSTR and TRTR.}\label{Fig_VAR_R2}
\end{figure}

\begin{figure}[!ht]
    \begin{subfigure}{1\linewidth}
    \centering
    \includegraphics[width=\textwidth]{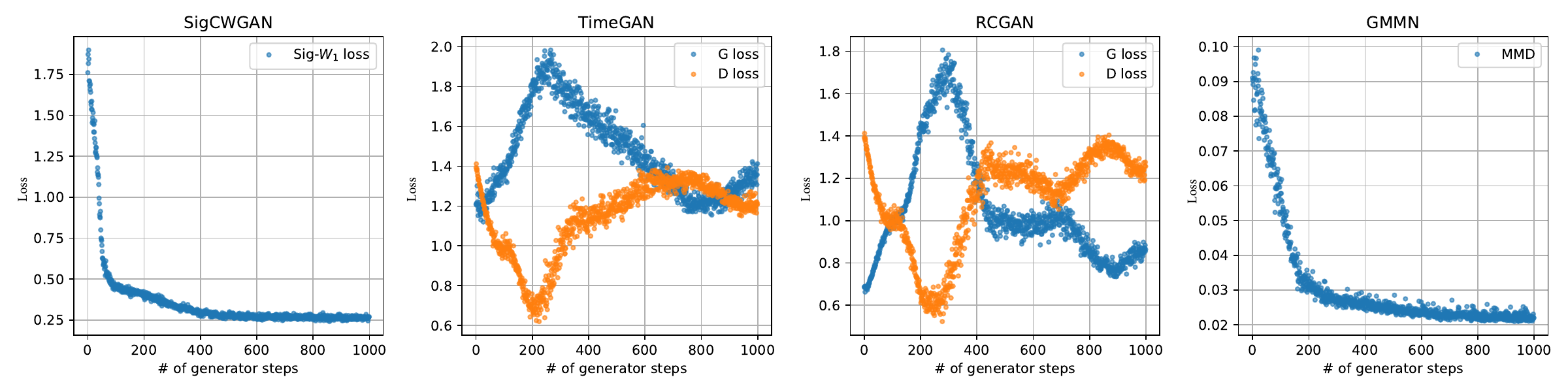}
    \end{subfigure}
    \quad
        \begin{subfigure}{1\linewidth}
            \centering
   \includegraphics[width=\textwidth]{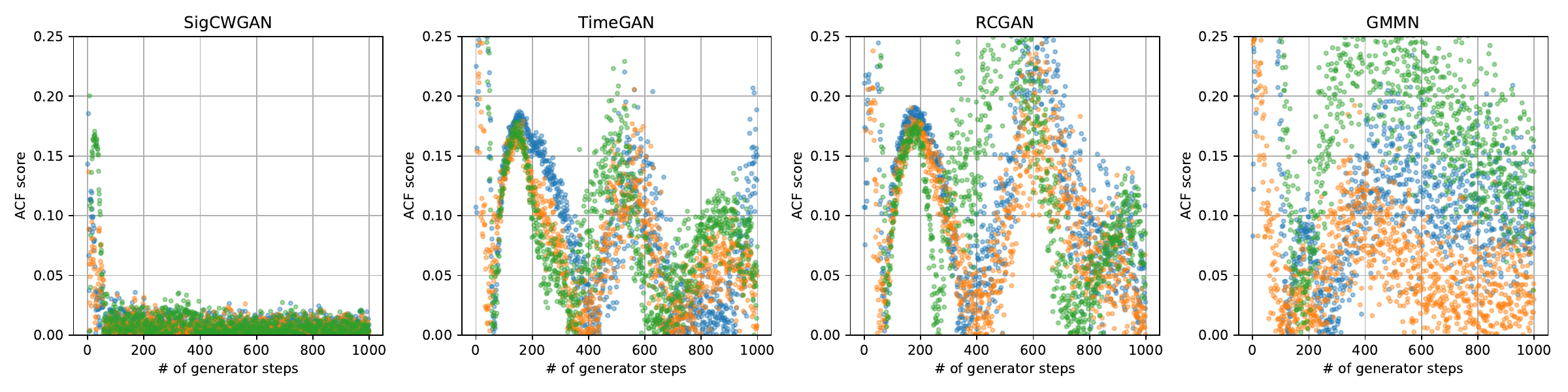}
    \end{subfigure}

    \caption{(Upper panel)  Evolution of the training loss functions. (Lower panel) Evolution of the ACF scores. Each colour represents the ACF score of each feature dimension. Results are for the 3-dimensional $\operatorname{VAR}(1)$ model based on Eqn. \eqref{eqn:var_model} for $\phi=0.8$ and $\sigma=0.8$. } \label{fig:GAN_acf_scores_phi=0.8_sigma=0.8}
\end{figure}

Moreover, we validate the training stability of different methods. Figure \ref{fig:GAN_acf_scores_phi=0.8_sigma=0.8} shows the development of the loss function and ACF scores through the course of training for the 3-dimensional $\operatorname{VAR}(1)$ model. It indicates the stability of the SigCWGAN optimisation in terms of training iterations, in contrast to all the other algorithms, especially RCGAN and TimeGAN that involve a min-max optimisation, as identified in the 1st Challenge in Section~\ref{sec:problem_formulation}. While the ACF scores of the baseline models oscillate heavily, the SigCWGAN ACF score and Sig-$W_1$ distance converge nicely towards zero. Also, although the MMD loss converges nicely towards zero, in contrast, the ACF scores do not converge. This highlights the stability and usefulness of the Sig-$W_1$ distance as a loss function. 

To assess the efficiency of different algorithms, we train all the algorithms for the same amount of time (2 minutes) and compare the test metrics of each method. Table \ref{table_Var_dim=3_time} show a higher efficiency of SigCWGAN, which yields the best performance in terms of all the metrics except for the metric on the marginal distribution. 
 
Furthermore, the SigCWGAN has the advantage of generating the realistic long time series over the other models, which is reflected by that the marginal density function of a synthetic sampled path of 80,000 steps is much closer to that of real data than baselines in Figure \ref{fig:marginal_comparison_var_1__dim3_long}. 
\begin{figure}[!ht]
    \centering
    \includegraphics[width=\textwidth]{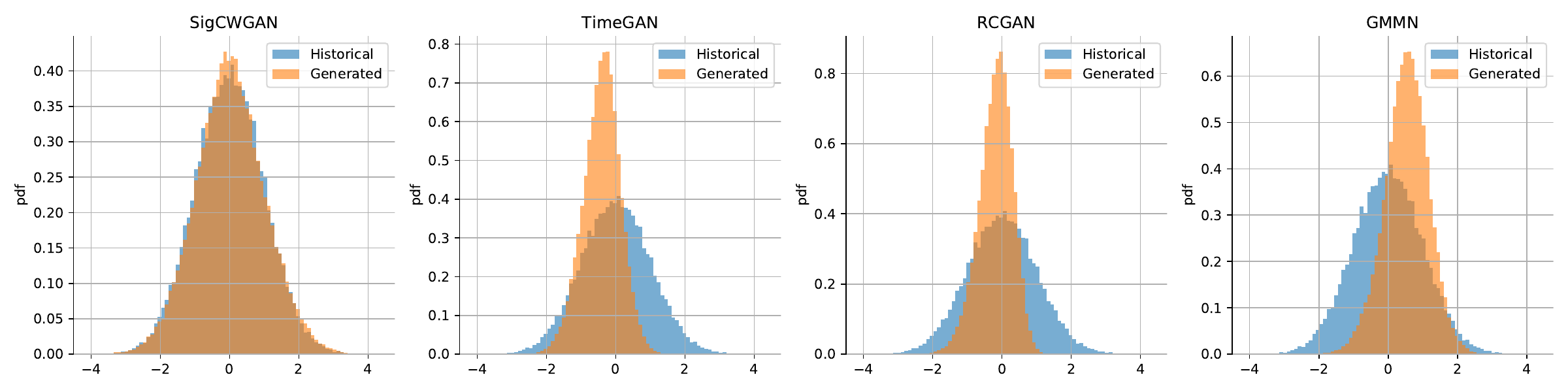}
    \caption{Comparison of the marginal distributions of one long sampled path (80,000 steps) with the real distribution.}
    \label{fig:marginal_comparison_var_1__dim3_long}
\end{figure}

\subsection{SPX and DJI index dataset}

The dataset of the S$\&$P 500 index (SPX) and Dow Jones index (DJI) consists time series of indices and their realized volatility, which is retrieved from the Oxford-Man Institute's "realised library"\cite{heber2009oxford}. We aim to generate a time series of both the log return of the close prices and the log of median realised volatility of (a) the SPX only; (b) the SPX and DJI. Here we choose the length of past and future path to be $3$. By cross-validation, the optimal degree of signature ($M_1=M_2$) is 3 and 2 for the SPX dataset and SPX/DJI dataset, respectively.

Table \ref{table_stocks_reshape} shows that SigCWGAN achieves the superior or comparable performance to the other baselines. The SigCWGAN generates the realistic synthetic data of the SPX and DJI data shown by the marginal distribution comparison with that of real data in Figure \ref{fig:marginal_comparison_SPX_DJI_SigCWGAN}. For the SPX only data, GMMN performs slightly better than our model in terms of the fitting of lag-1 auto-correlation and marginal distribution ($\leq 0.0013$), but it suffers from the poor predictive performance and feature correlation in Table \ref{table_stocks_reshape} and Figure \ref{Fig:Correlation_Comparison_SPX_DJI}. When the SigCWGAN is outperformed, the difference is negligible. Furthermore, the test metrics, i.e. the ACF loss and density metric, of our model are evolving much smoother than the test metrics of the other baseline models shown in Figure \ref{Fig:loss_SPX}. Moreover, the ACF plot shown in Figure \ref{Fig:Correlation_Comparison_SPX_DJI} shows that SigCWGAN has the better fitting for the auto-correlation for various lag values, which indicates the superior performance in terms of capturing long temporal dependency.

It is worth noting that our SigCWGAN model outperforms GARCH, the classical and widely used time series model in econometrics, on both the SPX and SPX/DJI data, as shown in Table \ref{table_stocks_reshape}. The poor performance of the GARCH model could be attributed to its parametric nature and the potential issues of model mis-specification when applied to empirical data. 

\begin{figure}[!ht]
    \centering
    \includegraphics[width=\textwidth]{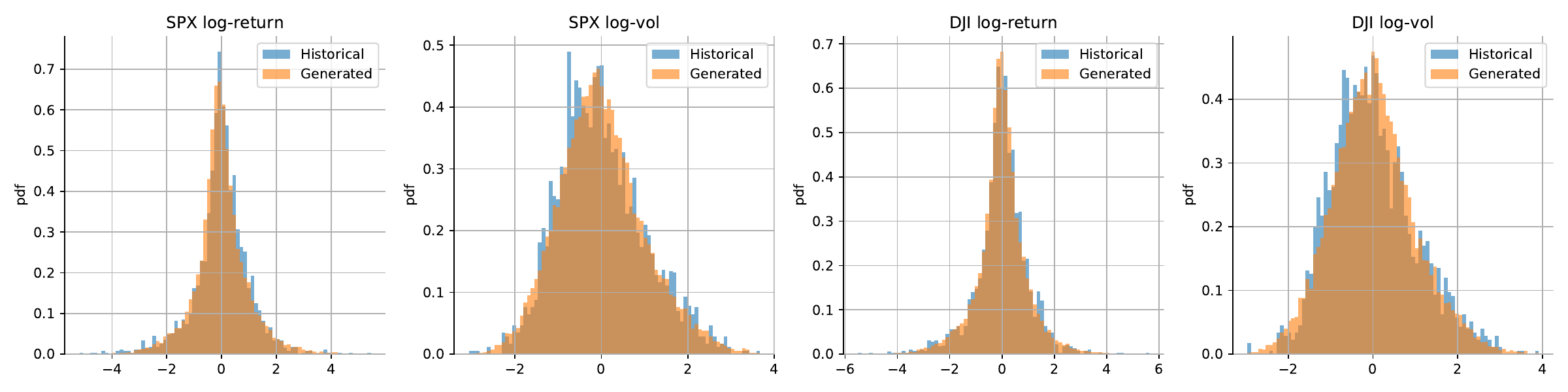}
    \caption{Comparison of the marginal distributions of the generated SigCWGAN paths and the SPX and DJI data.}
    \label{fig:marginal_comparison_SPX_DJI_SigCWGAN}
\end{figure}
\input{numerical_results/table_stock_reshape}
\begin{figure}[H]
\centering
\includegraphics[width=\textwidth]{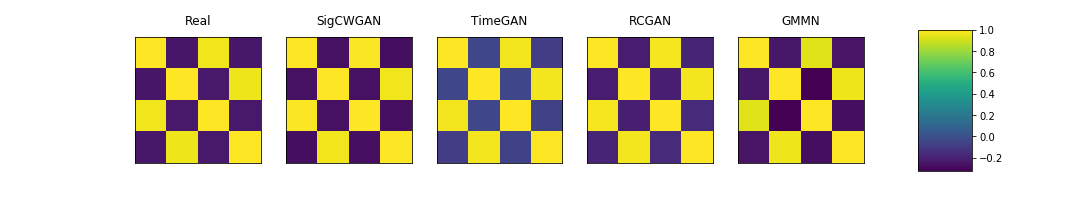}
\caption{Comparison of real and synthetic cross-correlation matrices for SPX and DJI log-return and log-volatility data. On the far left the real cross-correlation matrix from SPX and DJI data is shown.  $x$/$y$-axis represents the feature dimension while the color of the $(i, j)^{th}$ block represents the correlation of $X_{t}^{(i)}$ and $X_{t}^{(j)}$. The color bar on the far right indicates the range of values taken. }\label{Fig:Correlation_Comparison_SPX_DJI}
\end{figure}

 \begin{figure}[t]
        \centering
        \includegraphics[width=0.8\textwidth]{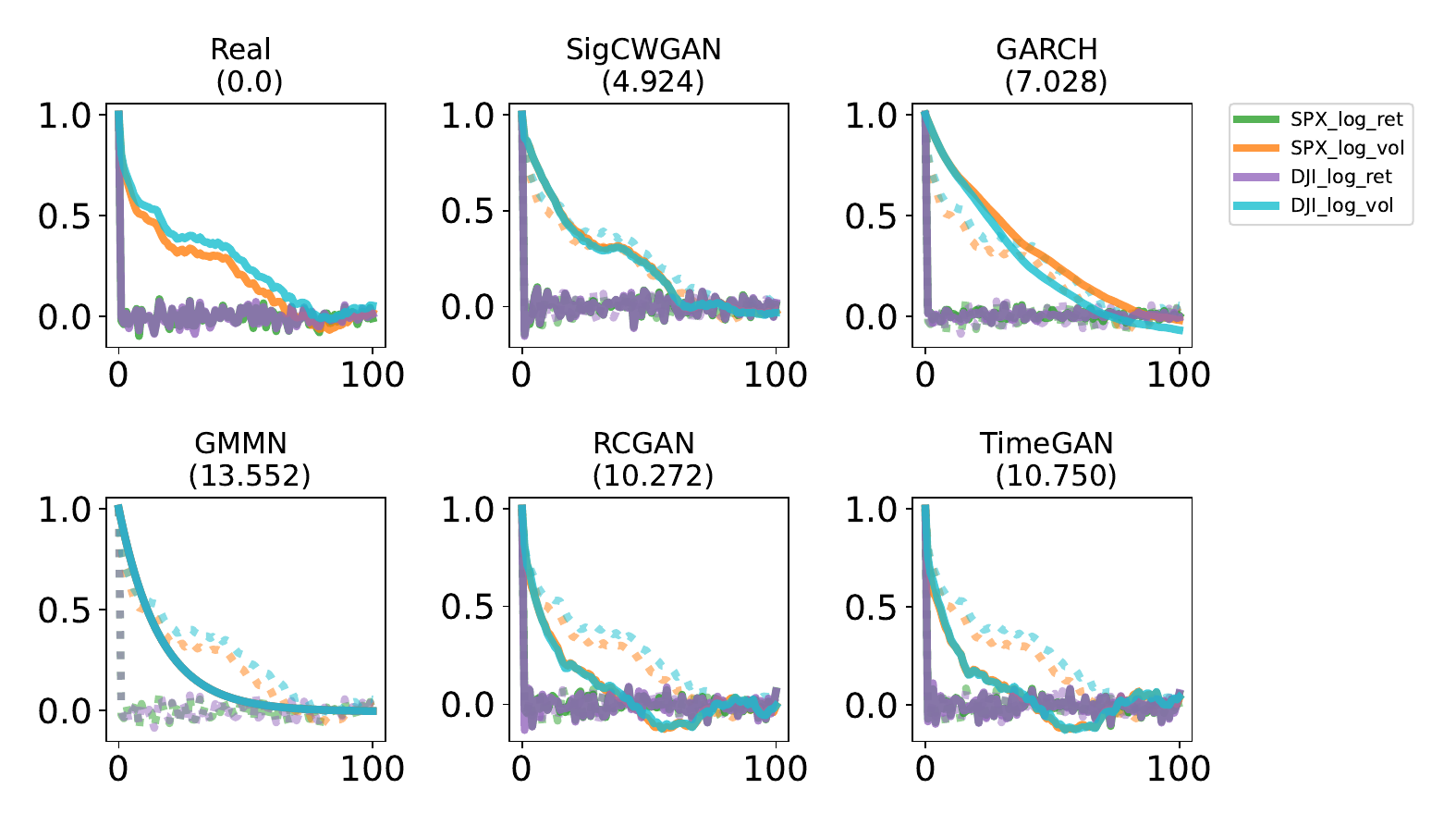}
        \caption{ACF plot for each channel on the SPX/DJI dataset. Here x-axis represents the lag value ( with maximum lag equal to 100) and y-axis represent the corresponding auto-correlation. The length of real/generated time series used to compute the ACF is 1000. The number in the bracket under each model is the sum of the absolute difference between the correlation coefficients computed from real (dashed line) and generated (solid line) samples.}
        \label{fig:long_acf}
    \end{figure}
\subsection{Bitcoin-USD dataset}
The Bitcoin-USD dataset contains hourly data of Bitcoin price in USD from 2021 to 2022. We use the data in 2021 (2022) for the training (testing), respectively, which are illustrated in Figure \ref{Fig:btc-usd}.  We apply our method to learn the future log-return of the future 6 hours given the past 24 hours. We encode the future and past paths with their signatures of depth 4.
Table \ref{table:btc-usd} demonstrates that our proposed SigCWGAN outperforms the other baselines in terms of almost all the test metrics. The $R^{2}$ score of the RCGAN (0.3165) is slightly better than that of the SigCWGAN by 0.0155, whilst SigCWGAN achieves superior performance than the RCGAN in terms of other metrics, especially marginal distribution (2.0532 v.s. 2.803). The better performance of the SigCWGAN to capture the temporal dependency is also verified by the additional results of the autocorrelation metric and the $R^2$-metric for different lag values are provided in tables \ref{table:btc autocorrelation lag} and \ref{table:btc r2 lag}.
\begin{figure}[t]
\centering
\includegraphics[width=\textwidth]{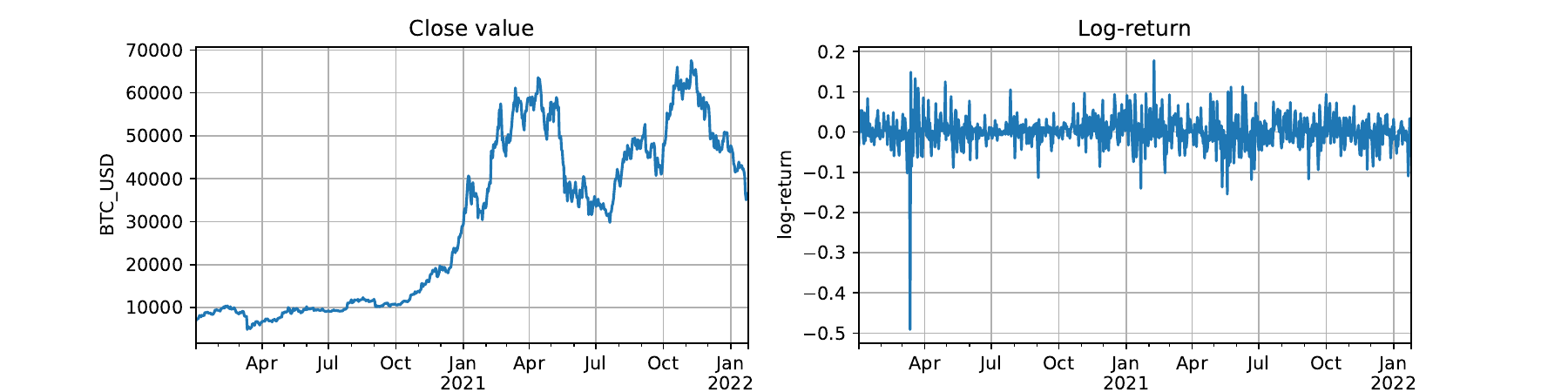}
\caption{The evolution of the close value (Left) and log return of BTC-USD from January 2021 to January 2023.}\label{Fig:btc-usd}
\end{figure}

\newcolumntype{Y}{>{\centering\arraybackslash}X}
\begin{table}[!ht]
\centering
\resizebox{0.9\columnwidth}{!}{%
\begin{tabularx}{1.2\textwidth}{l| Y || Y || Y  || Y}
\toprule
\multicolumn{1}{l|}{Metrics} & \multicolumn{1}{c||}{marginal distribution} &   \multicolumn{1}{c||}{auto-correlation}& 
 \multicolumn{1}{c||}{$R^2(\%)$} & \multicolumn{1}{c}{Sig-$W_1$}    
\\\midrule
SigCWGAN   & \textbf{2.0532}  & 0.091 & 0.3320 & \textbf{0.0829}  \\\midrule
TimeGAN   & 2.8037   & 0.1203   & 0.7582 & 0.1675 \\\midrule
RCGAN     & 2.8603   & \textbf{0.0532}   & \textbf{0.3165}    & 0.0994  \\\midrule
GMMN      & 2.8212   & 0.2093   & 0.3904    & 0.0846   \\\midrule
GARCH      & 4.5063   & 0.0872   & 123.77    & 2.73   \\\bottomrule
\end{tabularx}}
\caption{Numerical results of BTC-USD data experiment. We use the relative error of TSTR $R^2$ against TRTR $R^2$ as the $R^2$ metric. }\label{table:btc-usd}
\end{table}

\section{Conclusion}\label{Section_conclusion}

In this paper, we developed the conditional Sig-Wasserstein GAN for time series generation based on the explicit approximation of $W_{1}$ metric using the signature features space. This eliminates the problem of having to approximate a costly critic / discriminator and, as a consequence, dramatically simplifies training. Our method achieves state-of-the-art results on both synthetic and empirical dataset.

Our proposed conditional Sig-Wasserstein GAN is proved to be effective for generating time series of a moderate dimension. However, it may suffer the curse of dimensionality caused by high path dimension. It might be interesting to explore how to combine SigCWGAN with the implicit generative model to learn the low dimensional latent embedding and hence cope with the high dimensional path case. Moreover, on the theoretical level, it is worthy of investigating the conditions, under which the $W_1$ metric on the signature space coincides with the Sig-$W_1$ metric. 

%% file: numerical_results/table_var_dim=3_time.tex
\newcolumntype{Y}{>{\centering\arraybackslash}X}
\begin{table}[!ht]
\centering
\resizebox{\columnwidth}{!}{%
\begin{tabularx}{1.4\textwidth}{l| Y || Y || Y  || Y}
\toprule
\multicolumn{1}{l|}{Metrics} & \multicolumn{1}{c||}{marginal distribution} &   \multicolumn{1}{c||}{auto-correlation}& 
 \multicolumn{1}{c||}{$R^2(\%)$} & \multicolumn{1}{c}{Sig-$W_1$}    
\\\midrule
SigCWGAN   & 0.0314  &\textbf{0.0085} & \textbf{0.0394} & \textbf{0.4286}  \\\midrule
CWGAN   &0.0086  & 0.0110  & 0.0350  & 0.4384  \\\midrule
TimeGAN   & 0.0243   &  0.0320   & 0.0229  & 0.4680 \\\midrule
RCGAN     & 0.0095   & 0.0332   & 0.0214    & 0.4466  \\\midrule
GMMN      & \textbf{0.0084}   & 0.0298   & 0.0026    & 0.4499   \\\bottomrule
\end{tabularx}}
\caption{Numerical results of $\operatorname{VAR}(1)$ for $d = 3$ with fixed training time of 2 minutes}\label{table_Var_dim=3_time}
\end{table}

%% file: numerical_results/table_stock_reshape.tex

\newcolumntype{Y}{>{\centering\arraybackslash}X}
\begin{table}[!ht]
\centering
\resizebox{\columnwidth}{!}{%
\begin{tabularx}{1.4\textwidth}{l| Y || Y || Y || Y || Y}
\toprule
\multicolumn{1}{l|}{Metrics} & \multicolumn{1}{c||}{marginal distribution} &   \multicolumn{1}{c||}{auto-correlation}& 
 \multicolumn{1}{c||}{correlation} &
 \multicolumn{1}{c||}{$R^2(\%)$} & \multicolumn{1}{c}{Sig-$W_1$}    
\\\midrule
SigCWGAN   & 0.01730, {0.01674}  & 0.01342, \textbf{0.01192}     & \textbf{0.01079}, \textbf{0.07435} &2.996, {7.948}   & \textbf{0.18448}, \textbf{4.36744}                \\\midrule
TimeGAN   & 0.02155, 0.02127    & 0.05792, 0.03035     & {0.12363}, {0.61488} & 5.955, 8.586 & 0.58541, 5.99482\\\midrule
RCGAN     & 0.02094, \textbf{0.01655}   & 0.03362, 0.04075     & 0.04606, 0.15353     & \textbf{2.788}, \textbf{7.190}  & 0.47107, 5.43254                   \\\midrule
GMMN      & 0.01608, 0.02387   & \textbf{0,01283}, 0.02676  & 0.04651, 0.22380   & 9.049, 7.384 & 0.59073, 6.23777   \\\midrule
GARCH&\textbf{0.01583},0.01670 &0.13392, 0.11337&0.15791, 0.7290 & 12.1253, 12.5686&0.64825, 6.15344 \\\bottomrule
\end{tabularx}}
\caption{Numerical results of the stock datasets. In each cell, the left/right number are the result for the SPX data/ the SPX and DJI data respectively. We use the relative error of TSTR $R^{2}$ against TRTR  $R^{2}$ as the $R^2$ metric.}
\label{table_stocks_reshape}
\end{table}

%% file: Sections/section5_appendix.tex
\appendix
\section{PRELIMINARY}\label{SupplementaryMaterial}

For the sake of precision, we start by introducing basic concepts around the signature of a path, which lays the foundation for our analysis on the signature approximation for Wasserstein-1 Distance. We complete this section by providing the proof of Lemma \ref{Lemma_lp}, which is essential for the derivation of the proposed Sig-W$_1$ metric. 

\subsection{Signature of a path}
We introduce the $p$-variation as a measure of the roughness of the path. For ease of notation, let $J$ denote a compact time interval.
\begin{definition} [$p$-Variation]\label{p_variation}
Let $p\geq 1$ be a real number. Let $X: J \rightarrow E$ be a continuous path. The $p$-variation of $X$ on the interval $J$ is defined by 
\begin{equation}
    \vert\vert X\vert\vert_{p,J}=\left[ \sup_{\mathcal{D}\subset J}\sum_{j=0}^{r-1}\left\vert X_{t_{j+1}}-X_{t_j}\right\vert^p\right]^{\frac{1}{p}},
\end{equation}
where the supremum is taken over any time partition of $J$, i.e. $\mathcal{D} = (t_{1}, t_{2}, \cdots, t_{r})$.  \footnote{
Let $J = [s, t]$ be a closed bounded interval. A time partition of $J$ is an increasing sequence of real numbers $\mathcal{D} = (t_{0}, t_{1}, \cdots, t_{r})$ such that $s = t_{0} < t_{1}< \cdots < t_{r} = t$. Let $\vert \mathcal{D} \vert$ denote the number of time points in $\mathcal{D}$, i.e. $\vert \mathcal{D} \vert = r+1$. $\Delta \mathcal{D}$ denotes the time mesh of $\mathcal{D}$, i.e. $\Delta  \mathcal{D}:=\overset{r-1}{\underset{i = 0}{\max}} (t_{i+1} - t_{i})$.}
\end{definition}
Let $\mathcal{C}^{p}(J,E)$ denote the set of all continuous paths mapping from $J$ to $E$ of finite $p$-variation. The larger the $p$-variation is, the rougher a path is. The compactness of the time interval $J$ cannot ensure the finite $1$-variation of a continuous path in general. For example, Brownian motion has $(2+\varepsilon)$-variation a.s $\forall \varepsilon >0$, but it has infinite $p$-variation a.s.$\forall p \in [1, 2]$.

For each $p \geq 1$, the $p$-variation norm of a path $X \in \mathcal{C}_{p}(J, E)$ is denoted by $\vert\vert X \vert\vert_{p-var}$ and defined as follows:
\begin{eqnarray*}
	\vert \vert X \vert\vert_{p-var} = \vert\vert X \vert\vert_{p, J} + \sup_{t \in J} \vert\vert X_{t}\vert\vert.
\end{eqnarray*}

Recall that $\pi_n$ is the projection map from the tensor algebra element to its truncation up to level $n$. To differentiate with $\pi_n$, we also introduce another projection map $\Pi_n: T((E)) \rightarrow E^{\otimes n}$, which maps any $\mathbf{a} = (a_0, a_1, \cdots, a_n, \cdots)$ to its $n^{th}$ term $a_n$.

For concreteness, we state the decay rate of the signature for paths of finite $1$-variation. However, there is a similar statement of the factorial decay for the case of paths of finite $p$-variation \cite{lyons2007differential}. 
\begin{lemma}[Factorial Decay of the Signature]\label{Lemma_factorial_decay}
Let $X \in \mathcal{C}_{1}(J, E)$. Then there exists a constant $C>0$, such that for all $m \geq 0$,
\begin{eqnarray*}
\vert \Pi_{m}(S(X)) \vert \leq  C \frac{\vert X \vert_{1-var}^{m}}{m!}.
\end{eqnarray*}
\end{lemma}

\begin{lemma}\label{Lemma_continuity_sig}
Let $\mathcal{K}$ denote a compact set of $\Omega_0([0, T], E)$. Then the range  $S(\Omega_0([0, T], E))$ is a compact set on $T^{\mathrm{p}}(E)$ endowed with $l^{\mathrm{p}}$ topology.
\end{lemma}
\begin{proof}
The proof boils down to showing the continuity of the signature map $S$ from $\Omega_0([0, T], E)$ with $1-$variation norm to $\mathbf{T}^{\mathrm{p}}(E)$ with $l^{\mathrm{p}}$ topology. Let $X, Y \in \Omega_0([0, T], E)$, which are controlled by the control function $\omega$, e.g., $\omega(s,t):=\max(|X_{s,t}|_{1-var}, |Y_{s, t}|_{1-var})$ for all $0<s<t<T$. Let $|X_{[s, t]}-Y_{[s,t]}|_{1-var} \leq \epsilon \omega(s,t)$ for some $\epsilon \in \mathbb{R}^{+}$. Then by the continuity of the signature map in Theorem 3.10, \cite{lyons2007differential} and the admissible norm $l^1$, it holds that for an integer $n \geq 1$,
\begin{eqnarray*}
| \Pi_m(S(X)) - \Pi_m(S(Y)) |_{1} \leq \epsilon \frac{\omega(0,T)^{m}}{\beta n!},
\end{eqnarray*}
where $\beta = 2\left(1+\sum_{r=3}^{\infty}\left(\frac{2}{r-2}\right)^{2}\right)$. The direct calculation leads to that
\begin{eqnarray*}
||S(X) - S(Y)||_{\mathrm{p}} \leq ||S(X) - S(Y)||_{\mathrm{1}} \leq \sum_{m = 1}^{\infty} \left(\epsilon \frac{\omega(0,T)^{m}}{\beta m!}\right) =  \frac{\epsilon}{\beta} \sum_{m =1}^{\infty}\left(\frac{\omega(0,T)^m}{m!}\right) < +\infty.
\end{eqnarray*}
\end{proof}
\subsection{Expected Signature of stochastic processes}
\begin{definition}\label{Def_ROC_ES}
Let $X$ denote a stochastic process, whose signature is well defined almost surely. Assume that $\mathbb{E}[S(X)]$ is well defined and finite. We say that $\mathbb{E}[S(X)]$ has infinite radius of convergence, if and only if for every $\lambda \geq 0$,
\begin{eqnarray*}
\sum_{n \geq 0} \lambda^{n} |\Pi_{n}(\mathbb{E}[S(X)])| < \infty.
\end{eqnarray*}
\end{definition}

\subsection{The Signature Wasserstein-1 metric (Sig-$W_{1}$)}\label{subsec: sigw1}
In the following, we provide the proof of Lemma \ref{Lemma_lp}.
\begin{proof}
Let $\left(e_{I} = e_{i_{1}} \otimes \cdots e_{i_{n}}\right)_{I}$ be the canonical basis of $T((E))$. For any $a \in T((E))$, we write $a = (a_{I})$, i.e. $a = \sum_{I} a_I e_I$. Then $\left(e_{I}^{*} = e^{*}_{i_{1}} \otimes \cdots e^{*}_{i_{n}}\right)_{I = (i_1, \cdots, i_n) }$ is the basis of $T((E))^{*}$ and we can write $L = \sum_{I} l_I e^{*}_{I}$.

To prove Eq. \eqref{Eqn_T(E)*_q_norm}, we solve the constraint optimization of maximising $\bm{L}a$ with the constraint $||a||_p = 1$ by the Lagrange multiplier method. W.l.o.g, we only prove the case for $L\neq 0$ as it is trivial for $\bm{L}=0$, (when $L = 0$, $La \equiv 0$, $||\bm{L}||_q = 0$, and Eq. \eqref{Eqn_T(E)*_q_norm} holds). More specifically, we solve the unconstrained optimisation 
\begin{eqnarray*}
\mathcal{L}(a, \lambda):=\sup_{a} |\bm{L}a|+\lambda \left(\sum_{I}|a_I|^{\mathrm{p}} - 1\right),
\end{eqnarray*}
where $\bm{L} \neq 0$.\\
The optimal $(a^{*}, \lambda^{*})$ is a solution to the below equations:
\begin{eqnarray*}
&&\frac{\partial{\mathcal{L}}}{\partial a_I} = (\text{sign}(a_I l_I)l_I + (\lambda (\mathrm{p}|a_I|^{\mathrm{p}-1} \text{sign}(a_I))) ) = 0, \forall I\\
&& \frac{\partial{\mathcal{L}}}{\partial \lambda} = \sum_I |a_I|^{\mathrm{p}} - 1 = 0.
\end{eqnarray*}
Then we obtain that $a^{*} = (a_I^{*})_{I}$ with $a^{*}_I =\text{sign}(l_I) \frac{|l_I|^{\frac{1}{\mathrm{p}-1}}}{(\sum_I |l_I|^{\mathrm{p}/(\mathrm{p}-1)})^{1/\mathrm{p}}} = \text{sign}(l_I) \frac{|l_I|^{\frac{1}{\mathrm{p}-1}}}{(\sum_I |l_I|^{\mathrm{q}})^{1/\mathrm{p}}}$. Then it follows that
\begin{eqnarray*}
&&l_I a^{*}_I = |l_I|\cdot \frac{|l_I|^{\frac{1}{\mathrm{p}-1}}}{(\sum_I |l_I|^{\mathrm{q}})^{1/\mathrm{p}}} = \frac{|l_I|^{\mathrm{q}}}{(\sum_I |l_I|^{\mathrm{q}})^{1/\mathrm{p}}} \geq 0;\\
&&|La^{*}| = La^* = \sum_I l_I a^{*}_I =  \frac{\sum_I |l_I|^{\mathrm{q}}}{ \left(\sum_I |l_I|^{\mathrm{q}}\right)^{1/\mathrm{p}}} =  (\sum_I |l_I|^{\mathrm{q}})^{1-1/\mathrm{p}} = (\sum_I |l_I|^{\mathrm{q}})^{1/\mathrm{q}} = ||L||_\mathrm{q}.
\end{eqnarray*}
By H\^older's inequality,
\begin{eqnarray*}
\sup_{||a||_\mathrm{p} = 1} |\bm{L}a| \leq \sup_{||a||_\mathrm{p} = 1} ||a||_\mathrm{\mathrm{p}} ||\bm{L}||_\mathrm{q} = ||L||_\mathrm{q},
\end{eqnarray*}
and the superum $||\bm{L}||_\mathrm{q}$ is obtained when $a = a^{*}$. We complete the proof of Eq. \eqref{Eqn_T(E)*_q_norm}.

The proof of Eq. \eqref{Eqn_T(E)_p_norm} is similar to the above. We only need to show the supremum taken over $||L|| = 1$ is the same as that $||L|| \leq 1$. Again we only prove for $a \neq 0$ as the $a = 0$ case is trivial. 
Similarly to the above, when $L^{*}(a) := (l_{I}^{*})$ with
\begin{eqnarray}
l^{*}_I = \text{sign}(a_I) \frac{|a_l|^{\frac{1}{\mathrm{p}-1}}}{(\sum_I |a_I|^{\mathrm{q}})^{1/\mathrm{p}}},
\end{eqnarray}
$L^{*}(a)$ attains the supremum $\sup_{||L||_\mathrm{q} = 1} L(a) = ||a||_\mathrm{p}$ and $||L^{*}||_\mathrm{q} = 1$. 
By H\^older's inequality,
\begin{eqnarray}
\sup_{||L||_\mathrm{q} \leq 1} |\bm{L}a| \leq \sup_{||\bm{L}||_\mathrm{q} \leq 1} ||a||_\mathrm{p} ||\bm{L}||_\mathrm{q} \leq ||a||_\mathrm{p}.\label{super_la}
\end{eqnarray}
As $\sup_{||\bm{L}||_\mathrm{q} \leq 1} La$ can not exceed $||a||_\mathrm{p}$ and $\bm{L}^{*}(a) = ||a||_\mathrm{p}$, it follows 
\begin{eqnarray*}
\sup_{||L||_\mathrm{q} \leq 1} |\bm{L}(a)| = ||a||_\mathrm{p}.
\end{eqnarray*}
\end{proof}
\section{CONDITIONAL SIGNATURE WASSERSTEIN GANS }\label{sec:ImplementationCSigWGAN}
In this section, we provide the algorithmic details of the Conditional Signature Wasserstein GANs for practical applications. 
\subsection{Path transformations}\label{subsec:path transformation}
The core idea of SigCWGAN is to lift the time series to the signature feature as a principled and more effective feature extraction. In practice, the signature feature may often be accompanied with several of the following path transformations:
\begin{itemize}
    \item Time jointed transformation ( Definition 4.3, \cite{levin2013learning} );
    \item Cumulative sum transformation: it is defined to map every $(X_{t})_{t = 1}^{T}$ to $CS_{t}:= \sum_{i = 1}^{t} X_{i}, \forall t \in \{1, \cdots, T\}$ and $CS_{0} = \mathbf{0}$ ( Equation (2.20) in \cite{chevyrev2016primer} ).  
    \item Lead-Lag transformation ( Equation (2.8) in \cite{chevyrev2016primer} ).
    \item Lag added transformation: The $m$-lag added transformation of $(X_{t})_{t=1}^{T}$ is defined as follows: $\text{Lag}_{m}(X) = (Y_{t})_{t=1}^{T-m}$, such that
\begin{eqnarray*}
Y_{t} = (X_{t}, \cdots, X_{t+m}).
\end{eqnarray*}
\end{itemize}

Although in our analysis on the Sig-$W_{1}$ metric, we use the time augmented path to embed the discrete time series $X$ to a continuous path for the ease of the discussion. However, to use Sig-$W_1$ metric to differentiate two measures on the path space, the only requirement for the way of embeddings a discrete time series to a continuous path is that this embedding needs to ensure the bijection between the time series and its signature. Therefore, in practice we can choose other embedding to achieve that; for example, by applying the lead-lag transformation to time series, one can ensure the one-to-one correspondence between the time series and the signature.

\subsection{AR-FNN Architecture}
\label{sec:arfnn_architecture}
We give a detailed description of the AR-FNN architecture below. For this purpose let us begin by defining the employed transformations, namely the parametric rectifier linear unit and the residual layer. 

\begin{definition}[Parametric rectifier linear unit]
The parametrised function $\phi_\alpha \in C(\mathbb{R}, \mathbb{R}), \alpha \geq 0$ defined as 
\begin{equation*}
	\phi_\alpha(x) = \max(0, x) + \alpha \min(0, x)
\end{equation*}
is called \emph{parametric rectifier linear unit} (\emph{PReLU}). 
\end{definition}
	
\begin{definition}[Residual layer]
Let $F: \mathbb{R}^n \to \mathbb{R}^n$ be an affine transformation and $\phi_\alpha, \alpha \geq 0$ a PReLU. The function $R: \mathbb{R}^n \to \mathbb{R}^n$ defined as 
\begin{align*}
R(x) = x+\phi_\alpha\circ F(x)
\end{align*}
where $\phi_\alpha$ is applied component-wise, is called \emph{residual layer}. 
\end{definition}

The AR-FNN is defined as a composition of PReLUs, residual layers and affine transformations. Its inputs are the past $p$-lags of the $d$-dimensional process we want to generate as well as the $d$-dimensional noise vector. A formal definition is given below. 

\begin{definition}[AR-FNN]
\label{def:ar-fnn}
	Let $d, \bar{p} \in \mathbb{N}$, $A_1: \mathbb{R}^{d(\bar{p}+1)} \to \mathbb{R}^{50}$, $A_4: \mathbb{R}^{50} \to \mathbb{R}^{d}$ be affine transformations, $\phi_\alpha, \alpha \geq 0$ a PReLU and $R_2, R_3: \mathbb{R}^{50}\to\mathbb{R}^{50}$ two residual layers. Then the function $\textrm{ArFNN}: \mathbb{R}^{d\bar{p}} \times \mathbb{R}^d \to \mathbb{R}^{d}$ defined as 
	\begin{equation*}
		\textrm{ArFNN}(x, z) = A_4 \circ R_3 \circ R_2 \circ \phi_\alpha \circ A_1(xz)
	\end{equation*}
	where $xz$ denotes the concatenated vectors $x $ and $ z$, is called \emph{autoregressive feedforward neural network} (\emph{AR-FNN}). 
\end{definition} 
The pseudocode of generating the next $\bar{q}$-step forecast using $G^{\theta}$ is given in Algorithm \ref{alg_generator}.

\begin{algorithm}[H]
    \caption{Pseudocode of Generating the next $q$-step forecast using $G^{\theta}$}\label{alg_generator}
        \hspace*{\algorithmicindent} \textbf{Input:  $x_{t-\bar{p}+1:t},  G_{\theta}$ } \\
    \hspace*{\algorithmicindent} \textbf{Output: $\hat{x}_{t+1:t+\bar{q}}$} 
   \begin{algorithmic}[1]
   \State $\hat{x}_{\text{future}} \leftarrow$ a matrix of zeros of dimension $d \times \bar{q}$.
   \State $\hat{x} \leftarrow $the concatenation of $x_{t-\bar{p}+1:t}$ and $\hat{x}_{\text{future}}$.
    \For{$i = 1: \bar{q}$ }
    \State We sample $Z_{i}$ from the iid standard normal distribution.
    \State $\hat{x}_{t+i} = G(\hat{x}_{t+i-\bar{p}: t+i-1}, Z_{i})$.
    \EndFor
    \Return  $\hat{x}_{t+1:t+\bar{q}}$.
    \end{algorithmic}
\end{algorithm}

\section{NUMERICAL IMPLEMENTATIONS} \label{sec numerical implementations}
We use the following public codes for implementing the below three baselines:
\begin{itemize}
    \item RCGAN: \url{https://github.com/ratschlab/RGAN}
    \item Time-GAN: \url{https://github.com/jsyoon0823/TimeGAN}
    \item GMMN: \url{https://github.com/yujiali/gmmn}
\end{itemize}
Additionally, we implement a Conditional Wasserstein GAN (CWGAN) in the VAR(1) example: we perform the min-max optimisation~\eqref{eq:minmax} where the discriminator is parametrised by the same neural network architecture as the generator, i.e. a 3-layer FNN. We ensure that the discriminator is 1-Lipschitz by adding a gradient penalty term introduced by \cite{gulrajani2017improved}. 


For a fair comparison, we use the same neural network generator architecture, namely the 3-layer AR-FNN described in \autoref{sec:arfnn_architecture}, for the SigCWGAN, TimeGAN, RCGAN and GMMN. The TimeGAN and RCGAN discriminators take as inputs the conditioning time series $X_{1:\bar{p}}$ concatenated with the synthetic time series $X_{\bar{p}+1:\bar{p}+\bar{q}}$. Both discriminators use the AR-FNN as the underlying architecture. However, the first affine layer is adjusted such that the AR-FNN is defined as a function of the concatenated time series, i.e. $\bar{p}+\bar{q}$ lags and not $\bar{p}$-lags as for the generator. Similarly, the MMD is computed by concatenating the conditioning and synthetic time series. In order to obtain the bandwidth parameter for computing the MMD of the GMMN we benchmarked the median heuristic against using a mixture of bandwidths spanning multiple ranges as proposed in \cite{li2015generative} and found latter to work best. In our experiments we used three kernels with bandwidths $0.1, 1, 5$. 

All algorithms were optimised for a total of 1,000 generator weight updates. The neural network weights were optimised by using the Adam optimiser \cite{kingma2014adam} and learning rates for the generators were set to 0.001. For the RCGAN and TimeGAN we applied two time-scale updates (TTUR) \cite{heusel2017gans} and set the learning rate to 0.003. Furthermore, we updated the discriminator’s weights two times per generator weight update in order to improve convergence of the GAN. 

In our numerical experiments, to compute the signature for the SigCWGAN method, we choose to apply the following path transformations on the time series before computing the signatures: (1) we combine the path $x_{\text{past}}$ with its cumulative sum transformed path, denoted by $y_{\text{past}}$, which is a $2d$-dimensional path; (2) we apply 1-lag added transformation on $y_{past}$; (3) it follows with the Lead-Lag transformation. The signature of such transformed path can well capture the marginal distributions, auto-correlations and other temporal characteristics of the time-series data.

In the following, we describe the calculation of the test metrics precisely. Let $(X_{t})_{t=1}^{T}$ denote a $d$-dimensional time series sampled from the real target distribution. We first extract the input-out pairs $(X_{t-\bar{p}+1:t}, X_{t+1:t+\bar{q}})_{t \in \mathcal{T}}$, where $\mathcal{T}$ is the set of time indexes. Given the generator $G$, for each input sample $(X_{t-\bar{p}+1: t})$, we generate one sample of the $\bar{q}$-step forecast $\hat{X}^{(t)}_{t+1, t+\bar{q}}$ (if $G$ is not a conditional generator, we generate a sample of $\bar{q}$-step forecast $\hat{X}^{(t)}_{t+1, t+\bar{q}}$ without any conditioning variable.).  The synthetic data generated by $G$ is given by $(\hat{X}^{(t)}_{t+1, t+\bar{q}})_t$, which we use to compute the test metrics. 
\paragraph{Metric on marginal distribution} Following \cite{wiese2019DH}, we use $(X_{t+1:t+\bar{q}})_{t \in \mathcal{T}}$ and $(\hat{X}^{(t)}_{t+1:t+\bar{q}})_{t\in \mathcal{T}}$ as the samples of the marginal distribution of the real data and synthetic data per each time step. For each feature dimension $i \in \{1, \cdots, d\}$, we compute two empirical density functions based on the histograms of the real data and synthetic data resp. denoted by $\hat{df}_{r}^{i}$ and $\hat{df}_{G}^{i}$. Then the metric on marginal distribution of the true and synthetic data is given by 
\begin{eqnarray*}
\frac{1}{d}\sum_{i = 1}^{d} \vert \hat{df}_{r}^{i} - \hat{df}_{G}^{i} \vert_{1}. 
\end{eqnarray*}
\paragraph{Absolute difference of lag-1 auto-correlation}
The auto-covariance of $i^{th}$ feature of the real data with lag value $k$ is computed by 
\begin{eqnarray}\label{eqn: auto_cov_real}
\rho_{r}^{i}(k):= \frac{1}{T-k}\sum_{t=1}^{T-k}(X_{t}^{i} - \bar{X}^{i})(X_{t+k}^{i} - \bar{X}^{i}),
\end{eqnarray}
where $\bar{X}^{i}$ is the average of $(X_{t}^{i})_{t=1}^{T}$. \\
For the synthetic data, we estimate the auto-covariance of $i^{th}$ feature with lag value $k$ is computed by
\begin{eqnarray}\label{eqn: auto_cov_sync}
\rho_{G}^{i}(k):= \frac{1}{|\mathcal{T}|} \sum_{t=1}^{|\mathcal{T}|} \hat{X}_{t+1}^{(t),i} \hat{X}_{t+k+1}^{(t), i}  - \left(\frac{1}{|\mathcal{T}|} \sum_{t = 1}^{|\mathcal{T}|}\hat{X}_{t+1}^{(t), i}\right)\left(\frac{1}{|\mathcal{T}|}\sum_{t=1}^{|\mathcal{T}|} \hat{X}_{t+k+1}^{(t),i}\right).
\end{eqnarray}
The estimator of the lag-1 auto-correlation of the real/synthetic data is given by $\frac{\rho_{r}^{i}(1)}{\rho_{r}^{i}(0)}$/ $\frac{\rho_{G}^{i}(1)}{\rho_{G}^{i}(0)}$. 
The ACF score is defined to be the absolute difference of lag-1 auto-correlation given as follows:
\begin{eqnarray*}
\frac{1}{d}\sum_{i = 1}^{d}\left \vert \frac{\rho_{r}^{i}(1)}{\rho_{r}^{i}(0)} -  \frac{\rho_{G}^{i}(1)}{\rho_{G}^{i}(0)}\right \vert.
\end{eqnarray*}
\paragraph{Metric on the correlation}
We estimate the covariance of the $i^{th}$ and $j^{th}$ feature of time series from the true data as follows:
\begin{eqnarray}\label{eqn:cov_real_est}
\text{cov}^{i, j}_{r}=\frac{1}{T} \sum_{t=1}^{T} X^{i}_{t} X^{j}_{t} - \left(\frac{1}{T} \sum_{t=1}^{T} X^{i}_{t}\right)\left(\frac{1}{T} \sum_{t=1}^{T} X^{j}_{t}\right).
\end{eqnarray}
Similarly, we estimate the covariance of the $i^{th}$ and $j^{th}$ feature of time series from the synthetic data by 
\begin{eqnarray}\label{eqn:cov_sync_est}
\text{cov}^{i, j}_{G}=\frac{1}{|\mathcal{T}|}\frac{1}{q} \sum_{t=1}^{|\mathcal{T}|}\sum_{s=1}^{q} \hat{X}^{(t), i}_{t+s} X^{(t),j}_{t+s} - \left(\frac{1}{|\mathcal{T}|} \sum_{t\in \mathcal{T}} \hat{X}^{(t), i}_{t+s}\right)\left(\frac{1}{|\mathcal{T}|} \sum_{t \in |\mathcal{T}|} X^{(t),j}_{t+s}\right).
\end{eqnarray}
Thus the estimator of the correlation of the $i^{th}$ and $j^{th}$ feature of time series from the real/synthetic data are given by $\tau_{r}^{i, j}:=\frac{\text{cov}^{i, j}_{r}}{\sqrt{\text{cov}^{i, i}_{r} \text{cov}^{j, j}_{r}}}$ and $\tau_{G}^{i, j}:=\frac{\text{cov}^{i, j}_{G}}{\sqrt{\text{cov}^{i, i}_{G} \text{cov}^{j, j}_{G}}}$. 
Then the metric on the correlation between the real data and synthetic data is given by $l_{1}$ norm of the difference of two correlation matrices $\left(\tau_{r}^{i, j}\right)_{i, j \in \{1, \cdots, d\}}$ and $\left(\tau_{G}^{i, j}\right)_{i, j \in \{1, \cdots, d\}}$.
\paragraph{TRTR/TSTR $R^{2}$} We split the input-output pairs $(X_{t-\bar{p}+1: t}, X_{t+1})$ from the real data into the train set and test set. We apply the linear signature model on real training data $(X_{t-\bar{p}+1: t}, X_{t+1})$, validate it and compute the corresponding $R^{2}$ on the on the real test data (TRTR $R^{2}$). Then we apply the same linear signature model on the synthetic data $(X_{t-\bar{p}+1: t}, \hat{X}_{t+1})$, where $\hat{X}_{t+1}$ is simulated by the generator conditioning on the $X_{t-\bar{p}+1: t}$. We evaluate the trained model on the real test data and corresponding $R^{2}$ is called (TSTR $R^{2}$).
\section{SUPPLEMENTARY NUMERICAL RESULTS} \label{sec numerical results}
 \renewcommand{\thetable}{D.\arabic{table}}
 \renewcommand{\thefigure}{D.\arabic{figure}}
 \setcounter{figure}{0}
\setcounter{table}{0}

\subsection{VAR(1) dataset}
We conduct the extensive experiments on VAR(1) with different hyper-parameter settings, i.e. $d \in \{1, 2, 3\}$, $\sigma, \phi \in \{0.2, 0.5, 0.8\}$. 

\paragraph{Test metrics of different models} We apply SigCWGAN, CWGAN and the other above-mentioned methods on VAR(1) differentset with various hyper-parameter settings. The summary of the test metrics of all models on $d$ dimensional VAR(1) data for $d = 1, 2, 3$ can be found in Table \ref{table_Var_dim=1}, \ref{table_Var_dim=2} and \ref{table_Var_dim=3} respectively.

\input{numerical_results/table_var_dim=1}

\input{numerical_results/table_var_dim=2}

\input{numerical_results/table_var_dim=3}
Additionally apart from Figure \ref{Fig_VAR_R2} the $R^2$ comparison, we provide the bar charts to compare the performance of different methods on the VAR data in terms of other test metrics in Figure \ref{fig:var_only_comparison_density} to $\sim$ \ref{fig:var_only_comparison_cross_correl}.
\paragraph{Training stability}
Figures~\ref{fig:VAR(1) dim2 score metrics} and~\ref{fig:VAR(1) dim3 score metrics} demonstrate the stability of the SigCWGAN optimisation in terms of training iterations in contrast to other baselines, in particular two baselines involving the min-max game optimization. 

\begin{figure}[b]
    \centering
    \includegraphics[width=0.9\textwidth]{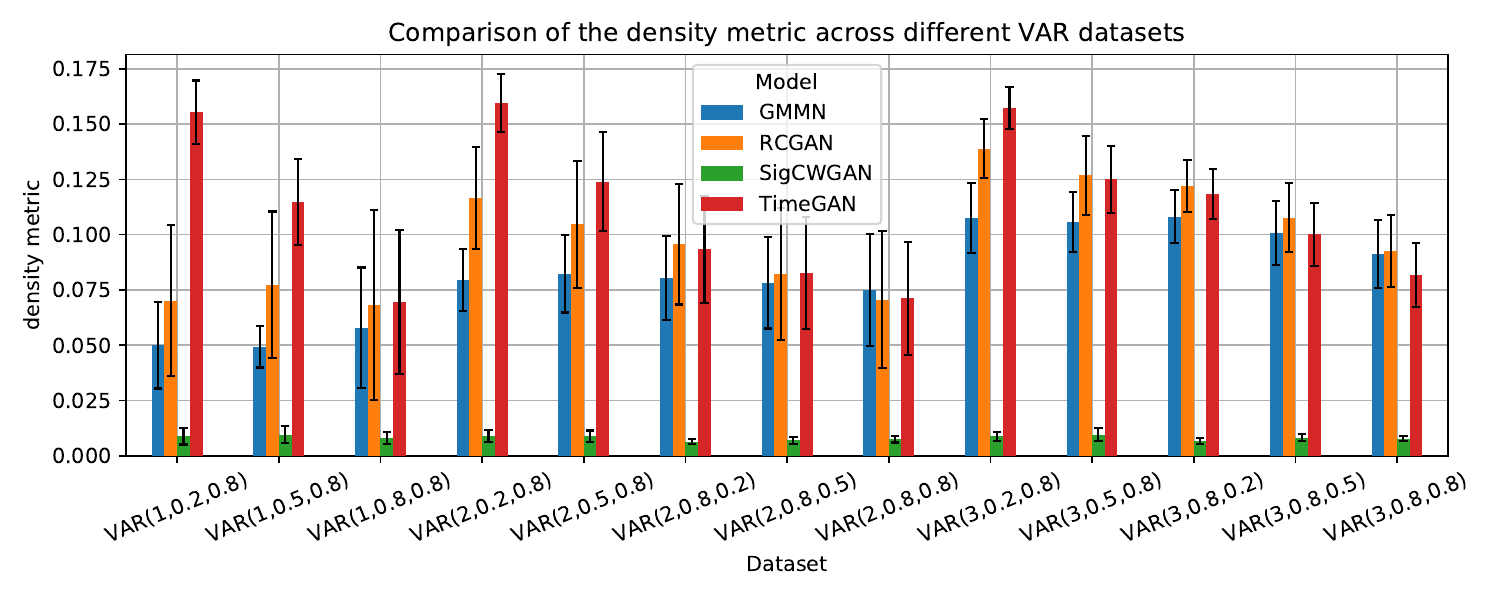}
    \caption{Comparison of the performance on the density metric across all algorithms and benchmarks.}
    \label{fig:var_only_comparison_density}
\end{figure}
\begin{figure}[b]
    \centering
    \includegraphics[width=0.9\textwidth]{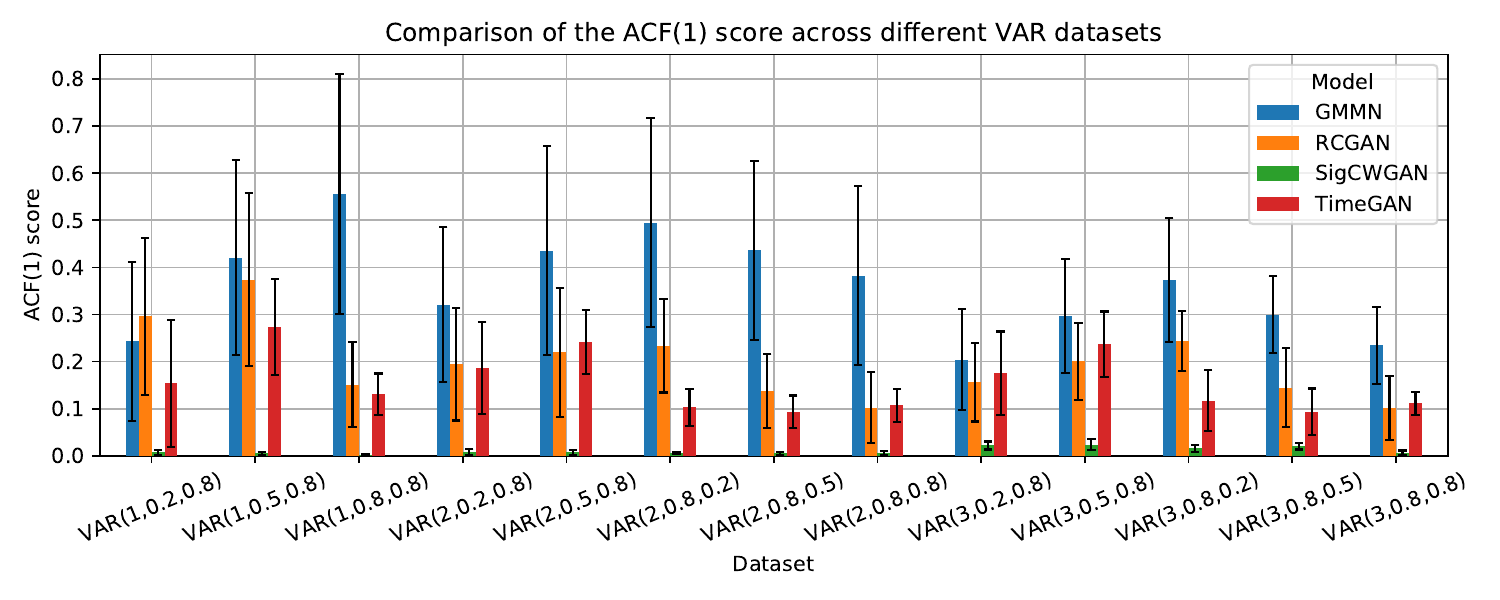}
    \caption{Comparison of the performance on the absolute difference of lag-1 autocorrelation across all algorithms and benchmarks.}
    \label{fig:var_only_comparison_acf}
\end{figure}
\begin{figure}[t]
    \centering
    \includegraphics[width=0.9\textwidth]{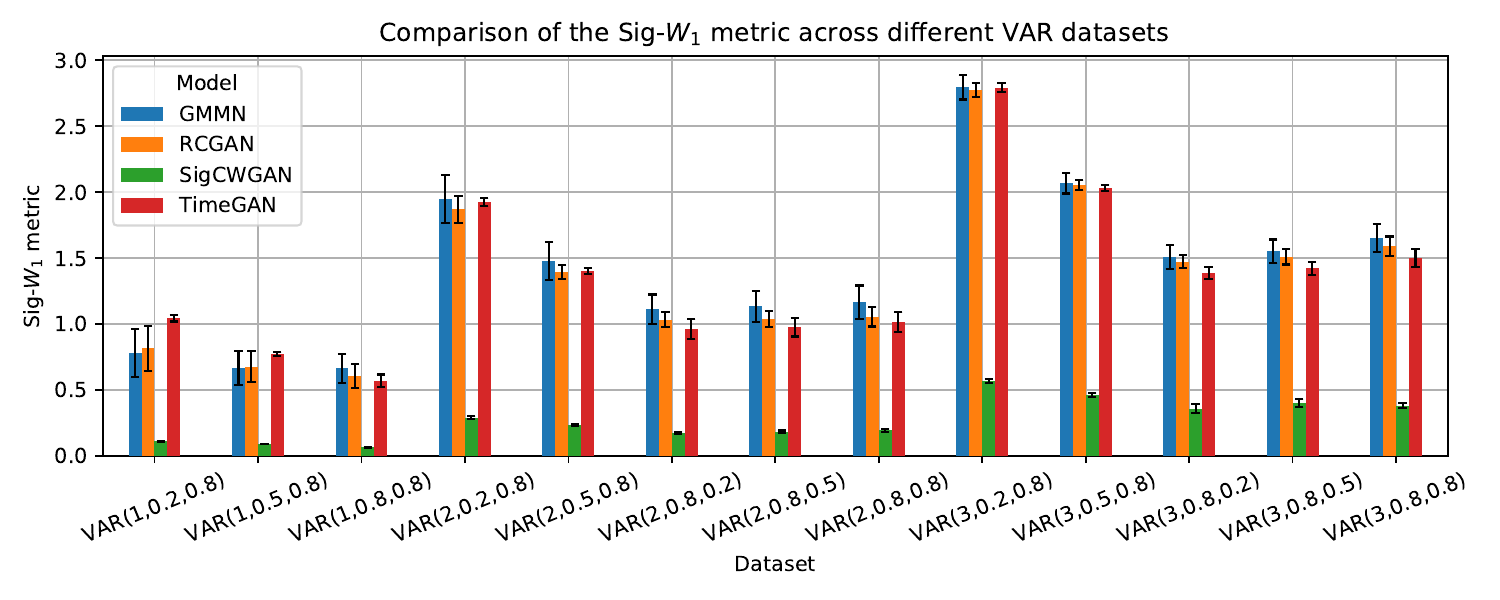}
    \caption{Comparison of the performance on the Sig-$W_1$ metric across all algorithms and benchmarks.}
    \label{fig:var_only_comparison_sig_w1}
\end{figure}
\begin{figure}[t]
    \centering
    \includegraphics[width=0.9\textwidth]{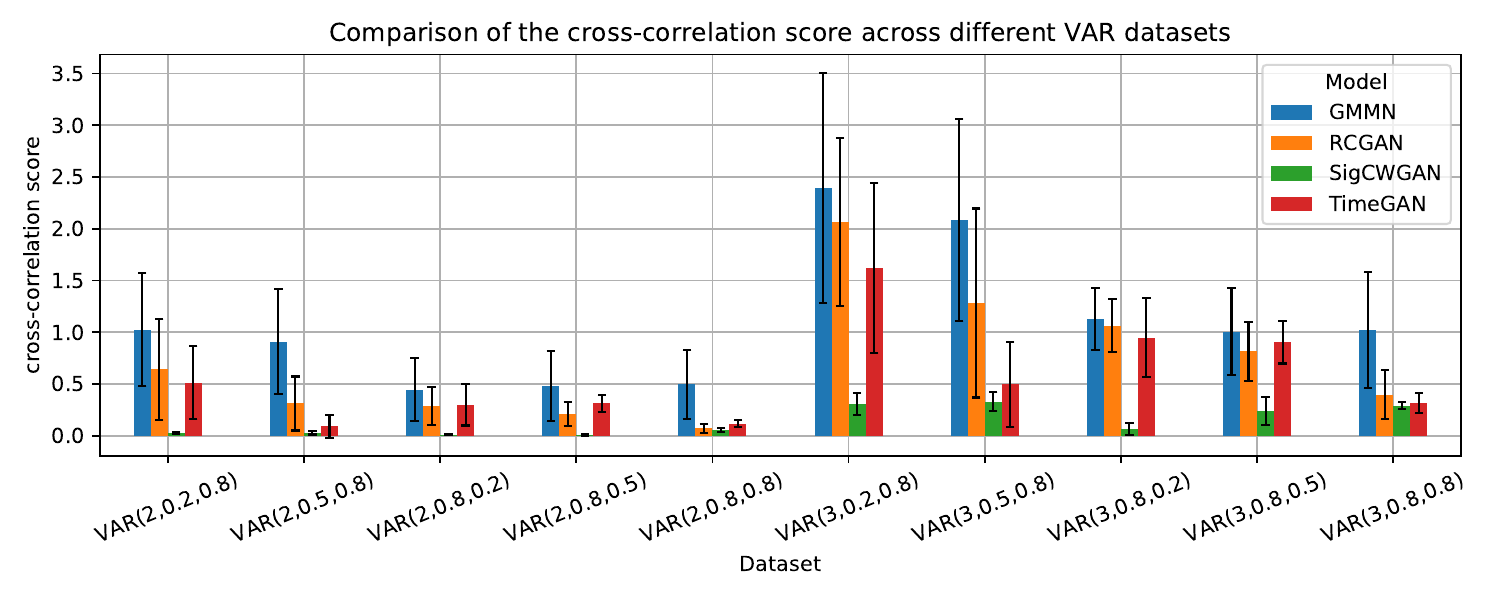}
    \caption{Comparison of the performance on the cross-correlation metric across all algorithms and benchmarks.}
    \label{fig:var_only_comparison_cross_correl}
\end{figure}

\clearpage

\begin{figure}
	\centering
	\begin{subfigure}{\linewidth}
		\centering
		\includegraphics[width=\textwidth]{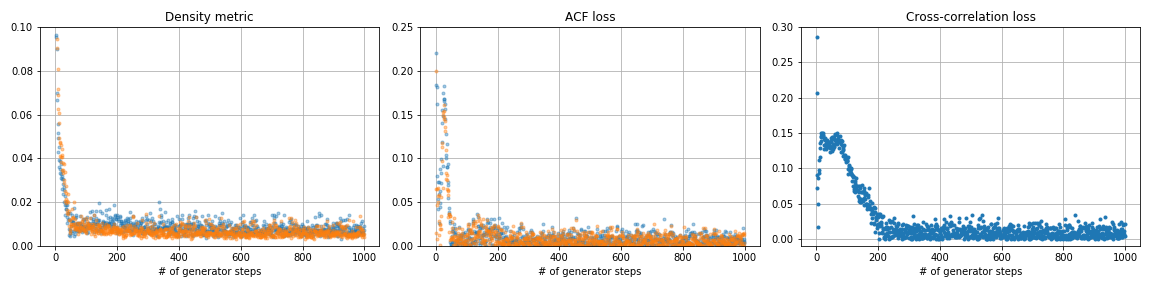}
		\caption{SigCWGAN}
	\end{subfigure}
	\begin{subfigure}{\linewidth}
		\centering
		\includegraphics[width=\textwidth]{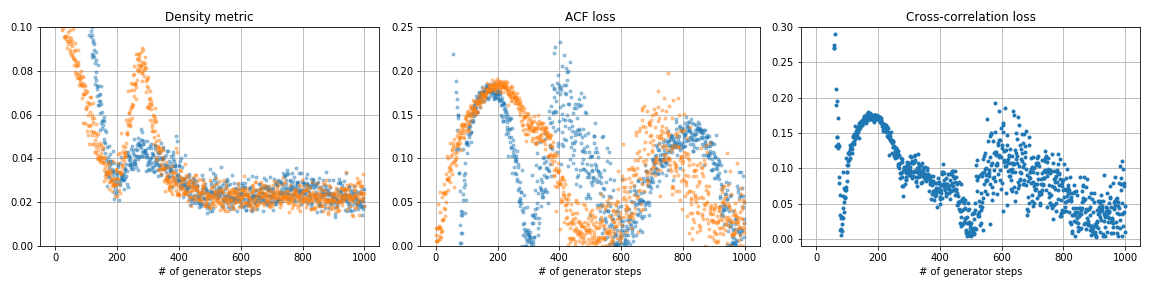}
		\caption{TimeGAN}
	\end{subfigure}
	\begin{subfigure}{\linewidth}
		\centering
		\includegraphics[width=\textwidth]{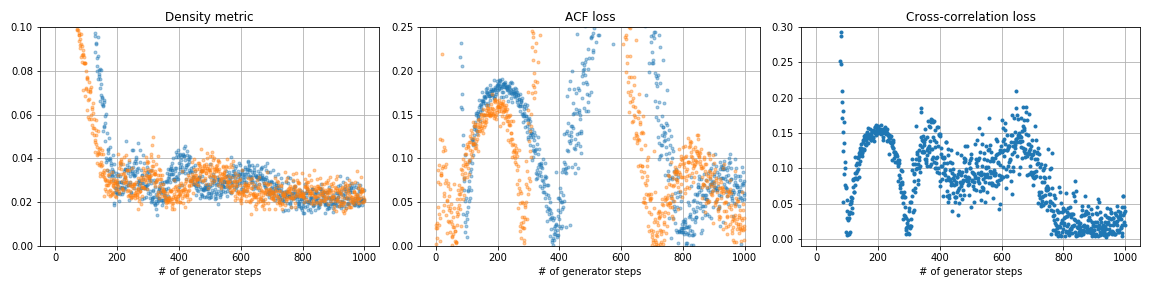}
		\caption{RCGAN}	
	\end{subfigure}
	\begin{subfigure}{\linewidth}
		\centering
		\includegraphics[width=\textwidth]{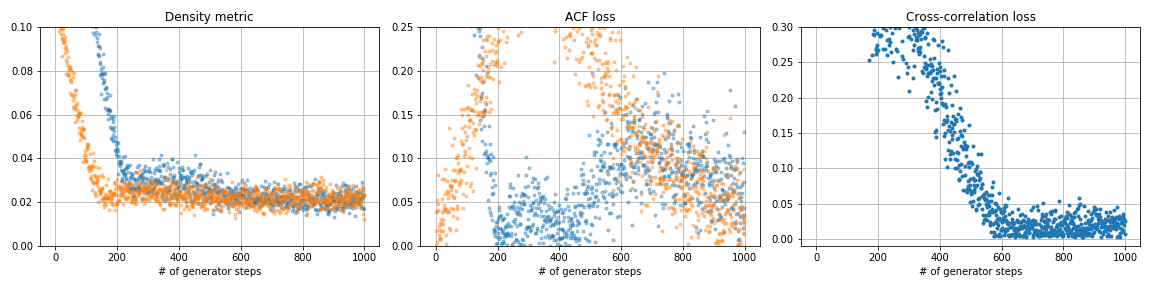}
		\caption{GMMN}
	\end{subfigure}
\caption{Exemplary development of the considered distances and score functions during training for the $2$-dimensional $\operatorname{VAR}(1)$ model with autocorrelation coefficient $\phi=0.8$ and co-variance parameter $\sigma=0.8$. The colours blue and orange indicate the relevant distance / score for each dimension.}
\label{fig:VAR(1) dim2 score metrics}
\end{figure}
\clearpage
\begin{figure}
	\centering
	\begin{subfigure}{\linewidth}
		\centering
		\includegraphics[width=\textwidth]{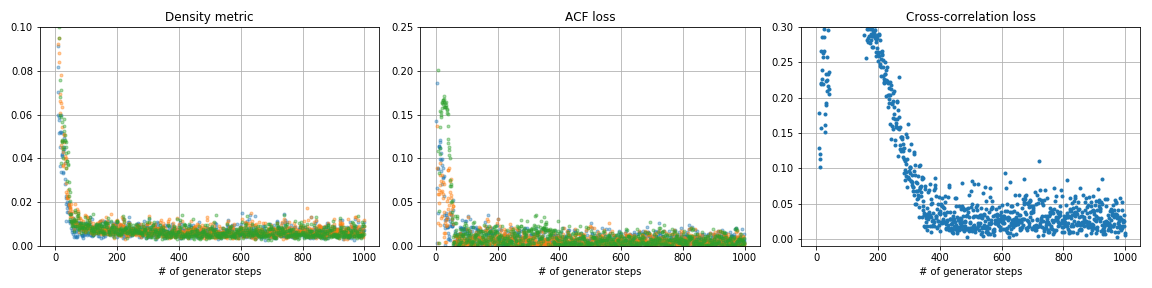}
		\caption{SigCWGAN}
	\end{subfigure}
	\begin{subfigure}{\linewidth}
		\centering
		\includegraphics[width=\textwidth]{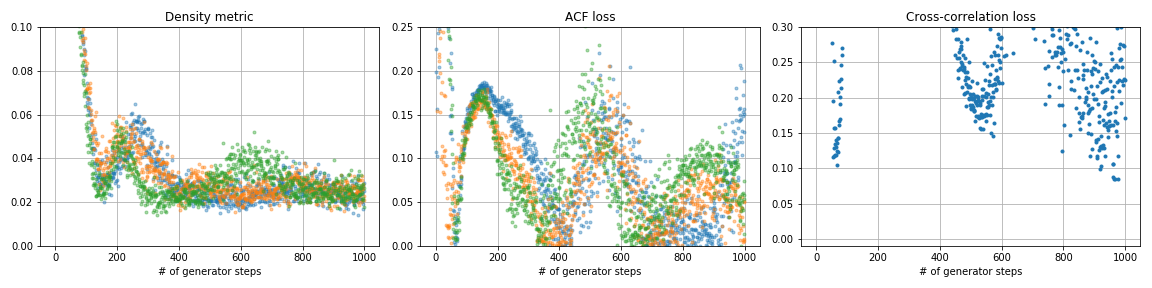}
		\caption{TimeGAN}
	\end{subfigure}
	\begin{subfigure}{\linewidth}
		\centering
		\includegraphics[width=\textwidth]{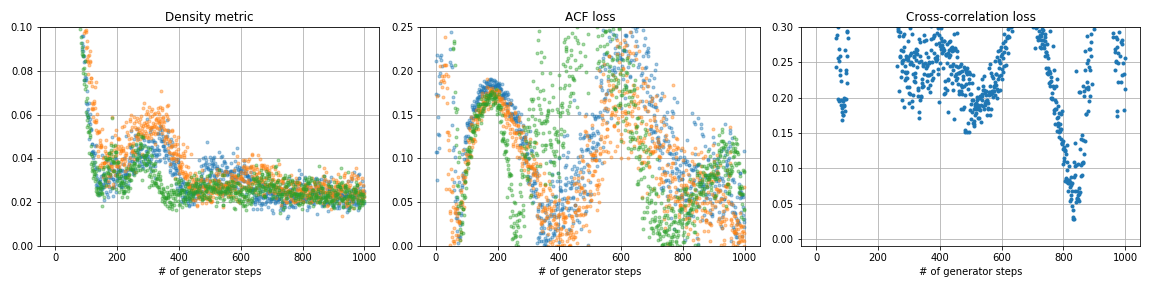}
		\caption{RCGAN}	
	\end{subfigure}
	\begin{subfigure}{\linewidth}
		\centering
		\includegraphics[width=\textwidth]{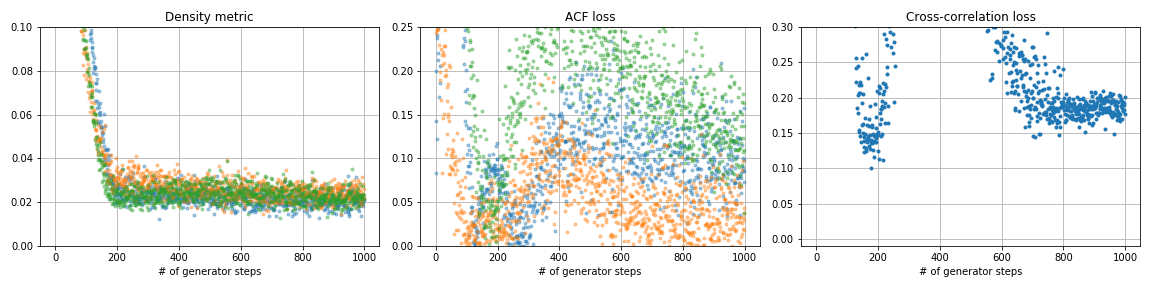}
		\caption{GMMN}
	\end{subfigure}
	\caption{Exemplary development of the considered distances and score functions during training for the $3$-dimensional $\operatorname{VAR}(1)$ model with autocorrelation coefficient $\phi=0.8$ and co-variance parameter $\sigma=0.8$. The colours blue, orange and green indicate the relevant distance / score for each dimension.}
	\label{fig:VAR(1) dim3 score metrics}
\end{figure}
\clearpage
\subsection{ARCH(p)}
We implement extensive experiments on ARCH(p) with different $p-$lag values, i.e. $p\in\{2,3,4\}$. We choose the optimal degree of signature 3. The numerical results are summarized in Table~\ref{table_arch}. The best results among all the models are highlighted in bold.

\input{numerical_results/table_arch}
\clearpage

\subsection{SPX and DJI dataset}
We provide the supplementary results on the SPX and DJI dataset. The summary of test metrics of different models is given by Table \ref{table_stocks}. The test metrics over the training process of each method on (1) SPX dataset and (2) SPX and DJI dataset can be found in Figure \ref{Fig:loss_SPX} and Figure \ref{fig:loss_spx_dji}. The fitting of different models in terms of the cross-correlation matrix of the log-return and log-realized volatility of SPX are presented in Figure \ref{Fig:Correlation_Comparison_SPX}.

\input{numerical_results/table_stocks}

\clearpage

\begin{figure}
	\centering
	\begin{subfigure}{\linewidth}
		\centering
		\includegraphics[width=\textwidth]{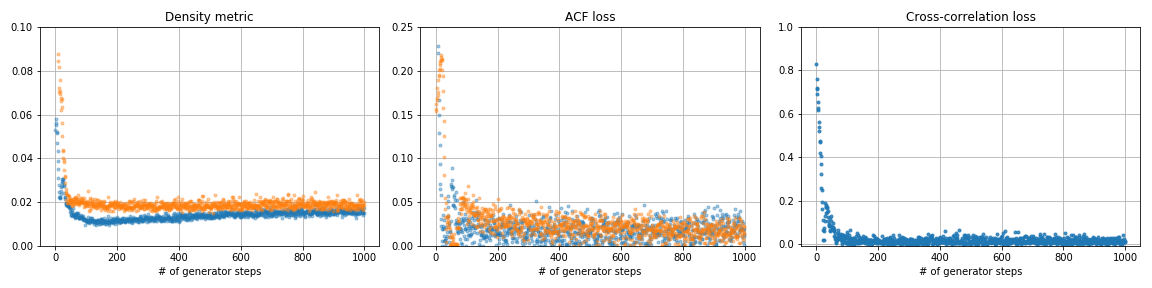}
		\caption{SigCWGAN}
	\end{subfigure}
	\begin{subfigure}{\linewidth}
		\centering
		\includegraphics[width=\textwidth]{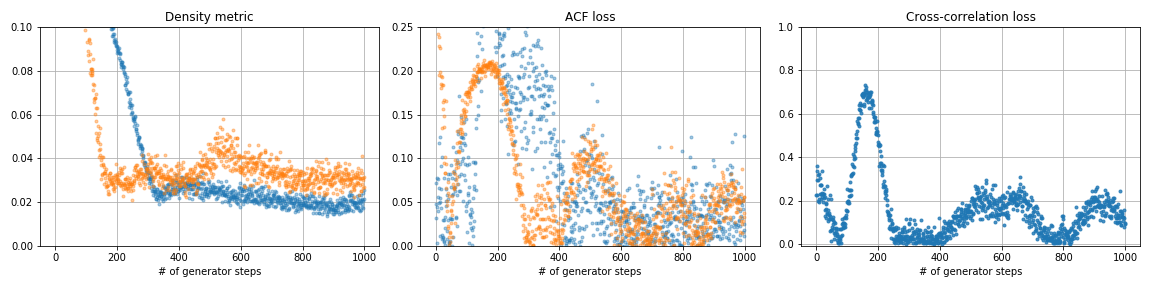}
		\caption{TimeGAN}
	\end{subfigure}
	\begin{subfigure}{\linewidth}
		\centering
		\includegraphics[width=\textwidth]{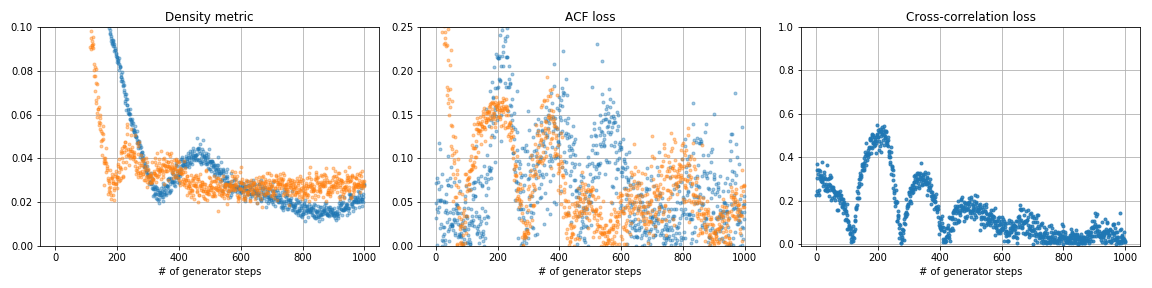}
		\caption{RCGAN}	
	\end{subfigure}
	\begin{subfigure}{\linewidth}
		\centering
		\includegraphics[width=\textwidth]{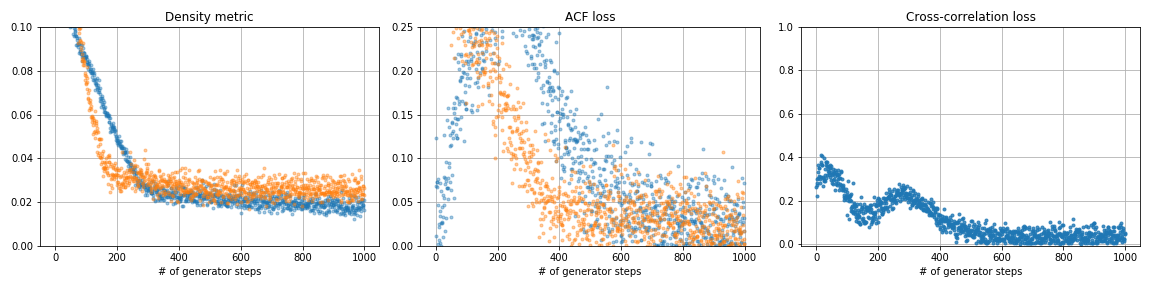}
		\caption{GMMN}
	\end{subfigure}
\caption{Exemplary development of the considered distances and score functions during training for SPX data.}\label{Fig:loss_SPX}
\end{figure}

\clearpage

\begin{figure}
	\centering
	\begin{subfigure}{\linewidth}
		\centering
		\includegraphics[width=\textwidth]{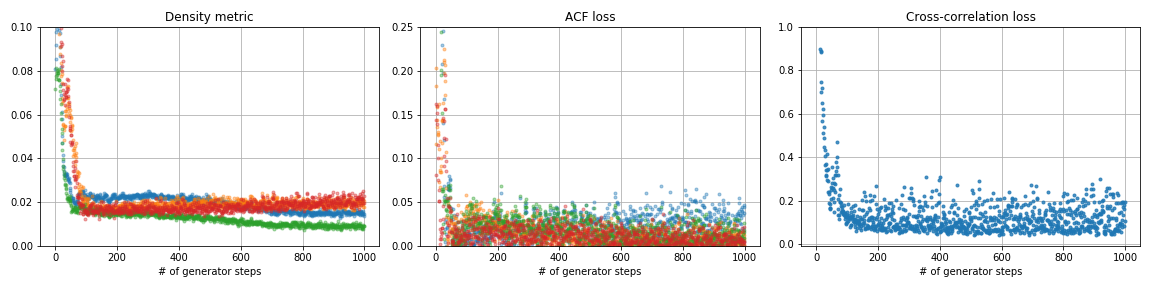}
		\caption{SigCWGAN}
	\end{subfigure}
	\begin{subfigure}{\linewidth}
		\centering
		\includegraphics[width=\textwidth]{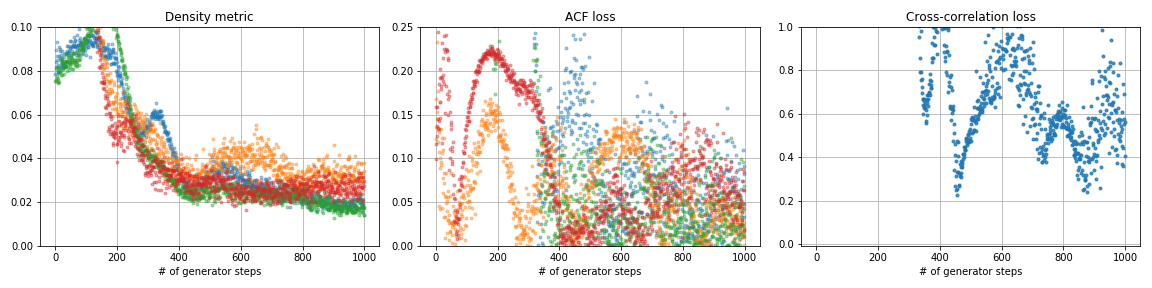}
		\caption{TimeGAN}
	\end{subfigure}
	\begin{subfigure}{\linewidth}
		\centering
		\includegraphics[width=\textwidth]{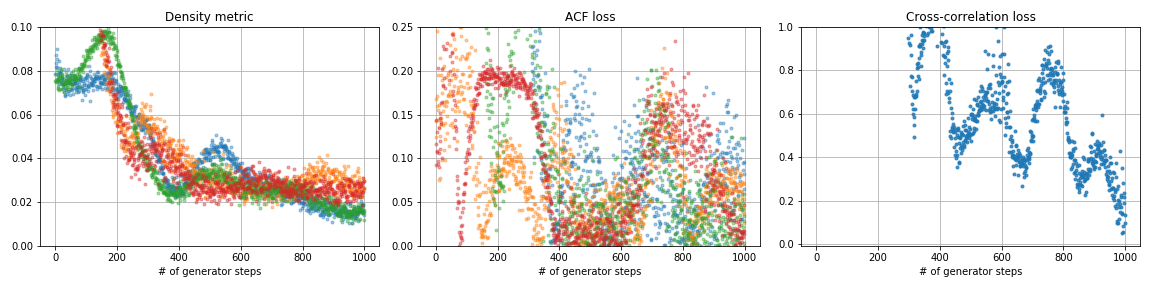}
		\caption{RCGAN}	
	\end{subfigure}
	\begin{subfigure}{\linewidth}
		\centering
		\includegraphics[width=\textwidth]{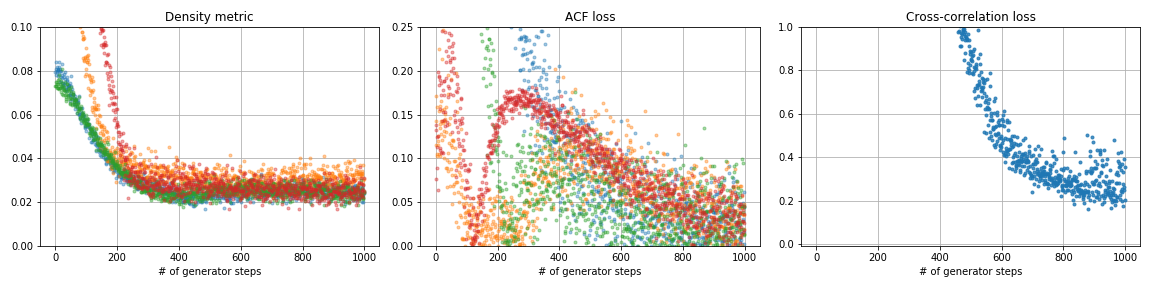}
		\caption{GMMN}
	\end{subfigure}
\caption{ExemplARCHry development of the considered distances and score functions during training for SPX and DJI data.}\label{fig:loss_spx_dji}
\end{figure}

\clearpage
\begin{figure}[!ht]
\centering
\includegraphics[width=\textwidth]{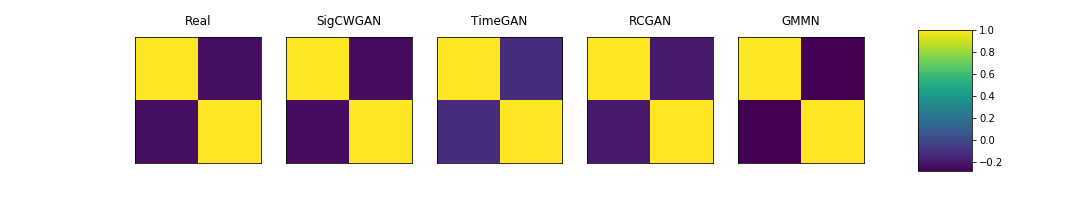}
\caption{Comparison of real and synthetic cross-correlation matrices for SPX data. On the far left the real cross-correlation matrix from SPX log-return and log-volatility data is shown. $x$/$y$-axis represents the feature dimension while the color of the $(i, j)^{th}$ block represents the correlation of $X_{t}^{(i)}$ and $X_{t}^{(j)}$. The colorbar on the far right indicates the range of values taken. Observe that the historical correlation between log-returns and log-volatility is negative, indicating the presence of leverage effects, i.e. when log-returns are negative, log-volatility is high.}\label{Fig:Correlation_Comparison_SPX}
\end{figure}

\begin{table}[!ht]
\centering
\resizebox{\columnwidth}{!}{%
\begin{tabularx}{1.4\textwidth}{||l|| Y || Y || Y || Y  || Y|| Y}
\toprule
\backslashbox{Model}{~Lag}&
\multicolumn{1}{c||}{1 }& \multicolumn{1}{c||}{2 } &   \multicolumn{1}{c||}{4}& 
 \multicolumn{1}{c||}{6} & \multicolumn{1}{c||}{8}    
\\ \bottomrule
SigCWGAN  & 2.996, 7.948 & \textbf{3.510}, 5.928
 & \textbf{3.801}, \textbf{7.439} & \textbf{3.944}, \textbf{9.103} & \textbf{5.534}, \textbf{10.742} \\\midrule
 TimeGAN  & 5.955, 8.586 & 8.470, 9.925 &  10.838, 14.816  & 13.163, 20.139 & 16.922, 22.870\\\midrule
RCGAN   & \textbf{2.788}, \textbf{7.190} & 3.701, \textbf{5.425} & 5.090, 9.407& 6.033, 12.424 & 9.380, 16.599 \\\midrule
GMMN & 9.049, 7.384 & 11.275, 9.150 &19.302, 14.466 & 25.832, 21.690 & 28.269, 24.778 \\\midrule
GARCH   &104.776, 100.749 & 99.359, 109.313  & 103.137, 109.53 & 102.939, 107.669 & 102.527, 104.779 \\\midrule
\end{tabularx}}
\caption{$R^{2}$ metric (\%) of the stock datasets for different lag values. In each cell, the left/right number are the result for the SPX data/ the SPX and DJI data respectively. }\label{table_relativeR2stock}
\end{table}

\begin{table}[!ht]
\centering
\resizebox{\columnwidth}{!}{%
\begin{tabularx}{1.4\textwidth}{||l|| Y || Y || Y || Y  || Y|| Y}
\toprule
\backslashbox{Model}{~Lag}&
\multicolumn{1}{c||}{1 }& \multicolumn{1}{c||}{2 } &   \multicolumn{1}{c||}{3}& 
 \multicolumn{1}{c||}{4} & \multicolumn{1}{c||}{5}    
\\ \bottomrule
SigCWGAN  & 0.01342, \textbf{0.01192} & 0.02234, 0.02576
 & 0.05744, \textbf{0.03592} & 0.07389, 0.08646 & 0.12571, 0.1057\\\midrule
 TimeGAN  & 0.05792, 0.03035 & 0.06070, 0.03182&  0.06823, 0.09887   & 0.05735, 0.10609 & \textbf{0.08387}, 0.15083\\\midrule
RCGAN   & 0.03362, 0.04075 &  0.03134, 0.03977 & 0.06692, 0.08859 & \textbf{0.05641}, 0.07687 & 0.09089, 0.11083 \\\midrule
GMMN & \textbf{0.01283}, 0.02676 & \textbf{0.0177}, \textbf{0.0253} & \textbf{0.04293}, 0.06476 & 0.06740, \textbf{0.06952} & 0.09589, \textbf{0.09906} \\\midrule
GARCH   & 0.4721, 0.4559& 0.4661, 0.4741  & 0.6198,0.6282& 0.7292,0.7312 & 0.8250,0.8212 \\\midrule
\end{tabularx}}
\caption{Autocorrelation metric for the stock datasets for different lag values. In each cell, the left/right number are the result for the SPX data/ the SPX and DJI data respectively.}\label{table_ACFstock}
\end{table}

\subsection{Bitcoin dataset}
We provide the additional numerical results on the Bitcoin dataset as follows.
\begin{table}[!ht]
\centering
\resizebox{\columnwidth}{!}{%
\begin{tabularx}{1.4\textwidth}{||l|| Y || Y || Y || Y  || Y|| Y}
\toprule
\backslashbox{Model}{~Lag}&
\multicolumn{1}{c||}{1 }& \multicolumn{1}{c||}{2 } &   \multicolumn{1}{c||}{3}& 
 \multicolumn{1}{c||}{4} & \multicolumn{1}{c||}{5}    
\\ \bottomrule
SigCWGAN  & 0.0911 & 0.2814 & 0.3 & 0.3021 & 0.3028  \\\midrule
TimeGAN   & 0.1203 & 0.2170 & 0.2312 & 0.2329 & 0.2568  \\\midrule
RCGAN   & 0.0533 & 0.1486 & 0.15 & 0.1654 & 0.1751  \\\midrule
GMMN & 0.2093 & 0.3436 & 0.3448 & 0.3478 & 0.4333 \\\midrule
GARCH & 0.0872 & 0.1314 & 0.1378 & 0.1471 & 0.1501 \\\midrule
\end{tabularx}}
\caption{Autocorrelation metric for the BTC dataset for different lag values}\label{table:btc autocorrelation lag}
\end{table}


\begin{table}[!ht]
\centering
\resizebox{\columnwidth}{!}{%
\begin{tabularx}{1.4\textwidth}{||l|| Y || Y || Y || Y  || Y}
\toprule
\backslashbox{Model}{~Lag}&
\multicolumn{1}{c||}{1 }& \multicolumn{1}{c||}{2 } &   \multicolumn{1}{c||}{4}& 
 \multicolumn{1}{c||}{6}
\\ \bottomrule
SigCWGAN  & 0.3320 & 0.2681 & 0.3843 & 0.3128  \\\midrule
TimeGAN   & 0.7582 & 0.5962 & 0.6394 & 0.6018  \\\midrule
RCGAN   & 0.3165 & 0.2191 & 0.2898 & 0.2261  \\\midrule
GMMN & 0.3904 & 0.3436 & 0.2583 & 0.3426  \\\midrule
GARCH & 123.77 & 87.35 & 96.47 & 117.34 \\\midrule
\end{tabularx}}
\caption{$R^{2}$ metric (\%) of the BTC dataset for different lag values. }\label{table:btc r2 lag}
\end{table}

%% file: numerical_results/table_var_dim=1.tex
\newcolumntype{Y}{>{\centering\arraybackslash}X}
\begin{table}[!ht]
\caption{Numerical results of $\operatorname{VAR}(1)$ for $d = 1$}\label{table_Var_dim=1}
\centering 
\begin{tabularx}{0.6\textwidth}{l*3{|Y}}
\toprule
\multicolumn{1}{l|}{}     & \multicolumn{3}{c}{Temporal Correlations}                                                       \\ \midrule
Settings                  & \multicolumn{1}{c}{$\phi=0.2$} & \multicolumn{1}{c}{$\phi=0.5$} & \multicolumn{1}{c}{$\phi=0.8$} \\ \midrule
\multicolumn{4}{c}{Metric on marginal distribution}                                                                                                   \\ \midrule
SigCWGAN  & 0.0124 & 0.0100 & 0.0069 \\
CWGAN  & \textbf{0.0070} & 0.0085 & 0.0110 \\
TimeGAN  & 0.0304 & 0.0307 & 0.0194 \\
RCGAN  & 0.0187 & \textbf{0.0065} & \textbf{0.0054} \\
GMMN  & 0.0096 & 0.0087 & 0.0073 \\ \midrule
\multicolumn{4}{c}{Absolute difference of lag-1 autocorrelation}                                                                                                      \\ \midrule
SigCWGAN  & \textbf{0.0124} & \textbf{0.0039} & 0.0044 \\
CWGAN  & 0.0614 & 0.0179 & 0.0109 \\
TimeGAN  & 0.0495 & 0.0787 & 0.0100 \\
RCGAN  & 0.0429 & 0.0124 & \textbf{0.0029} \\
GMMN  & 0.0219 & 0.0248 & 0.0118 \\ \midrule
\multicolumn{4}{c}{$R^2$ obtained from TSTR. (TRTR first row.)}                                                                                               \\ \midrule
\multicolumn{1}{l|}{TRTR} & 0.0457 & 0.2568 & 0.6434 \\ \midrule
SigCWGAN  & 0.0451 & \textbf{0.2562} & \textbf{0.6431} \\
CWGAN  & 0.0338 & 0.2406 & 0.6269 \\
TimeGAN  & 0.0432 & 0.2506 & 0.6365 \\
RCGAN  & 0.0437 & \textbf{0.2562} & 0.6429 \\
GMMN  & \textbf{0.0452} & 0.2539 & 0.6317 \\ \midrule
\multicolumn{4}{c}{Sig-$W_1$ distance}                                                                                                    \\ \midrule
SigCWGAN  & \textbf{0.0524} & \textbf{0.0476} & 0.0393 \\
CWGAN  & 0.0560 & 0.0584 & 0.0528 \\
TimeGAN  & 0.0648 & 0.0641 & 0.0660 \\
RCGAN  & 0.0546 & 0.0505 & 0.0437 \\
GMMN  & 0.0540 & 0.0482 & \textbf{0.0378} \\ \bottomrule

\end{tabularx}
\end{table}

%% file: numerical_results/table_var_dim=2.tex
\newcolumntype{Y}{>{\centering\arraybackslash}X}
\begin{table}[!ht]
\caption{Numerical results of $\operatorname{VAR}(1)$ for $d = 2$}\label{table_Var_dim=2}
\begin{tabularx}{1.1\textwidth}{l|Y|Y|Y||Y|Y|Y}
\toprule
\multicolumn{1}{c|}{}         & \multicolumn{3}{c||}{Temporal Correlations (fixing $\sigma=0.8$)}                                                   & \multicolumn{3}{c}{Feature Correlations (fixing $\phi=0.8$)}                                 \\ \midrule
\multicolumn{1}{c|}{Settings} & \multicolumn{1}{c|}{$\phi=0.2$}      & \multicolumn{1}{c|}{$\phi=0.5$}      & \multicolumn{1}{c||}{$\phi=0.8$}      & \multicolumn{1}{c|}{$\sigma=0.2$}    & \multicolumn{1}{c|}{$\sigma=0.5$}    & $\sigma=0.8$    \\ \midrule
\multicolumn{7}{c}{Metric on marginal distribution}                                                                                                                                                                                                 \\ \midrule
\multicolumn{1}{l|}{SigCWGAN}  & 0.0270 & 0.0122 & \textbf{0.0084} & \textbf{0.0089} & \textbf{0.0088} & \textbf{0.0084} \\
\multicolumn{1}{l|}{CWGAN}  & 0.0197 & 0.0134 & 0.0105 & 0.0154 & 0.0131 & 0.0105 \\
\multicolumn{1}{l|}{TimeGAN}  & 0.0270 & 0.0270 & 0.0173 & 0.0197 & 0.0164 & 0.0173 \\
\multicolumn{1}{l|}{RCGAN}  & 0.0120 & 0.0115 & 0.0091 & 0.0092 & 0.0104 & 0.0091 \\
\multicolumn{1}{l|}{GMMN}  & \textbf{0.0098} & \textbf{0.0110} & 0.0101 & 0.0104 & 0.0106 & 0.0101 \\ \midrule
\multicolumn{7}{c}{Absolute difference of lag-1 autocorrelation}                                                                                                                                                                                   \\ \midrule
\multicolumn{1}{l|}{SigCWGAN}  & \textbf{0.0069} & \textbf{0.0035} & \textbf{0.0054} & \textbf{0.0070} & \textbf{0.0062} & \textbf{0.0054} \\
\multicolumn{1}{l|}{CWGAN}  & 0.0523 & 0.0198 & 0.0415 & 0.1138 & 0.0219 & 0.0415 \\
\multicolumn{1}{l|}{TimeGAN}  & 0.0484 & 0.0589 & 0.0110 & 0.0300 & 0.0345 & 0.0110 \\
\multicolumn{1}{l|}{RCGAN}  & 0.0401 & 0.0057 & 0.0294 & 0.0308 & 0.0373 & 0.0294 \\
\multicolumn{1}{l|}{GMMN}  & 0.0318 & 0.0505 & 0.0537 & 0.0868 & 0.0679 & 0.0537 \\ \midrule
\multicolumn{7}{c}{$L_1$-norm of real and generated cross correlation matrices}                                                                                                                                                                    \\ \midrule
\multicolumn{1}{l|}{SigCWGAN}  & \textbf{0.0060} & 0.0097 & 0.0122 & \textbf{0.0040} & \textbf{0.0054} & 0.0122 \\
\multicolumn{1}{l|}{CWGAN}  & 0.0820 & 0.1909 & 0.0048 & 0.0254 & 0.1592 & 0.0048 \\
\multicolumn{1}{l|}{TimeGAN}  & 0.0435 & 0.0243 & 0.0134 & 0.0401 & 0.0441 & 0.0134 \\
\multicolumn{1}{l|}{RCGAN}  & 0.0669 & 0.0286 & 0.0160 & 0.1614 & 0.1551 & 0.0160 \\
\multicolumn{1}{l|}{GMMN}  & 0.0066 & \textbf{0.0006} & \textbf{0.0014} & 0.0110 & 0.0103 & \textbf{0.0014} \\  \midrule
\multicolumn{7}{c}{$R^2$ obtained from TSTR. (TRTR first row.)}                                                                                                                                                                                    \\ \midrule
\multicolumn{1}{l|}{TRTR}     & 0.0420 & 0.2563 & 0.6467 & 0.6421 & 0.6444 & 0.6467 \\ \midrule
\multicolumn{1}{l|}{SigCWGAN}  & \textbf{0.0406} & \textbf{0.2552} & \textbf{0.6458} & \textbf{0.6416} & \textbf{0.6438} & \textbf{0.6458} \\
\multicolumn{1}{l|}{CWGAN}  & -0.0019 & 0.1573 & 0.5901 & 0.5850 & 0.6038 & 0.5901 \\
\multicolumn{1}{l|}{TimeGAN}  & 0.0337 & 0.2327 & 0.6298 & 0.6239 & 0.6344 & 0.6298 \\
\multicolumn{1}{l|}{RCGAN}  & 0.0295 & 0.2130 & 0.6166 & 0.5997 & 0.5984 & 0.6166 \\
\multicolumn{1}{l|}{GMMN}  & 0.0291 & 0.2296 & 0.6156 & 0.5823 & 0.5943 & 0.6156 \\ \midrule
\multicolumn{7}{c}{Sig-$W_1$ distance}                                                                                                                                                                                                               \\ \midrule
\multicolumn{1}{l|}{SigCWGAN}  & \textbf{0.1913} & \textbf{0.1590} & \textbf{0.1190} & \textbf{0.2535} & \textbf{0.1235} & \textbf{0.1190} \\
\multicolumn{1}{l|}{CWGAN}  & 0.2684 & 0.2702 & 0.2244 & 0.3487 & 0.2158 & 0.2244 \\
\multicolumn{1}{l|}{TimeGAN}  & 0.2057 & 0.2036 & 0.1372 & 0.2719 & 0.1445 & 0.1372 \\
\multicolumn{1}{l|}{RCGAN}  & 0.2116 & 0.2165 & 0.1657 & 0.3292 & 0.2386 & 0.1657 \\
\multicolumn{1}{l|}{GMMN}  & 0.2118 & 0.1831 & 0.1508 & 0.2977 & 0.1761 & 0.1508 \\ 
\bottomrule
\end{tabularx}
\end{table}

%% file: numerical_results/table_var_dim=3.tex
\begin{table}[]
\caption{Numerical results of $\operatorname{VAR}(1)$ for $d = 3$}\label{table_Var_dim=3}
\centering
\begin{tabularx}{1.1\textwidth}{l|Y|Y|Y||Y|Y|Y}
\toprule
\multicolumn{1}{c|}{}         & \multicolumn{3}{c||}{Temporal Correlations (fixing $\sigma=0.8$)} & \multicolumn{3}{c}{Feature Correlations (fixing $\phi=0.8$)} \\ \midrule
\multicolumn{1}{c|}{Settings} & $\phi=0.2$          & $\phi=0.5$          & $\phi=0.8$          & $\sigma=0.2$       & $\sigma=0.5$       & $\sigma=0.8$       \\ \midrule
\multicolumn{7}{c}{Metric on marginal distribution}                                                                                                            \\ \midrule
\multicolumn{1}{l|}{SigCWGAN}  & 0.0254 & 0.0112 & \textbf{0.0077} & \textbf{0.0085} & \textbf{0.0076} & \textbf{0.0077} \\
\multicolumn{1}{l|}{CWGAN}  & 0.0142 & 0.0148 & 0.0194 & 0.0210 & 0.0113 & 0.0194 \\
\multicolumn{1}{l|}{TimeGAN}  & 0.0222 & 0.0218 & 0.0193 & 0.0188 & 0.0110 & 0.0193 \\
\multicolumn{1}{l|}{RCGAN}  & 0.0112 & 0.0157 & 0.0121 & 0.0156 & 0.0159 & 0.0121 \\
\multicolumn{1}{l|}{GMMN}  & \textbf{0.0098} & \textbf{0.0092} & 0.0101 & 0.0176 & 0.0162 & 0.0101 \\ \midrule
\multicolumn{7}{c}{Absolute difference of lag-1 autocorrelation}                                                                                              \\ \midrule
\multicolumn{1}{l|}{SigCWGAN}  & \textbf{0.0137} & \textbf{0.0066} & \textbf{0.0054} & \textbf{0.0045} & \textbf{0.0025} & \textbf{0.0054} \\
\multicolumn{1}{l|}{CWGAN}  & 0.0590 & 0.0242 & 0.1325 & 0.0864 & 0.0785 & 0.1325 \\
\multicolumn{1}{l|}{TimeGAN}  & 0.0554 & 0.0385 & 0.0374 & 0.1219 & 0.0879 & 0.0374 \\
\multicolumn{1}{l|}{RCGAN}  & 0.0864 & 0.0532 & 0.0217 & 0.1434 & 0.1303 & 0.0217 \\
\multicolumn{1}{l|}{GMMN}  & 0.0315 & 0.0584 & 0.0968 & 0.1183 & 0.1348 & 0.0968 \\ \midrule
\multicolumn{7}{c}{$L_1$-norm of real and generated cross correlation matrices}                                                                               \\ \midrule
\multicolumn{1}{l|}{SigCWGAN}  & \textbf{0.0331} & \textbf{0.0498} & \textbf{0.0055} & \textbf{0.0532} & \textbf{0.0401} & \textbf{0.0055} \\
\multicolumn{1}{l|}{CWGAN}  & 0.0628 & 0.2067 & 0.2812 & 0.4365 & 0.2071 & 0.2812 \\
\multicolumn{1}{l|}{TimeGAN}  & 0.6549 & 0.3619 & 0.1542 & 0.2644 & 0.3153 & 0.1542 \\
\multicolumn{1}{l|}{RCGAN}  & 0.4552 & 0.3441 & 0.0500 & 0.1448 & 0.4355 & 0.0500 \\
\multicolumn{1}{l|}{GMMN}  & 0.0811 & 0.1225 & 0.2405 & 0.3018 & 0.3883 & 0.2405 \\ \midrule
\multicolumn{7}{c}{$R^2$ obtained from TSTR. (TRTR first row.)}                                                                                               \\ \midrule
TRTR                         & 0.0420 & 0.2532 & 0.6509 & 0.6459 & 0.6485 & 0.6509 \\ \midrule
\multicolumn{1}{l|}{SigCWGAN}  & \textbf{0.0388} & \textbf{0.2490} & \textbf{0.6492} & \textbf{0.6446} & \textbf{0.6469} & \textbf{0.6492} \\
\multicolumn{1}{l|}{CWGAN}  & -0.0150 & 0.1770 & 0.5928 & 0.5462 & 0.5676 & 0.5928 \\
\multicolumn{1}{l|}{TimeGAN}  & -0.0088 & 0.2039 & 0.6045 & 0.5600 & 0.6026 & 0.6045 \\
\multicolumn{1}{l|}{RCGAN}  & 0.0092 & 0.1994 & 0.5921 & 0.5064 & 0.5456 & 0.5921 \\
\multicolumn{1}{l|}{GMMN}  & -0.0115 & 0.1683 & 0.5388 & 0.4899 & 0.4920 & 0.5388 \\ \midrule
\multicolumn{7}{c}{Sig-$W_1$ distance}                                                                                                                        \\ \midrule
\multicolumn{1}{l|}{SigCWGAN}  & \textbf{0.4289} & \textbf{0.3817} & \textbf{0.2374} & \textbf{0.2648} & \textbf{0.3999} & \textbf{0.2374} \\
\multicolumn{1}{l|}{CWGAN}  & 0.4653 & 0.4173 & 0.3226 & 0.3875 & 0.4692 & 0.3226 \\
\multicolumn{1}{l|}{TimeGAN}  & 0.5030 & 0.4321 & 0.3087 & 0.3753 & 0.4415 & 0.3087 \\
\multicolumn{1}{l|}{RCGAN}  & 0.4751 & 0.4418 & 0.3034 & 0.4334 & 0.4859 & 0.3034 \\
\multicolumn{1}{l|}{GMMN}  & 0.4621 & 0.4159 & 0.3151 & 0.3946 & 0.4939 & 0.3151 \\ \midrule
\bottomrule

\end{tabularx}
\end{table}

%% file: numerical_results/table_arch.tex
\newcolumntype{Y}{>{\centering\arraybackslash}X}
\begin{table}[!ht]
\caption{Numerical results of the ARCH(p) datasets.}
\centering 
\label{table_arch}
\begin{tabularx}{.7\textwidth}{l|Y || Y|| Y}
\toprule
Settings & \multicolumn{1}{c||}{$p=2$} & \multicolumn{1}{c||}{$p=3$} & \multicolumn{1}{c}{$p=4$}\\ \midrule
\multicolumn{4}{c}{Metric on marginal distribution}                 \\\midrule
SigCWGAN   & 0,00918                 & 0,00880    &     \textbf{0.01142}     \\
TimeGAN   & 0,02569                 & 0,02119               &     0.2191   \\
RCGAN     & 0,01069                 & 0,01612               &     0.01182   \\
GMMN      & \textbf{0,00744}        & \textbf{0,00783}               &    0.01259    \\\midrule
\multicolumn{4}{c}{Absolute difference of lag-1 autocorrelation}                                           \\\midrule
SigCWGAN   & \textbf{0,00542}                 & \textbf{0,00852}      &     \textbf{0.01106}   \\
TimeGAN   & 0,01714                 & 0,02401                &    0.03267   \\
RCGAN     & 0,05372                 & 0,01685                &     0.04879  \\
GMMN      & 0,02056        & 0,00859                &    0.01441   \\\midrule
\multicolumn{4}{c}{$L_1$-norm of real and generated cross correlation matrices}                               \\\midrule
SigCWGAN   & 0,00462        & \textbf{0,00546}      &    0.00489    \\
TimeGAN   & \textbf{0,00315}                 & 0,06551               &    0.04408    \\
RCGAN     & 0,01604                 & 0,08823              &     \textbf{0.00235}   \\
GMMN      & 0,04326                 & 0,03930               &    0.01603    \\\midrule
\multicolumn{4}{c}{$R^2$ obtained from TSTR. (TRTR first row.)}                   
\\\midrule
TRTR      & {0,32168}        & {0,32615}       &   0.33305    \\ \midrule
SigCWGAN   & \textbf{0,31623}                 & \textbf{0,31913}        &     \textbf{0.31642}  \\
TimeGAN   & 0,30835                 & 0,30556                &     0.30240  \\
RCGAN     & 0,31146        & 0,30727                &    0.30924   \\
GMMN      & 0,27982                 & 0,28072                &    0.30742   \\\midrule
\multicolumn{4}{c}{Sig-$W_1$ distance}                                          \\\midrule
SigCWGAN   & \textbf{0,12210}        & \textbf{0,14682}      &     \textbf{0.14098}   \\
TimeGAN   & 0,20228                 & 0,22761               &    0.23398    \\
RCGAN     & 0,18781                 & 0,20943               &    0.21876    \\
GMMN      & 0,26797                 & 0,26853               &     0.25811   \\ \bottomrule
\end{tabularx}
\end{table}

%% file: numerical_results/table_stocks.tex
\newcolumntype{Y}{>{\centering\arraybackslash}X}
\begin{table}[!ht]
\caption{Numerical results of the stocks datasets.}
\centering 
\label{table_stocks}
\resizebox{0.65\columnwidth}{!}{
\begin{tabularx}{.7\textwidth}{l| Y || Y}
\toprule
Data type & \multicolumn{1}{c||}{SPX} & \multicolumn{1}{c}{SPX + DJI} \\ \midrule
\multicolumn{3}{c}{Metric on marginal distribution}                 \\\midrule
SigCWGAN  & 0.0142 & \textbf{0.0093} \\
CWGAN  & 0.0111 & 0.0151 \\
TimeGAN  & 0.0089 & 0.0156 \\
RCGAN  & \textbf{0.0088} & 0.0140 \\
GMMN  & 0.0107 & 0.0181 \\ \midrule
\multicolumn{3}{c}{Absolute difference of lag-1 autocorrelation}                                           \\\midrule
SigCWGAN  & 0.0302 & 0.0447 \\
CWGAN  & 0.1617 & 0.0571 \\
TimeGAN  & \textbf{0.0180} & 0.0232 \\
RCGAN  & 0.0350 & 0.0515 \\
GMMN  & 0.0273 & \textbf{0.0106} \\ \midrule
\multicolumn{3}{c}{$L_1$-norm of real and generated cross correlation matrices}                               \\\midrule
SigCWGAN  & 0.0503 & \textbf{0.0747} \\
CWGAN  & \textbf{0.0131} & 0.3908 \\
TimeGAN  & 0.0793 & 0.5401 \\
RCGAN  & 0.0654 & 0.3959 \\
GMMN  & 0.0409 & 0.2103 \\ \midrule
\multicolumn{3}{c}{$R^2$ obtained from TSTR. (TRTR first row.)}                   
\\\midrule
TRTR      & 0.3689 & 0.3731 \\ \midrule
SigCWGAN  & \textbf{0.3576} & 0.3466 \\
CWGAN  & 0.2744 & 0.2694 \\
TimeGAN  & 0.3551 & \textbf{0.3602} \\
RCGAN  & 0.3037 & 0.3532 \\
GMMN  & 0.3375 & 0.3368 \\ \midrule
\multicolumn{3}{c}{Sig-$W_1$ distance}                                          \\\midrule
SigCWGAN  & \textbf{0.0985} & \textbf{0.1307} \\
CWGAN  & 0.1684 & 0.2881 \\
TimeGAN  & 0.1265 & 0.2321 \\
RCGAN  & 0.1462 & 0.2353 \\
GMMN  & 0.1257 & 0.2448 \\ \bottomrule
\end{tabularx}}
\end{table}